\documentclass[11pt]{article}

\usepackage[english]{babel}
\usepackage[utf8x]{inputenc}
\usepackage[T1]{fontenc}

\usepackage{fullpage}
\usepackage[sort&compress,square,comma,authoryear]{natbib}

\usepackage{amsthm}
\usepackage{color}

 \theoremstyle{plain}
\newtheorem{theorem}{Theorem}
\newtheorem{lemma}{Lemma}
\newtheorem{definition}{Definition}
\newtheorem{proposition}{Proposition}

\newtheorem{remark}{Remark}

\author{Yuval Dagan\thanks{Weizmann Institute of Science, Israel, \texttt{yuvaldag@gmail.com }}
	\and
	Ohad Shamir\thanks{Weizmann Institute of Science, Israel, \texttt{ohad.shamir@weizmann.ac.il}}}
\date{}

\usepackage{bm}
\usepackage{mathtools}
\usepackage{amssymb}

\usepackage{enumitem}

\usepackage{times}

\usepackage{amsfonts,color,float,verbatim}
\usepackage{algorithm,algorithmic}
\usepackage{wrapfig}

\providecommand{\helli}{\mathrm{h}}

\providecommand{\CD}{\mathrm{CD}}
\providecommand{\BCD}{\mathrm{BCD}}

\DeclarePairedDelimiterX{\infdivx}[2]{(}{)}{%
#1\;\delimsize\|\;#2%
}

\newcommand{\reals}{\mathbb{R}}
\newcommand{\E}{\mathbb{E}}

\newcommand{\bx}{\mathbf{x}}

\newcommand{\Ucal}{\mathcal{U}}

\newcommand{\secref}[1]{Sec.~\ref{#1}}
\newcommand{\subsecref}[1]{Subsection~\ref{#1}}

\renewcommand{\eqref}[1]{Eq.~(\ref{#1})}
\newcommand{\lemref}[1]{Lemma~\ref{#1}}
\newcommand{\thmref}[1]{Thm.~\ref{#1}}

\title{Detecting Correlations with Little Memory and Communication}

\begin{document}

\maketitle

\begin{abstract}
We study the problem of identifying correlations in multivariate data, under information constraints: Either on the amount of memory that can be used by the algorithm, or the amount of communication when the data is distributed across several machines. We prove a tight trade-off between the memory/communication complexity and the sample complexity, implying (for example) that to detect pairwise correlations with optimal sample complexity, the number of required memory/communication bits is at least quadratic in the dimension. Our results substantially improve those of \cite{shamir2014fundamental}, which studied a similar question in a much more restricted setting. To the best of our knowledge, these are the first provable sample/memory/communication trade-offs for a practical estimation problem, using standard distributions, and in the natural regime where the memory/communication budget is larger than the size of a single data point. To derive our theorems, we prove a new information-theoretic result, which may be relevant for studying other information-constrained learning problems.
\end{abstract}
\section{Introduction}

Information constraints play a key role in statistical learning and 
estimation problems. One always-present constraint is the sample size: 
We attempt to infer something about an underlying 
distribution, given only a finite amount of data sampled from that 
distribution. Indeed, the sample complexity for 
tackling various statistical problems is a central area in learning 
theory and statistics. However, in many situations, we are faced with 
additional information-based constraints, besides the sample complexity. For 
example, in practice the amount of \emph{memory} used by the learning algorithm 
might be limited. In other cases, we might wish to solve a distributed version 
of the learning problem, where the data is randomly partitioned across several machines. Since communication between machines is invariably slow and expensive 
compared to internal processing, we might wish to solve the problem using a 
bounded amount of \emph{communication}. 

In recent years, an emerging body of literature has attempted to formally study 
the effect of such memory and communication constraints in learning problems. In many cases, it turns out that one can still solve a given problem 
with less memory or communication, but at the cost of a larger sample 
complexity. Thus, a fascinating question is whether such trade-offs are 
unavoidable, and what is the optimal trade-off. 

In this paper, we study memory, communication, and sample complexity trade-offs for detecting correlations in multivariate data, one of the simplest and most common statistical estimation problems. In this problem, we are given a sequence of i.i.d. samples $\bx_1,\bx_2,\ldots$ from some zero-mean distribution over $\reals^d$, and our goal is to detect correlated coordinates. For simplicity, let us focus for now on the case of pairwise correlations, and assume that for some pair of coordinates $(i,j)\in\{1,\ldots,d\}^2$, $\E\left[x_i x_j\right]=\rho > 0$, whereas for any other pair of coordinates $(i',j')$, $\E\left[x_{i'}x_{j'}\right]=0$.

In the absence of memory or communication constraints, and given a sample 
$\bx_1,\ldots,\bx_t$, a simple approach is to compute the empirical average
$
\frac{1}{t}\sum_{l=1}^{t}x_{l,i}x_{l,j}
$
for every possible coordinate pair, use 
concentration of measure to bound the difference between this empirical average 
and the true expectation with high probability, and thus determine which 
of these subsets is indeed correlated. For example, Hoeffding's inequality and a union bound implies that if all coordinates are 
bounded in $[-1,+1]$ almost surely, and $
t= \Omega\left(\log(d)/\rho^2\right),
$
then with arbitrarily high constant probability over the sample 
$\bx_1,\ldots,\bx_t$, there will be a unique coordinate pair $(i,j)$ for which
$
\left|\frac{1}{t}\sum_{l=1}^{t}x_{l,i}x_{l,j}\right|\geq 
\frac{\rho}{2}
$,
and this pair corresponds to the correlated coordinates.

Although this approach is quite reasonable in terms of the sample complexity, it requires us to compute and maintain $\binom{d}{2}\approx d^2$ averages, which can be problematic in 
memory/communication-constrained variants of the problem: For an algorithm 
which streams over the data, we need at least $\Omega(d^2)$ memory to keep 
track of all the possible correlations. Similarly, when the data is distributed across several machines, we need at least $\Omega(d^2)$ bits of communication, to compute the empirical averages of every pair of coordinates.

What can we do if our memory or communication budget is less than $\Omega(d^2)$? Considering the case of streaming, memory-bounded algorithms first, a trivial solution is not to estimate all $\binom{d}{2}$ averages at once, but rather a smaller group of averages at a time. For example, if we only have enough memory to estimate one average, we can start with estimating the empirical correlation of coordinates $1$ and $2$, until we are sufficiently confident whether they are correlated or not, then move to coordinates $1$ and $3$, and so on. However, if we stream over the data, the price we pay is a larger sample size: In general, if we have $s$ bits of memory, the approach above requires a sample size of $t= \tilde{O}\left(d^2/\left( \rho^2 s \right)\right)$ to detect the correlation with any constant probability (where the $\tilde{O}$ notation hides factors logarithmic in $d,\rho$\footnote{For example, we need $O(\log(1/\rho))$ bits of precision to determine if an average is above $\rho$, and $O(\log(d))$ bits to index coordinates. We note that sometimes such logarithmic factors can be reduced with various tricks (e.g., \citet{luo2005universal}), but these are not the focus of our paper.}). In other words, the approach we just described satisfies
$ts = \tilde{O}\left(d^2/\rho^2\right)$.
More generally, if the algorithm is allowed to perform $\ell$ passes over the same data $\bx_1,\ldots,\bx_t$, then with the same approach, we can detect the correlation assuming 
\[
ts\ell=\tilde{O}\left(d^2/\rho^2\right)~.
\]
A natural question is whether this naive approach is improvable. Can we have an algorithm where the product of the memory size $s$ and the total number of data points processed $t\ell$ is smaller, perhaps less than quadratic in the dimension $d$?

An analogous situation occurs in the context of distributed algorithms with communication constraints: If each machine has $n$ i.i.d. data points, and can send $s$ bits of communication, we can split the $m$ machines to $\tilde{O}(d^2/s)$ groups, and have the machines in each group broadcast the empirical average of a different subset of $\tilde{\Theta}(s)$ coordinate pairs. Aggregating these averages and outputting the pair with the highest empirical correlation, we will succeed with any constant probability as long as 
$
\frac{m}{d^2/s}\cdot n \geq \tilde{\Omega}\left(\frac{1}{\rho^2}\right)
$
(namely, as long as we can compute the empirical average of at least $\tilde{\Omega}(1/\rho^2)$ data points, for each and every coordinate pair).This implies that the protocol will succeed, with the total number $ms$ of bits communicated at most
\[
\tilde{O}\left(d^2/(n\rho^2)\right).
\]
Note that the non-trivial regime here is $n\rho^2\ll 1$ (otherwise, any single machine can detect the correlation based on its own data, without any communication). In this regime, we see that the protocol above requires communication complexity quadratic in the dimension $d$. Again, it is natural to ask whether this simple approach can be improved, and whether the quadratic dependence on $d$ is avoidable.

Perhaps surprisingly, we show in this paper that these approaches are in fact optimal (up to logarithmic factors), and establish a tight trade-off between sample complexity and memory/communication complexity for detecting correlations. Moreover, we show this for simple, natural data distributions; under minimal algorithmic assumptions; and for both pairwise and higher-order correlations (see below for a discussion of related results). In a nutshell, our contributions are the following: 
\begin{itemize}[leftmargin=*]
\item We prove that if the correlation $\rho$ is sufficiently small (polynomially in $d$), then for any algorithm with $s$ bits of memory, which performs at most $\ell$ passes over a sample of size $t$, we must have $ts\ell = \tilde\Omega(d^2/\rho^2)$ for it to detect the correlated coordinates. Also, in a distributed setting, a communication of $\tilde\Omega\left(d^2/\left(n \rho^2 \right)\right)$ bits is necessary in general. This matches the upper bounds described above up to logarithmic factors. We prove these results for two families of natural distributions: over binary vectors in $\{-1,+1\}^d$, and for Gaussian distributions over $\reals^d$.
\item For binary vectors, we actually provide a more general result, which applies also to higher-order correlations. Specifically, we assume that there is some unique set $I$ of indices such that
$
\mathbb{E} \prod_{i \in I} X_i = \rho
$,
and $I$ comes from some known family of $k$ possible subsets (the previous bullet refers to the special case where $|I|=2$, and the family of $k=\binom{d}{2}$ coordinate pairs). Assuming $\rho$ is polynomially small in $k$, we show that in the memory-constrained setting, $ts\ell = \tilde\Omega(k/\rho^2)$, and in the communication-constrained setting, $\tilde \Omega\left(k/\left(n\rho^2\right)\right)$ bits are required. This directly generalize the results from the previous bullet, and establishes that one cannot in general improve over the naive approach of estimating the correlation separately for every  candidate set $I$. 
\item To obtain our theorems, we develop a general information-theoretic result, which may be of independent interest and can be roughly stated as follows: Assume that $\mu_0, \mu_1, \dots, \mu_k$ are distributions over the same sample space, which are close to each other in the sense that for any $1 \le i \le k$ and any event $E$,
$
\left\lvert \mu_i(E)/\mu_0(E) - 1\right\rvert \le \rho
$.
Additionally, assume that $\mu_1,\dots, \mu_k$ are pairwise uncorrelated, in the sense that for any $i\neq j$, 
$\int \frac{d\mu_i}{d\mu_0} \frac{d\mu_j}{d\mu_0} d\mu_0
= \int \frac{d\mu_i}{d\mu_0}  d\mu_0
\int \frac{d\mu_j}{d\mu_0} d\mu_0
=1$.
Then any algorithm for identifying the distribution $\mu_i$ given a sample requires either $ts\ell = \tilde \Omega(k/\rho^2)$  in the memory-constrained setting, or $\tilde\Omega(k/(n\rho^2))$ bits of communication in a communication-constrained setting. This can be seen as generalizing the main technical result of \citet{braverman2016communication} (Theorem 4.4), which proved a related lower bound in the context of communication constraints, assuming that the $k$ distributions are defined over a product space. Here, we essentially replace independence assumptions by a weaker pairwise uncorrelation assumption, which is crucial for proving our results. \end{itemize}

\subsection*{Related Work}

The question of proving lower bounds on learning under memory and communication 
constraints has been receiving increasing attention recently, and related 
questions have long been studied in theoretical computer science and other 
fields. Thus, it is important to emphasize the combination of assumptions that 
place our setting apart from most other works:
\begin{itemize}[leftmargin=*]
	\item \emph{The task is a statistical learning problem, based on  
	i.i.d. examples from some underlying distribution}: For example, there is a 
	large literature on memory lower bounds for streaming algorithms (see for 
	instance \citet{alon1996space, bar2002information, muthukrishnan2005data} and 
	references therein). However, these mostly focus on 
	problems which are not standard learning problems, and/or that the data 
	stream is adversarially generated rather than stochastically generated 
	(which makes proving lower bounds easier). Similarly, there are many results on communication complexity (see for instance \citet{KuNi97}), but most of them refer to non-learning problems, or where the data is adversarially generated and distributed across machines (rather than randomly, which again makes lower bounds easier to prove).
	\item \emph{Memory/communication budget is larger than the size of a single 
	data point}: This is arguably the most common regime in practice. There are several works which studied the more constrained setting, where the memory or 
	communication budget is smaller than the size of a single data point (but 
	still larger than the required output), for 
	problems such as sparse mean estimation, sparse regression, detecting 
	low-rank subspaces, and multi-armed bandits	\citep{shamir2014fundamental,steinhardt2015minimax,crouch2016stochastic,braverman2016communication}. 
	 Also, there has been a line 
	of works on hypothesis testing and statistical
	estimation with finite memory, in a regime where the memory is insufficient 
	to precisely express the required output (see \cite{hellman1970learning,leighton1986estimating,ertin2003sequential,kontorovich2012statistical} and references therein). 
	\item \emph{Results are for a standard, natural estimation problem, and where 
	multiple communication rounds / passes over the data are allowed}: 
	A breakthrough line of recent works  
\citep{raz2016fast,raz2017time,moshkovitz2017mixing,moshkovitz2017strongmixing,kol2017time,garg2017extractor,beame2017time}
	showed that for binary classification problems, which satisfy certain combinatorial or algebraic conditions, any one-pass streaming algorithm would require either quadratic memory (in the dimension), or exponential sample size. So far, these conditions were shown to hold for learning parities and variants thereof (all strongly involving Boolean computations over $\mathbb{Z}_2$). Although such problems are very important in learning theory, they are arguably synthetic in nature and not commonly encountered in practice. In this paper, we focus on detecting correlations, which is a standard and common estimation problem. Moreover, whereas the results above apply to memory-constrained, one-pass 
	algorithms, our results apply to both memory and communication constraints, and where multiple passes / communication rounds are allowed (building on techniques developed in \cite{braverman2016communication}). On the flip side, the gaps we show in the required sample size (with and without information constraints) are polynomial in the dimension, whereas the results above imply exponential gaps. We discuss the differences and the similarities in more depth in \secref{sec:raz-compare}.
\end{itemize}

Perhaps the work closest to ours is \citet{shamir2014fundamental}, which also 
studied the problem of detecting correlations with memory/communication 
constraints, and showed trade-offs between the memory/communication complexity and the sample complexity. For example, in the context of memory constraints, that paper showed that there exists a distribution over $d$-dimensional vectors, with a particular correlation value $\rho$ (depending on $d$), such that detecting the correlation is statistically feasible given $O(d^2\log^2(d))$ examples, but any one-pass algorithm with only $s\ll d^2/\log^2(d)$ bits of memory requires a strictly larger sample size of at $\Omega(d^4/s)$ examples. However, that result is weaker than ours in several respects: First, it applies to a much more restrictive family of algorithms (where only one round of communication is allowed in the communication-constrained setting, and only one pass over the data in the memory-constrained setting). Second, it only 
applies to a certain carefully-tailored and unnatural family of data 
distributions, and does not imply communication/memory/sample trade-offs in the 
context, say, of vectors with bounded or Gaussian entries. Third, the result 
only holds for a particular choice of the correlation parameter $\rho$ 
(depending on the other problem parameters), rather 
than holding for any small enough correlation. Fourth, the result is specific to pairwise correlations, whereas we prove more general results, applying to  higher-order correlations and (potentially) to other information-constrained learning problems. Moreover, proving these results require fundamentally new ideas, which we develop in this paper.

Finally, for pairwise correlations, the problem we study is closely related to the \emph{light-bulb problem}, proposed by Leslie Valiant at the very first COLT conference \citep{valiant1988functionality}. That problem is equivalent to 
identifying a pairwise correlation in data drawn from the $d$-dimensional Boolean cube. However, while we ask whether $o(d^2)$
memory/communication is possible, Valiant asked whether $o(d^2)$
\emph{runtime} is possible. For the light-bulb problem, the best 
algorithm we are aware of \citep{valiant2015finding} requires a runtime of only
$O(d^{1.62})$. However, a close inspection of 
the results indicates that this only applies when the correlation parameter 
$\rho$ is close to being an absolute constant (a regime which also makes the 
communication/memory-constrained setting easier -- see Remark 
\ref{remark:rhosmall}). Although communication/memory 
complexity and computational complexity are not the same, our results suggest that no algorithm for the light-bulb problem 
can run in time $o(d^2)$ (as a function of $d$), if the correlation to be detected is small enough. 

Our paper is structured as follows: In \secref{sec:prelim}, we introduce notation and necessary definitions. In \secref{sec:main}, we present our main results, and in \secref{sec:ideas}, we sketch our main proof ideas and techniques. Full proofs are provided in Appendix \ref{sec:proofs}, and some additional results are provided in Appendix \ref{sec:tuples}. 

\section{Preliminaries}\label{sec:prelim}

For any integer $k \ge 1$, the notation $[k]$ denotes the set $\{1, \dots, k\}$. We use the standard $O(),\Omega()$ big-O notation to hide constants, and $\tilde{O}(),\tilde{\Omega}()$ to hide constants as well as polylogarithmic factors. For any distribution $\mu$ and any integer $n \ge 1$, define by $\mu^n$ the distribution over $n$ i.i.d samples from $\mu$.

\subsection{Communication protocols and memory-limited algorithms}

In the context of communication-constrained algorithms, we consider a multi-party setting where there are $m\ge 1$ parties/machines, and each party  receives an input visible only to her (i.e. a sample of data points). The parties communicate using broadcast messages with the goal of calculating some function over all of the inputs. A \emph{protocol} defines the communication between the parties: which party is to speak next and which message she should send as a function of her input, the message history and some randomness. The \emph{communication complexity} of a protocol is the maximal number of bits sent in this protocol, where the maximum is over all possible inputs and over the randomness of the protocol\footnote{It is well-known that worst-case and average-case communication complexity are equivalent up to constants, so our lower bounds also apply to the communication complexity in expectation over the inputs and the randomness of the protocol. To see this, note that if there is a protocol $\pi$ with expected communication complexity $b$, succeeding with probability $9/10$, then by Markov's inequality, a protocol $\pi'$ which simulates $\pi$ and stops after $10b$ bits of communication still succeeds with probability $8/10$, and has maximal communication complexity $10b$.}. The \emph{transcript} of a protocol contains all the messages sent.

\begin{definition}
Let $m,n \ge 1$ be integers, let $k\ge 2$ be an integer and let $\mu_1, \dots, \mu_k$ be distributions on the same sample space. An \emph{$(m,n)$-protocol identifying $\mu \in \{\mu_1, \dots, \mu_k\}$ with error $\varepsilon$} is an $m$-party communication protocol where each party receives as an input an independent set of $n$ i.i.d. samples from the same distribution $\mu_i$. Additionally, for any $i\in [k]$, the protocol outputs the index $i$ of the distribution $\mu_i$ which generated the data, with probability at least $1-\varepsilon$.
\end{definition}

We emphasize that the protocols we consider are not restricted in terms of the number of messages sent or the number of communication rounds: We are only interested in the overall communication complexity, namely the total number of bits sent between machines.

In the memory-constrained setting, we consider an algorithm which is allowed to perform $\ell$ passes over $t$ data points sampled i.i.d. from some distribution, with a memory limitation of $s$ bits:
\begin{definition}
Let $t,s, \ell\ge 1$ be integers and let $\mu_1, \dots, \mu_k$ be distributions on the same sample space.
A \emph{$(t,s,\ell)$-algorithm identifying $\mu \in \{\mu_1, \dots, \mu_k\}$ with error $\varepsilon$} is an algorithm receiving $t$ i.i.d. samples from $\mu_i$ for some $1 \le i \le k$. This algorithm goes over all samples sequentially in $\ell$ passes, using at most $s$ bits of memory (formally, letting $x_1,x_2,\ldots,x_{t\ell}$ be $\ell$ copies of the data set in order, we assume the algorithm can be written recursively as $u_{i+1}=f_i(x_i,u_{i})$, where $u_i\in \{0,1\}^s$ for all $i$ denotes the memory of the algorithm after handling example $x_i$, $f_i$ is an arbitrary function, and the output is a function of $u_{t\ell+1}$). For any $i\in [k]$, the algorithm outputs the index $i$ of the distribution $\mu_i$ generating the data, with probability at least $1-\varepsilon$.
\end{definition}

\subsection{Centered families of distributions}

For our results, we will consider families of distributions which are all close to one another, in the following sense:

\begin{definition}\label{def:cd}
Let $0 < \rho < 1$ be a number, let $k \ge 2$ be an integer and let $\mu_1, \dots, \mu_k$ be distributions on the same sample space $\Omega$ and the same set of events $\mathcal{F}$. We say that $\{\mu_1, \dots, \mu_k \}$ is a \emph{$\rho$-centered family of distributions} (or $\CD(\rho)$ for brevity), if there exists a distribution $\mu_0$ on the same sample space and the same set of events such that for any event $E \in \mathcal{F}$ and any $i \in [k]$,
\[
(1 - \rho) \mu_0(E) \le \mu_i(E) \le (1 + \rho) \mu_0(E).
\]
We say that $\{ \mu_1, \dots, \mu_k \}$ is \emph{centered around $\mu_0$}.
\end{definition}

\section{Main results}\label{sec:main}

Our results are based on two general theorems, which establish the difficulty of distinguishing generic distributions under communication and memory constraints respectively. These theorems are presented in \subsecref{subsec:general}. We then  apply them to the problem of detecting correlations, for distributions over binary vectors (\subsecref{subsec:binary}) and for Gaussian distributions (\subsecref{subsec:gaussian}).

\subsection{A General Theorem}\label{subsec:general}

Let $\left\{\mu_1,\dots, \mu_k \right\}$ be a $\CD(\rho)$ family of probability distributions centered around $\mu_0$ (namely, $\lvert \mu_i(E)/\mu_0(E)-1 \rvert \le \rho$ for any $i \in [k]$ and any event $E$). The following theorem establishes that under a certain technical condition  (\eqref{eq:corr}), any $(m,n)$ protocol would require a lot of communication to identify the distribution from which the input data is sampled:

\begin{theorem} \label{thm:main}
There exist positive numerical constants $C, C'$ such that the following holds. Let $\{ \mu_1, \dots, \mu_k \}$ be a $\CD(\rho)$ family of distributions centered around $\mu_0$, let $m,n \ge 1$ be integers such that $\rho \le (n \ln k)^{-1/2}/C'$.
If
\begin{equation} \label{eq:corr}
\sum_{S \subseteq [k] \colon \lvert S \rvert \ge 2} n^{-\lvert S \rvert/2} \rho^{-\lvert S \rvert} \left\lvert \mathbb{E}_{A \sim \mu_0} \prod_{i\in S} \left(\frac{\mu_i(A)}{\mu_0(A)} - 1\right) \right\rvert
\le \frac{1}{n}~,
\end{equation}
then any $(m,n)$-protocol identifying $\{\mu_1, \dots, \mu_k\}$ with error $1/3$ has a communication complexity of at least
\[ \frac{k}{C \rho^2 n \log (k /(n\rho^2))}. \]
In particular, \eqref{eq:corr} holds if there exists an integer $\ell \ge 2$ such that all the terms in \eqref{eq:corr} corresponding to $\lvert S \rvert \le \ell$ are zero, and $n \ge k^{2(\ell+1)/(\ell-1)}$.
\end{theorem}

The proof appears in \subsecref{sec:main-pr}, whereas \lemref{lem:main} is its main ingredient.
To explain the intuition, let $B_i$ (for $i\in[k]$) be the random variable $\frac{\mu_i(A)}{\mu_0(A)}$, where $A$ is sampled from $\mu_0$, and note that its expectation is always $1$. \eqref{eq:corr} corresponds to requiring $B_i$ to be approximately uncorrelated when $n$ is large enough, namely
\[
\sum_{S \subseteq [k] \colon \lvert S \rvert \ge 2} n^{-\lvert S \rvert/2} \rho^{-\lvert S \rvert} \left\lvert \mathbb{E}\prod_{i\in S} \left(B_i -\mathbb{E}[B_i]\right)\right\rvert
\le \frac{1}{n}~.
\]
The last part of the theorem simply states that this indeed holds, if the $B_i$ random variables are uncorrelated up to order $\ell$, and $n$ is large enough. In particular, for large $n$, pairwise uncorrelation ($\ell=2$) is sufficient. The theorem implies that if the distributions are ``uncorrelated'' in this sense, then the task of identifying $\mu \in \{ \mu_1, \dots, \mu_k\}$ requires a communication complexity of $\tilde\Omega(k/(n \rho^2))$. Crucially, the required communication scales linearly with the number of distributions $k$, and is no better than what we would need for solving $k$ completely independent problems, each involving distinguishing only two such distributions.

We now turn from communication complexity to memory complexity. The following theorem establishes a lower bound on the product of the sample size, memory, and number of data passes for any memory-constrained algorithm which identifies $\mu_1,\ldots,\mu_k$: 

\begin{theorem} \label{thm:mem}
There exist positive numerical constants $C^{(2)},C^{(3)}$ such that the following holds. Let $\{ \mu_1, \dots, \mu_k \}$ be a $\CD(\rho)$ family centered around $\mu_0$, and let $t,s,\ell \ge 1$ be integers. Assume that there exists $n \le C^{(2)}/(\rho^2\log k)$ such that the conditions of Theorem~\ref{thm:main} hold, with respect to $k$, $n$ and $\rho$. Then any $(t,s,\ell)$-algorithm identifying $\mu_1, \dots, \mu_k$ with $1/3$ error satisfies
\[
ts\ell \ge \frac{k}{C^{(3)} \rho^2 \log k}.
\]
\end{theorem}
The proof of the theorem is a simple reduction to the communication complexity lower bound of \thmref{thm:main}: Given a $(t,s, \ell)$ algorithm, and any $m,n$ such that $mn\geq t$, one can create an $(m,n)$ protocol which simulates the algorithm in a distributed setting as follows: Fixing some arbitrary order over the parties, each party in turn simulates the $(t,s,\ell)$ algorithm over its data. Once the party exhausts her data, the state of this algorithm (consisting of at most $s$ bits) is transmitted to the next party, which continues to simulate the algorithm, and so on. Once $t$ data points have been processed in this manner, the current party transmits the algorithm's state back to the first party, which starts simulating the next pass of the $(t,s,\ell)$ algorithm. This continues until $\ell$ such passes are done. Then, the output of the protocol is set as the output of the simulated $(t,s,\ell)$ algorithm. The overall communication complexity is at most $ts\ell/n$, so by \thmref{thm:main} (assuming its conditions are fulfilled), we must have
\begin{equation}\label{eq:memcom}
\frac{ts\ell}{n}~\geq~\frac{k}{C\rho^2 n\log(k/(n\rho^2))}.
\end{equation}
In particular, picking $m=k$ and $n=C^{(2)}/(\rho^2 \log k)$ for any constant $C^{(2)}\leq C'^{-2}$ concludes the proof.



We finish this subsection with two additional remarks:
\begin{remark}[Identification vs. binary decision] \label{remark:id-vs-binary}
In the results of this paper, we focus on the problem of identifying an underlying distribution, under the promise that it belongs to a certain family of distributions $\mu_1,\ldots,\mu_k$ (e.g., which pair of coordinates are correlated). An arguably easier task is to decide whether the underlying distribution is either some fixed $\mu_0$ or one of $\mu_1,\ldots,\mu_k$ (e.g., whether there exists a correlated pair of coordinates or not). However, our lower bounds apply to that task as well, with an almost identical proof. 
\end{remark}%
\begin{remark}[Data access] \label{remark:data-access}
	Our memory-based bounds assume that the algorithm performs one or more 
	passes over the data. An even weaker assumption might be that the algorithm 
	can access the data in an arbitrary order (i.e. has random access). 
	However, proving a super-linear (in dimension) memory lower bound in this 
	setting would imply a super-linear lower bound on the runtime of any 		random-access Turing machine, and unfortunately, this is related to difficult questions in computational complexity (see \citet[Section 1.2]{raz2016fast} for a related discussion).
\end{remark}

\subsection{Binary Vectors}\label{subsec:binary}

Having establishes our main technical results, we now turn to derive concrete bounds in the context of detecting correlations. In this subsection, we begin with the case of distributions over binary vectors, where the goal is to detect some unique (pairwise or higher-order) correlation. Concretely, fix some $0<\rho<1$, and define the sample space as $\Omega = \{-1,1\}^d$ for some $d \ge 2$. Let $\mathcal{I}$ be the set of all nonempty subsets of $\{1,\dots,d\}$. For any $I \in \mathcal{I}$, let $\mu_{I,\rho}$ be the distribution over $\Omega$ defined by 
\[
\mu_{I,\rho}((x_1,\dots,x_d)) = 2^{-d} (1 + \rho \prod_{i\in I} x_i).
\]
Namely, $\mu_{I,\rho}$ samples with probability $\frac{1}{2}(1+\rho)$ an element uniformly from all elements with an even number of $-1$ values in the coordinates corresponding to $I$ and with probability $\frac{1}{2}(1-\rho)$ it samples an element with an odd number of $-1$ values in $I$.
Note that $\mu_{I,\rho}$ encodes a unique correlated subset of indices in the following manner (the proof appears in \subsecref{sec:pr-biased}): 
\begin{lemma} \label{lem:biased-subset}
For any set $I' \in \mathcal{I}$, $I' \ne \emptyset$, it holds that
$
\mathbb{E}_{X\sim \mu_{I,\rho}} \prod_{i\in I'} X_i = \begin{cases}
\rho & I' = I \\
0 & I' \ne I
\end{cases}
$.
\end{lemma}

For any subset $\mathcal{U} \subseteq \mathcal{I}$ and $0 < \rho < 1$, let $\mathcal{P}_{\mathcal{U},\rho} = \{ \mu_{I,\rho} \colon I \in \mathcal{U}\}$.
We apply Theorems~\ref{thm:main} and \ref{thm:mem} on the problem of identifying an underlying distribution $\mu$, promised to belong to the family $\mathcal{P}_{\mathcal{U},\rho}$, to get communication and memory lower bounds (the proof appears in \subsecref{sec:pr-subset}).

\begin{theorem} \label{thm:subset-parity}
Fix some $\mathcal{U} \subseteq \mathcal{I}$ which satisfies $\lvert \mathcal{U} \rvert \ge 2$. Fix integers $m,n\geq 1$ such that $n \ge \lvert \mathcal{U} \rvert^6$, and a positive $\rho \le n^{-1/2} \ln^{-1/2} \lvert \mathcal{U} \rvert/C$, where $C$ is a numerical constant. Then any $(m,n)$ protocol identifying  $\mu \in \mathcal{P}_{\mathcal{U},\rho}$ with $1/3$ error has a communication complexity of at least
\[ \frac{\left\lvert\Ucal\right\rvert}{C \rho^2 n \log (\left\lvert\Ucal\right\rvert^2 /(n\rho^2))}. \]
\end{theorem}

For example, the case of detecting pairwise correlations corresponds to choosing $\mathcal{U} = \{ I \in \mathcal{I} \colon \lvert I \rvert = 2\}$. Since $|\mathcal{U}|=\binom{d}{2}=\Omega(d^2)$, this gives us a lower bound of $\tilde\Omega\left( \frac{d^2}{\rho^2 n}-m \right)$, or $\tilde\Omega\left( \frac{d^2}{\rho^2 n}\right)$. This is optimal up to logarithmic factors, as shown by the upper bound discussed in the introduction. More generally, for order-$r$ correlations (for some constant $r\geq 2$), we simply pick $\mathcal{U} = \{ I \in \mathcal{I} \colon \lvert I \rvert = r\}$, and since $|\mathcal{U}|=\binom{d}{r}=\Omega(d^r)$ in this case, the theorem implies a communication complexity of $\tilde\Omega\left( \frac{d^r}{\rho^2 n}\right)$. Again, this is tight up to logarithmic factors, using a straightforward generalization of the protocol for the pairwise case.

Next, we state the analogue of \thmref{thm:subset-parity} for the memory-constrained setting (derived from \thmref{thm:subset-parity} by the same communication-to-memory reduction discussed earlier):

\begin{theorem} \label{thm:mem-subset}
There exist numerical constants $C,C'>0$ such that the following holds. For any $\mathcal{U} \subseteq \mathcal{I}$ such that $\lvert \mathcal{U} \rvert \ge C'$ , any $\rho$ such that $0\le \rho\le \lvert \mathcal{U} \rvert^{-3} \ln^{-1/2} \lvert \mathcal{U} \rvert C^{-1}$, and any integers $t,s,\ell \ge 1$, it holds that any $(t,s,\ell)$-algorithm identifying $\mu \in \mathcal{P}_{\mathcal{U},\rho}$ with $1/3$ error satisfies
\[
ts\ell \ge \frac{\lvert \mathcal{U} \rvert}{C \ln \lvert\mathcal{U}\rvert \rho^2}.
\]
\end{theorem}

As a special case, the theorem implies that for detecting pairwise correlations, $ts\ell=\Omega(d^2/\rho^2)$, and for order-$r$ correlations, $ts\ell=\Omega(d^r/\rho^2)$. For example, assuming the number of passes $\ell$ is constant, it implies that we cannot successfully detect the correlation, unless either the memory is large (on order $d^r$), or the number of samples used is much larger than what is required without memory constraints (i.e. $\Omega(\log(d)/\rho^2)$) for any constant $r$). 

\begin{remark}[Constraints on problem parameters]\label{remark:rhosmall}
Theorem \ref{thm:mem-subset} requires the correlation $\rho$ to be sufficiently small compared to $|\Ucal|$. Such an assumption is necessary to get a strong lower bound: To see this, consider the case of detecting a pairwise correlation in binary vectors with memory constraints. If we can store $\tilde{O}(d/\rho^2)$ bits in memory, then we can simply collect and store $\tilde{O}(1/\rho^2)$ 
data points, and the empirical correlations in this data will reveal the true 
correlated coordinates with high probability. Thus, to prove an 
$\tilde{\Omega}(d^2)$ memory lower bound (as we do here), the correlation $\rho$ \emph{must} be smaller than $\tilde{O}(d^{-1/2})$.
Similarly, in a communication constrained setting, note that a communication budget of $\tilde{O}(d/\rho^2)$ bits enables the players to exchange $\tilde{O}(1/\rho^2)$ data points and find the correlation. Hence, in order to prove a communication lower bound of $\tilde{\Omega}\left(d^2/\left(n \rho^2\right)\right)$, one has to assume that $n = \tilde\Omega(d)$.
That being said, in the theorems above we require a stronger bounds on $\rho$ and $n$ than what these arguments imply. In Appendix \ref{sec:tuples}, we show that these requirements can be weakened to some extent,  for the case of $\mathcal{U} = \{ I \in \mathcal{I} \colon \lvert I \rvert = r\}$, $r\geq 2$. Precisely characterizing the parameter regimes where non-trivial lower bounds are possible is left to future work. 
\end{remark}



\subsection{Gaussian Distribution}\label{subsec:gaussian}

Having discussed distributions supported on binary vectors, we now turn to prove similar results for another cannonical family of distributions, namely Gaussian distributions on $\reals^d$. In what follows, we focus on pairwise correlations (since a multivariate Gaussian distribution is uniquely determined by its mean and covariance matrix, there is no sense in discussing higher-order correlations as in the binary case). 

Define $\mathcal{I}_2 = \left\{ S \subseteq [d] \colon \lvert S \rvert = 2\right\}$.
Fix some $d \ge 3$ and $0<\sigma<1$. For any set $I \in \mathcal{I}_2$, let $\eta_{I,\sigma}$ denote the zero-mean Gaussian distribution on $\mathbb{R}^d$, with covariance matrix $\Sigma_{I,\sigma}$ defined as follows:
\[
\Sigma_{I,\sigma}(i,j) = \begin{cases}
1 & i=j \\
\sigma & I = \{i,j\} \\
0 & \text{otherwise}.
\end{cases}
\]
In words, each individual coordinate has a variance of $1$, and each pair of distinct coordinates are uncorrelated, except for the pair $(i,j)$ with a correlation $\sigma$. Let $\mathcal{G}_\sigma=\{\eta_{I,\sigma} \colon I \in \mathcal{I}_2 \}$ be the set of all $\binom{d}{2}$ distributions defined this way. The following theorems are analogues of Theorems \ref{thm:subset-parity} and \ref{thm:mem-subset} for the case of pairwise correlations (the proof appears in \subsecref{sec:pr-normal}):


\begin{theorem} \label{thm:normal}
Fix some $n,m \ge 1$ and $0<\sigma<1$, such that $n \ge Cd^6$ for some numerical constant $C>0$ and $\sigma \le n^{-1/2} \ln^{-1/2} d \ln^{-1}(dnm/\sigma) / C$. Any $(m,n)$-protocol identifying $\eta \in \mathcal{G}_\sigma$ with $1/6$ error has a communication complexity of at least
\[
\frac{d^2}{C \sigma^2 \ln^2 (nmd/\sigma) \ln (d /(n\sigma^2)) n }~.
\]
\end{theorem}

\begin{theorem} \label{thm:normal-mem}
There exist numerical constants $C,C'>0$ such that the following holds. If $d\geq C'$, then for any $\sigma$ such that $ 0 < \sigma \le \left(C d^{3} \ln^{1/2} d \ln (d/\sigma)\right)^{-1}$ and any integers $t,s, \ell \ge 1$, it holds that any $(t,s,\ell)$-algorithm identifying  $\mu \in \mathcal{G}_\sigma$ with $1/6$ error satisfies
\begin{equation*}
ts\ell \ge \frac{d^2}{C \sigma^2 \ln^3 d \ln^2 (1/\sigma)}.
\end{equation*}
\end{theorem}

Whereas for binary vectors, our results are a direct corollary of Theorem~\ref{thm:main}, the proofs in the Gaussian case are more involved,  because no family of distinct Gaussian distributions satisfy the $\CD(\rho)$ property from Definition \ref{def:cd}. Instead, we need to work with truncated Gaussian distributions (which do satisfy this property), with some determinant calculations required to verify the conditions of Theorem~\ref{thm:main}. We then reduce the resulting bound on truncated distributions to non-truncated ones, to get \thmref{thm:normal}. \thmref{thm:normal-mem} is derived from \thmref{thm:normal} by the same communication-to-memory reduction discussed earlier.


\section{Proof Ideas}\label{sec:ideas}

In this section, we sketch the main ideas in the proof of \thmref{thm:main}, on which all our other results are based, and ignoring various technical issues. For simplicity, we discuss it in terms of the simpler problem of deciding whether the underlying distribution is $\mu_0$ or one of $\mu_1, \dots, \mu_k$ (as described in Remark~\ref{remark:id-vs-binary}). In particular, a successfull protocol for this problem should allow us to distinguish between $\mu_0$ and $\mu_i$, for all $i$, without knowing $i$ beforehand. The crux of our proof lies in showing that these $k$ tasks are ``essentially'' independent, in the sense that any protocol which solves all of them requires a communication of $\tilde\Omega(k)$ times the required communication for solving a single task. Formally, let $\Pi$ be the transcript (the aggregation of all messages sent) of the protocol; let $\mathbf{X} = \left(X^{(1)}, \dots, X^{(m)}\right)$ denote the $m$ (i.i.d.) sample sets given to the $m$ parties in the protocol, where $X^{(j)}$ is the input of party $j$; and let 
$P_{\Pi \mid \mathbf{X} \sim \mu_b^{mn}}$ denote the distribution of the transcript $\Pi$ conditioned on the inputs being distributed $\mu_i^{mn}$, for $i\in [k]$. It is easy to show that any protocol which successfully distinguishes between $\mu_0$ and $\mu_i$ must satisfy $d_{TV}\left( P_{\Pi \mid \mathbf{X} \sim \mu_0^{mn}}, P_{\Pi \mid \mathbf{X} \sim \mu_i^{mn}} \right) = \Omega(1)$, where $d_{TV}$ is total variation distance\footnote{The total variation distance between  two distributions with densities $p$ and $q$ is $\frac{1}{2}\int\left| p(x) - q(x)\right| dx$.}. In particular, this implies that 
\begin{equation} \label{eq:decomp}
\sum_{i=1}^k d_{TV}\left( P_{\Pi \mid \mathbf{X} \sim \mu_0^{mn}}, P_{\Pi \mid \mathbf{X} \sim \mu_i^{mn}} \right)^2
= \Omega(k)~.
\end{equation}
The proof proceeds by showing that the communication complexity (times an $\tilde{O}(n\rho^2)$ factor) upper bounds the left-hand side above, namely the sum of total variations over all $k$ individual tasks. This implies that the communication complexity is $\tilde{\Omega}(k/(n\rho^2))$ as required.

Intuitively, this assertion is true if the tasks are independent, so that information about one task does not convey information on another task. 
A concrete example (studied in \cite{shamir2014fundamental,steinhardt2015minimax,braverman2016communication}) is sparse mean estimation, where the goal is to distinguish a zero-mean product distribution on $\reals^k$, from similar product distributions where a few of the coordinates are slightly biased. Here, we can think of $\mu_i$ as the distribution where coordinate $i$ is slightly biased. Since this is a product distribution, statistics about one coordinate reveals no information about the statistics of other coordinates, so any single party has to send some information on all coordinates in order for a protocol to succeed -- hence the communication complexity must scale linearly with $k$. This idea lies at the heart of the papers mentioned above, and works well when the communication/budget is smaller than the dimension.

Unfortunately, this idea cannot be used as-is for showing lower bounds larger than the dimension. For example, in the context of pairwise correlations on $d$-dimensional data, an $\Omega(d^2)$ lower bound would require constructing a distribution over inputs $\bx$, so that if we consider the $d\times d$ matrix $\bx\bx'$, at least $\Omega(d^2)$ of its entries has a joint product distribution. But this is impossible, since this matrix is always of rank $1$, so no subset of more than $O(d)$ entries can be mutually independent. \citet{shamir2014fundamental}, which also studied correlations, circumvented this difficulty with an ad-hoc construction involving extremely sparse vectors, but as discussed in the introduction, the end result has several deficiencies. 

Our main technical contribution is to show how one can circumvent this hurdle, by relaxing the independence assumption to the milder technical assumption stated in \thmref{thm:main}, which only involves approximate uncorrelation and does apply to our problem. 

The proof proceeds by fixing a party $j$, and constructing a Markov chain
$\Pi \to X^{(j)} \to Y \to Z$, where $\Pi$ is the transcript of the protocol; $X^{(j)}$ is the data of party $j$, and $Z=(Z_1,\ldots,Z_k),Y=(Y_1,\ldots,Y_k)$ are carefully-constructed binary random vectors, defined as follows:
\begin{itemize}[leftmargin=*]
\item The transcript $\Pi$ is distributed as if the inputs of all players are drawn from $\mu_0$, and the input $X^{(j)}$ of player $j$ is distributed $\mu_0^n$.
\item The probability for each $Y_i$ to equal $1$ is a certain function of $\mu_i^n(X^{(j)})/\mu_0^n(X^{(j)})$.
\item $Z_i$ equals $Y_i$ after flipping it with probability $\frac{1}{2}-\tilde{\Theta}\left(\rho\sqrt{n}\right)$. Additionally, $X^{(j)}$ is distributed roughly $\mu_i^n$ conditioned on $Z_i=1$ (recall that $X^{(j)} \sim \mu_0^n$ unconditionally). Such a construction is possible from Bayes rule and the fact that $\mu_0^n(X^{(j)})$ and $\mu_i^n(X^{(j)})$ are close up to a multiplicative factor of $1 \pm \tilde{O}\left(\sqrt{n}\rho\right)$ for most values of $X^{(j)}$.
\end{itemize}
These random variables are constructed so that the following properties are satisfied:
\begin{itemize}[leftmargin=*]
\item $Y_1,\ldots,Y_k$ are approximately independent, in the sense that $\sum_{i=1}^{k}I(\Pi;Y_i)\leq \tilde{O}(1)\cdot I(\Pi;Y)$, where $I()$ denotes mutual information. Intuitively, this is due to a central limit phenomenon: If we consider the $k$ random variables $\frac{1}{\sqrt{n}}\log\frac{\mu_i^n\left(X^{(j)}\right)}{\mu_0^n\left(X^{(j)}\right)}$ for $i\in k$, they have an asymptotically Gaussian distribution as $n\rightarrow \infty$. Moreover, \eqref{eq:corr} in the theorem statement ensures that they are almost pairwise uncorrelated, but for Gaussian random variables, uncorrelation is equivalent to independence. Since each $Y_i$ is a function of the corresponding random variable, it follows that $Y_1,\ldots,Y_k$ are approximately independent for large enough $n$. 

\item Since $Z_i$ equals $Y_i$ after a nearly-unbiased random coin flip, and these are both binary random variables, one can show the strong data processing inequality\footnote{The data processing inequality states that for any Markov chain $U \to V \to W$, $I(U;W) \le I(U;V)$. In some cases, one can show a strong (strict) inequality, as the inequality used here.} $I(\Pi;Z_i)\leq O(n\rho^2)\cdot I(\Pi;Y_i)$. Combined with the previous item and the fact that $I(\Pi;Y)\leq I\left(\Pi;X^{(j)}\right)$ by the data processing inequality, we get that
\[
\sum_{i=1}^{k}I(\Pi;Z_i)~\leq~\tilde{O}(n\rho^2)\cdot I(\Pi;X^{(j)})~.
\]

\item The construction of the Markov chain as defined above implies that the distribution of the transcript $\Pi$, conditioned on $Z_i=1$, is close to the distribution of $\Pi$ conditioned on the input of party $j$ being drawn from $\mu_i^n$. In particular, $\helli^2(P_{\Pi}, P_{\Pi \mid X^{(j)} \sim \mu_i^n}) \approx \helli^2(P_{\Pi}, P_{\Pi \mid Z_i=1})$, where $\helli^2$ denotes the squared Hellinger distance\footnote{The squared Hellinger distance between random variables with densities $p$ and $q$ is $\frac{1}{2} \int\left( \sqrt{p(x)} - \sqrt{q(x)}\right)^2 dx$.}. Recall that unconditionally, $\Pi$ is distributed as if all input comes from $\mu_0$. By existing results (\citet[Lemma~6.2]{bar2004information}, \citet[Lemma~2]{braverman2016communication} and \cite{jayram2009hellinger}), we have that $\helli^2(P_{\Pi}, P_{\Pi \mid \mathbf{X} \sim \mu_i^{nm}}) ~\le~ O(1)\sum_{j=1}^m \helli^2(P_{\Pi}, P_{\Pi \mid X^{(j)} \sim \mu_i^n})$ as well as $\helli^2(P_{\Pi}, P_{\Pi \mid Z_i=1})\leq O(1)I(\Pi;Z_i)$. Together with the previous item, we get that
\begin{equation}\label{eq:sketchend}
\sum_{i=1}^{k}\helli^2(P_{\Pi}, P_{\Pi \mid \mathbf{X} \sim \mu_i^{nm}}) ~\le~ 
\tilde{O}(n\rho^2)\sum_{j=1}^{m} I(\Pi,X^{(j)}).
\end{equation}
\end{itemize}
Since $X^{(1)}, \dots, X^{(m)}$ are independent, the right-hand side of \eqref{eq:sketchend} is at most $\tilde{O}(n\rho^2)I(\Pi,\mathbf{X})$, which is at most $\tilde{O}(n\rho^2)$ times the communication complexity of the protocol. Also, the left-hand side of \eqref{eq:sketchend} can be shown to be at least $\sum_{i=1}^{k}d_{TV}(P_{\Pi},P_{\Pi \mid \mathbf{X} \sim \mu_i^{nm}})^2/2$, which by \eqref{eq:decomp}, is at least $\Omega(k)$. Combining everything, we get that the communication complexity is at least $\tilde{\Omega}\left(k/\left(n\rho^2\right)\right)$ as required.

\nocite{*}
\bibliographystyle{plainnat}

\appendix

\section{Proofs} \label{sec:proofs}

\subsection{Proof Preliminaries \label{sec:proof-prel}}

For any integers $0 \le a \le b$, define $\binom{b}{\le a} = \sum_{i=0}^a \binom{b}{i}$.

\subsubsection{Probability distances \label{sec:pr-dist}}

We will use two distance functions between probability distributions.

\begin{definition}
The \emph{total variation distance} between two probability measures with densities $p,q$  on the same sample space $\Omega$ is defined as
\[
d_{TV}(p,q) = \frac{1}{2} \int_{x \in \Omega} \lvert p(x) - q(x) \rvert dx = \sup_F \lvert P(F) - Q(F) \rvert
\]
where the supremum is over all events.
\end{definition}

\begin{definition}
The \emph{squared Hellinger distance} between two probability measures with densities $p$ and $q$ on the same sample space $\Omega$ is defined as
\[
\helli^2(p,q) = \frac{1}{2} \int_{x \in \Omega} \left( \sqrt{p(x)} - \sqrt{q(x)} \right)^2 dx.
\]
The \emph{Hellinger distance} is defined as $\helli(p,q) = \sqrt{\helli^2(p,q)}$.
\end{definition}

These distances are defined for discrete random variables in a similar manner. Both the total variation distance and the (non squared) Hellinger distance are $f$-divergences\footnote{An $f$-divergence: a function of two distributions $p$ and $q$ which can be written as $\int f(dp/dq)dq$ for a convex function $f$ with $f(1)=0$} which satisfy the triangle inequality. Additionally, these distances are polynomially equivalent:

\begin{proposition} \label{prop:heltv}
For any distributions $p$ and $q$,
\begin{equation} \label{eq:dist-ineq}
\helli^2(p,q) \le d_{TV}(p,q) \le \sqrt{2} \helli(p,q).
\end{equation}
\end{proposition}

\subsubsection{Data processing inequality}

We will frequently use the notation $P_X$ to denote the distribution of a random variable $X$, where $P_X(x)$ denotes $\Pr[X=x]$. Additionally, define a \emph{channel} $P_{Y | X}$ as a random function which gets as an input a member $x$ of some sample space and outputs a random $Y$ according to the distribution $P_{Y | X=x}$. We can compose a channel $P_{Y|X}$ over a distribution $P_X$ to get a new distribution, $P_{Y|X} \circ P_X$, the distribution over the output of the channel given that its input is distributed $P_X$. Similarly, we can compose channels together.

\begin{definition}
A \emph{Markov chain} $X_0 \to X_1 \to \cdots \to X_n$ consists of a distribution $P_{X_0}$ and channels $P_{X_1 | X_0}, \dots, P_{X_n | X_{n-1}}$. It induces a joint distribution $P_{X_0 \cdots X_n} = P_{X_0 (X_1 |X_0) \cdots (X_n | X_{n-1})}$ accordingly.
\end{definition}

The entropy of a random variable $X$ is denoted $H(X)$ and the mutual information between the random variables $X$ and $Y$ is defined as 
\[
I(X; Y) = H(X) - H(X \mid Y) = H(Y) - H(Y \mid X).
\]
The data processing inequality states that an information cannot increase while being transfered across a channel. It has two formulations: one in terms of mutual information and one in terms of $f$-divergence.

\begin{proposition}[Data processing inequality] \label{prop:DPI}
The following hold:
\begin{enumerate}
\item For any Markov chain $W \to X \to Y$, $I(W;Y) \le I(W;X)$.
\item Let $P_{X_1}$ and $P_{X_2}$ be distributions on the same sample space $\Omega$ and let $P_{Y|X}$ be a channel getting its input from $\Omega$. Then for any $f$-divergence $d_f$,
\[ d_f( P_{Y|X} \circ P_{X_1}, P_{Y|X} \circ P_{X_2}) \le d_f(P_{X_1}, P_{X_2}). \]
\end{enumerate}
\end{proposition}

\subsection{Proof of Theorem~\ref{thm:main}} \label{sec:main-pr}

Assume a sample space $\Omega \subseteq \{-1,1\}^k$, and assume a distribution $\mu_0$ over $\Omega$ which satisfies that for any $i \in \{1,\dots. k\}$, the probability for an element $x = (x_1,\dots,x_k)$ to satisfy $x_i = 1$ equals $1/2$. Given $0 < \rho < 1$, one can define the distributions $\mu_1, \dots, \mu_k$, where $\mu_i(x) = (1 + \rho x_i) \mu_0(x)$ for all $x$. We say that $\{ \mu_1, \dots, \mu_k \}$ is a \emph{binary centered familiy of distributions} (or $\BCD(\rho)$ for brevity), \emph{centered around $\mu_0$}.

We prove Theorem~\ref{thm:main} for all $\CD(\rho)$ families of distributions, however, it is sufficient to prove this theorem under a weaker condition on the distributions, namely, that the distributions are $\BCD(\rho)$: if Theorem~\ref{thm:main} is correct for all $\BCD(\rho)$ distributions then it is correct for all $\CD(\rho)$ distributions. This can be shown using a reduction: for every $\CD(\rho)$ family $\{\eta_1, \dots, \eta_k\}$,  there is a $\BCD(\rho)$ family $\{ \mu_1, \dots, \mu_k \}$ and a transformation $P_{\eta | \mu}$ transforming each $\mu_i$ to $\eta_i$, namely, $\eta_i = P_{\eta | \mu} \circ \mu_i$ for all $1 \le i \le k$. This transformation does not change the high order correlations: for all $S \subseteq [k]$, 
\begin{equation} \label{eq:corem}
\mathbb{E}_{X \sim \mu_0} \prod_{i \in S} \left(\frac{\mu_i(X)}{\mu_0(X)}-1\right) = \mathbb{E}_{Y \sim \eta_0} \prod_{i \in S} \left( \frac{\eta_i(Y)}{\eta_0(Y)}-1 \right),
\end{equation}
hence the condition \eqref{eq:corr} applies for $\{ \mu_1, \dots, \mu_k \}$ if and only if it applies for $\{ \eta_1, \dots, \eta_k \}$.
Given an input to the $\mu$-problem the parties can privately transform it to an $\eta$-input and simulate an $\eta$-protocol.

\begin{lemma} \label{lem:trans}
Let $\{ \eta_1, \dots, \eta_k \}$ be a $\CD(\rho)$ family. There exists a $\BCD(\rho)$ family $\{\mu_1, \dots, \mu_k \}$ and a channel $P_{\eta | \mu} \colon \Omega_\mu \to \Omega_\eta$ such that for all $1 \le i \le k$, $\eta_i = P_{\eta|\mu} \circ \mu_i$, where $\Omega_\mu$ and $\Omega_\eta$ are the sample spaces of the $\mu$-family and the $\eta$-family respectively. Additionally, \eqref{eq:corem} holds for all $S \subseteq [k]$.
\end{lemma}

\lemref{lem:trans} is prooved in \subsecref{sec:lem-trans}.
Assume for the rest of the proof that $\{\mu_1, \dots, \mu_k\}$ is a $\BCD(\rho)$ family of distributions.
We present two main lemmas.
In the first lemma, we assume a setting that there is just one party which gets some input $X \in \Omega^n$ and outputs $\Pi$. We bound the distance between the distribution of $\Pi$ conditioned on $X$ being distributed $\mu_0^n$ or $\mu_i^n$. The distance is bounded in terms of the amount of information that $\Pi$ reveals on $X$. This lemma contains the main technical contribution of this paper.

\begin{lemma} \label{lem:main}
Let $\{ \mu_1, \dots, \mu_k \}$ be a $\BCD(\rho)$ family on a sample space $\Omega \subseteq \{-1,1\}^k$ centered around $\mu_0$. Let $P_{\Pi \mid X}$ be some channel getting an input $X \in \Omega^n$. Under the assumptions of Theorem~\ref{thm:main} on $n$, $\rho$ and $k$,
\[
\sum_{i=1}^k \helli^2(P_{\Pi \mid X\sim \mu_0^n}, P_{\Pi \mid X \sim \mu_i^n})
\le C n \rho^2 \log(k^2/(n\rho^2))(I_{X \sim \mu_0^n}(\Pi ; X) + 1),
\]
for some numerical constant $C>0$.
\end{lemma}

\lemref{lem:main} is proved in \subsecref{sec:single}.
The next lemma utilizes results of \cite{jayram2009hellinger} and \cite{braverman2016communication} to show that \lemref{lem:main} implies \thmref{thm:main}. The tools developed in the prior work derive bounds on settings where there is just a single party who sends some output to settings with multiple communicating parties.

\begin{lemma} \label{lem:multi}
Let $\mu_0, \dots, \mu_k$ be probability distributions on the sample space $\Omega$ such that for every channel $P_{\Pi | X}$ with input in $\Omega^n$: 
\[
\sum_{i=1}^k \helli^2(P_{\Pi \mid X\sim \mu_0^n}, P_{\Pi \mid X \sim \mu_i^n})
\le \beta(I_{X \sim \mu_0^n}(\Pi ; X) + 1),
\]
for some $\beta > 0$. Then any $1/3$-error $(m,n)$ protocol identifying $\mu \in \{\mu_1, \dots, \mu_k \}$ has a communication complexity of at least $Ck/\beta$ for some numerical constant $C>0$.
\end{lemma}

\lemref{lem:multi} is proved in \subsecref{sec:multi}. Combining Lemma~\ref{lem:main} and Lemma~\ref{lem:multi} gives us the lower bound on the communication complexity in Theorem~\ref{thm:main} . To conclude the proof, it remains to prove that the condition stated at the end of the theorem is indeed sufficient for \eqref{eq:corr} to hold. This is shown in the following lemma:

\begin{lemma} \label{lem:nk6}
For any integer $\ell \ge 2$, if $n \ge k^{2(\ell+1)/(\ell-1)}$ then the sum of all terms in \eqref{eq:corr} corresponding to $\lvert S \rvert> \ell$ is at most $1/(2n)$.
\end{lemma}

\begin{proof}
Under the assumptions of the lemma,
\begin{align}
&\sum_{S \subseteq [k] \colon \lvert S \rvert \ge \ell+1} n^{-\lvert S \rvert/2} \rho^{-\lvert S \rvert} \left\lvert \mathbb{E}_{A \sim \mu_0} \prod_{i\in S} (\mu_i(A)/\mu_0(A) - 1) \right\rvert \nonumber\\
&\le\sum_{S \subseteq [k] \colon \lvert S \rvert \ge \ell+1} n^{-\lvert S \rvert/2} 
= \sum_{r = \ell+1}^k {\binom{k}{r}} n^{-r/2}
\le \sum_{r = \ell+1}^k \frac{k^r}{r!} n^{-r/2}
\le \frac{1}{n} \sum_{r = \ell+1}^k \frac{1}{r!}
\le \frac{1}{2n}, \label{eq:3333}
\end{align}
where the LHS of \eqref{eq:3333} follows from the definition of a $\CD(\rho)$ family: it always holds that $\lvert\mu_i(A)/\mu_0(A) -1 \rvert \le \rho$.
\end{proof}

\subsecref{sec:single} contains the proof of \lemref{lem:main}, \subsecref{sec:multi} contains the proof of \lemref{lem:multi} and \subsecref{sec:lem-trans} contains the proof of \lemref{lem:trans}.

\subsubsection{Proof of Lemma~\ref{lem:main} \label{sec:single}}

For the majority of our calculations we will assume that some high probability event holds. In what follows we give intuitive explanation about this event and why it holds and then more preceise definition and proof. Recall that the input $x \in \Omega^n$ contains $n$ samples from $\Omega = \{-1,1\}^k$ and define $x_{j,i}$ to be bit $i$ of sample $j$, for $1 \le i \le k$ and $1 \le j \le n$. Note that $x_{j,i} \in \{-1,1\}$, hence, for any $i$, any distribution $\mu$ over $\Omega$ and any $t > 0$, Hoeffding's bound implies that
\begin{equation} \label{eq:badevent}
\Pr_{x \sim \mu^n}\left[\left|\sum_{j=1}^n x_{j,i} - \mathbb{E}_{x \sim \mu^n}\left[\sum_{j=1}^n x_{j,i}\right]\right| > \sqrt{n} t \right] \le 2 e^{-t^2/2}.
\end{equation}
In particular, taking $t = \sqrt{2 \ln (k^2)}$ and performing a union bound over $i = 1, \dots, n$, one obtaines that with probability at least $1 - 2/k$, \eqref{eq:badevent} holds for all $i = 1,\dots, k$. We would like to replace $\mu$ by $\mu_{i'}$ for $i'=0,1,\dots,k$. Note that for any $i' = 0,1,\dots, k$, it holds that 
\[
\left| \mathbb{E}_{x \sim \mu_{i'}^n} \left[\sum_{j=1}^n x_{j,i}\right]\right| = \left|n \mathbb{E}_{y\sim\mu_{i'}}[y_i]\right|\le n \rho \le \sqrt{n},
\]
by definition of a $\BCD(\sigma)$ family of distributions and by the requrement $\rho \le \sqrt{n}$. 
This implies that for any $i' =0,1,\dots,k$,
\begin{equation}
\Pr_{x \sim {\mu_{i'}^n}}\left[x \in \mathcal{T}'' \right] \ge 1-2/k,
\end{equation}
where $\mathcal{T}''$ is the set of all $x \in \Omega^n$ which satisfies that for all $i \in \{1,\dots,k\}$, $\left|\sum_{j=1}^n x_{j,i} \right| \le \sqrt{n} \sqrt{2 \ln (k^2)}+\sqrt{n}$.

If $x \in \mathcal{T}''$ then for any $i'$, $\mu_{i'}^n(x)$ close to $\mu_0^n(x)$. Indeed, recall that for any $y \in \Omega$, $\mu_i(y)/\mu_0(y) = 1 + y_i \rho$. Hence,
\[
\mu_i^n(x) / \mu_0^n(x) 
= \prod_{j=1}^n (1 + \rho x_{j,i})
\approx 1 + \sum_{j=1}^n \rho x_{j,i}
= 1 \pm O(\rho \sqrt{n \log k}),
\]
and recall that $\rho = O(1/\sqrt{n \log k})$.
To conclude, there exists some set $\mathcal{T}''$ such that for any $i' \in \{0,1,\dots,k\}$, $\mu_{i'}(x \in \mathcal{T}'') \ge 1-2/k$ and for any $x \in \mathcal{T}''$ and any $i \in \{1,\dots,k\}$, $\mu_i^n(x)/\mu_0^n(x) = 1 \pm O(\sqrt{n \log k} \rho)$.

Next, we formalize the above intuition and prove a result which holds with a slightly higher probability. First, we can assume that the constant $C'$ in the statement of Theorem~\ref{thm:main} is sufficiently large such that
\begin{equation} \label{eq:99}
\rho \le \frac{1}{2 \sqrt{n} \left( 2 \sqrt{2\ln(8k^2)} + 3 \right)}.
\end{equation}
Denote by $\mathcal{T}$ the set of all samples $x \in \Omega^n$ such that for all $1 \le i \le n$,
\[
\left\lvert \frac{\mu_i^n(x)}{\mu_0^n(x)} -1 \right\vert
\le \alpha,
\]
where $\alpha$ is a positive number which satisfies the equation 
\begin{equation} \label{eq:4433}
\alpha = \left( 2 \sqrt{2\ln(2k^2/\alpha^2)} + 3 \right) \rho \sqrt{n}.
\end{equation}
In the next two lemmas, we will show that $\alpha = \tilde\Theta(\sqrt{n} \rho)$ and that additionally, for all $0 \le i' \le k$, $\mu_{i'}^n(\mathcal{T}^c) \le \alpha^2/k$. Hence, we change the above claim by replacing $1-2/k$ with $1-\alpha^2/k$.

\begin{lemma} \label{lem:alpha}
There is a unique positive number $\alpha$ which satisfies this equation and
\[
3 \sqrt{n} \rho \le \alpha \le \min\left\{ 1/2, \rho \sqrt{n} \left( 2 \sqrt{2\ln(2k^2/(9n \rho^2))} + 3 \right) \right\}.
\]
\end{lemma}

\begin{proof}
Such an $\alpha$ exists: if $\alpha\to 0$ then the RHS of \eqref{eq:4433} goes to $\infty$. If $\alpha =1/2$ then, from \eqref{eq:99}, the RHS of \eqref{eq:4433} equals
\[
\left( 2 \sqrt{2\ln(2k^2/\alpha^2)} + 3 \right) \rho \sqrt{n}
= \left( 2 \sqrt{2\ln(8k^2)} + 3 \right) \rho \sqrt{n}
\le 1/2.
\]
Hence, by the intermediate value theorem, there exists a value of $0<\alpha\le 1/2$ which satisfies the equation. Since the RHS is monotonically decreasing in $\alpha$ whenever $\alpha>0$, there is just one solution for $\alpha>0$. 

For the last inequalities, it holds from definition that $\alpha \ge 3 \rho \sqrt{n}$ and substituting $\alpha$ with $3 \rho \sqrt{n}$  in its definition implies that
\[
\alpha
= \left( 2 \sqrt{2\ln(2k^2/\alpha^2)} + 3 \right) \rho \sqrt{n}
\le \rho \sqrt{n} \left( 2 \sqrt{2\ln(2k^2/(9n \rho^2)} + 3 \right).
\]
\end{proof}


\begin{lemma} \label{lem:notinT}
For all $0 \le i' \le k$, $\mu_{i'}^n(\mathcal{T}^c) \le \alpha^2/k$.
\end{lemma}

\begin{proof}
Define $p = \frac{\alpha^2}{k}$ and $a = \sqrt{2\ln(2k/p)} + 1$ and let $\mathcal{T}'$ be the set of all $x \in \Omega^n$ such that for all $1 \le i \le k$, 
\[
\left\lvert \sum_{j=1}^n x_{j,i} \right\rvert \le \sqrt{n} a.
\]
The proof is divided into two main claims:
\begin{enumerate}
\item For all $0 \le i' \le k$, $\mu_{i'}^n\left(\left(\mathcal{T}'\right)^c\right) \le p$.
\item It holds that $\mathcal{T}' \subseteq \mathcal{T}$.
\end{enumerate}
This two claims suffice to conclude the proof. We start by proving the first claim.
Note that from definition of a $\BCD(\rho)$ family of distributions, for any $j \in [n]$ and $i \in [k]$, $\Pr_{A \sim \mu_0^n}[A_{j,i}=1]=\Pr_{A \sim \mu_0^n}[A_{j,i}=-1] = 1/2$. Next, note that for any $i \in [k]$, if $A \sim \mu_{i'}^n$, then by the definition of a $\BCD(\rho)$ family, $\Pr_{A \sim \mu_{i'}^n}[A_{j,i}=1] \le (1+\rho) \Pr_{A \sim \mu_0^n}[A_{j,i}=1] \le \frac{1}{2}(1+\rho)$, hence $\mathbb{E}_{A \sim \mu_{i'}^n}[A_{j,i}] \le \rho$, and similarly, $\mathbb{E}_{A \sim \mu_{i'}^n}[A_{j,i}] \ge -\rho$. We conclude that for any $0 \le i' \le k$, $j \in [n]$ and $i \in [k]$, $\left\lvert \mathbb{E}_{A \sim \mu_{i'}^n}[A_{j,i}] \right\rvert \le \rho$.
Hoeffding's inequality states that if $A_1, \dots, A_n$ are independent random variables getting values in $[-1,1]$, then for any $\beta > 0$,
\[
\Pr\left[ \left\lvert \sum_{j=1}^n A_j - \mathbb{E} \sum_{j=1}^n A_j \right\rvert > \sqrt{n} \beta \right]
\le 2 e^{-\beta^2/2}.
\]
Fix $0 \le i' \le k$ and $i \in [k]$, and let $A$ be a random variable distributed $\mu_{i'}^n$. Then,
\begin{align}
\Pr\left[ \left\lvert \sum_{j=1}^n A_{j,i} \right\rvert > \sqrt{n}a \right]
&\le \Pr\left[ \left\lvert \sum_{j=1}^n A_{j,i} - \mathbb{E}\sum_{j=1}^n A_{j,i} \right\rvert + \left\lvert \mathbb{E}\sum_{j=1}^n A_{j,i} \right\rvert> \sqrt{n}a  \right] \nonumber \\
&\le \Pr\left[ \left\lvert \sum_{j=1}^n A_{j,i} - \mathbb{E}\sum_{j=1}^n A_{j,i} \right\rvert > \sqrt{n}\sqrt{2\ln(2k/p)} \right] \label{eq:73}\\
&= \frac{p}{k}, \label{eq:74}
\end{align}
where \eqref{eq:73} follows from $\rho \le n^{-1/2}$ (see \eqref{eq:99}) which implies that $\left\lvert \mathbb{E} \sum_{j=1}^n A_{j,i} \right\rvert \le n \rho \le \sqrt{n}$; and \eqref{eq:74} follows from Hoeffding's inequality.
A union bound over $i \in [k]$ implies that
\[
\mu_{i'}^n \left( \left( \mathcal{T}'\right)^c \right)
\le \sum_{i =1}^k \Pr\left[ \left\lvert \sum_{j=1}^n A_{j,i} \right\rvert > \sqrt{n}a \right]
\le p.
\]

Next, we show that $\mathcal{T}' \subseteq \mathcal{T}$. Fix $x \in \mathcal{T}'$ and $i \in [k]$, and we will show that $\left\lvert \mu_i^n(x)/\mu_0^n(x) -1 \right\vert\le \alpha$ to conclude the proof.
Note that
\begin{equation} \label{eq:3712}
\rho \sqrt{n} a 
\le \left( \sqrt{2\ln(2k^2/\alpha^2)} + 1 \right) \rho \sqrt{n}
\le \alpha/2
\le 1/4,
\end{equation}
where the second inequality follows from the definition of $\alpha$ and the third inequality follows from the bound $\alpha\le 1/2$ which 
Let $\ell = \left\lvert \sum_{j=1}^n x_{j,i} \right\rvert$ and $b \in \{-1,1\}$ be the sign of $\sum_{j=1}^n x_{j,i}$ ($b=1$ if the sum equals zero). By definition of $\mathcal{T}'$, $\ell \le \sqrt{n}a$. There are $\frac{n+\ell}{2}$ values of $j$ for which $x_{j,i}=b$ and $\frac{n-\ell}{2}$ values for which $x_{j,i} = -b$. 
It holds that
\begin{align}
\frac{\mu_i^n(x)}{\mu_0^n(x)} 
&= (1+b\rho)^{(n+\ell)/2} (1-b\rho)^{(n-\ell)/2} \label{eq:502}\\
&= (1 - \rho^2)^{(n-\ell)/2} (1+ b\rho)^\ell \notag\\
&\le (1+ \rho)^\ell 
~\le~ (1 + \rho)^{\sqrt{n}a} 
~\le~ e^{\rho \sqrt{n}a} 
~\le~ 1 + 2 \rho \sqrt{n}a, \notag
\end{align}
where \eqref{eq:502} follows from the fact that by definition of a $\BCD(\rho)$ family, for any $x \in \Omega$, $\mu_i(x)/\mu_0(x) = 1 + \rho x_i$; one before the last inequality follows from $1+s \le e^s$ for all $s \in \mathbb{R}$;
and the last inequality follows from $e^s \le 1 + 2s$ for all $0 \le s \le 1$ and the from \eqref{eq:3712}.
Bounding from below,
\begin{align}
\frac{\mu_i^n(x)}{\mu_0^n(x)} 
&= (1 - \rho^2)^{(n-\ell)/2} (1+ b\rho)^\ell 
~\ge~ (1 - \rho^2)^{n/2} (1 - \rho)^{\sqrt{n}a} \notag\\
&\ge 1 - \rho^2 n/2 - \rho \sqrt{n} a 
~\ge~ 1 - \rho \sqrt{n} (a + 1/2), \label{eq:501}
\end{align}
where the last inequality follows from $\rho \le n^{-1/2}$ which follows from \eqref{eq:99}.
In conclusion, \eqref{eq:3712} and \eqref{eq:501} imply that 
\begin{equation} \label{eq:55}
\left\lvert \frac{\mu_i^n(x)}{\mu_0^n(x)} -1 \right\rvert
\le 2 a \rho \sqrt{n}
\le \alpha,
\end{equation}
where the last inequality follows from \eqref{eq:3712}. This confirms that $\mathcal{T}' \subseteq \mathcal{T}$ as required.
\end{proof}

To give intuition for the next part of the proof, assume the false assumption that $\left| \mu^n_i(x)/ \mu_0^n(x) - 1\right| \le \alpha$ for all $x \in \Omega^n$ (instead of only when $x \in \mathcal{T}$). Define a Markov chain $X \to Y \to Z$ as follows: first, $X$ is drawn from $\mu_0^n$. Then, given $X$, $Y = (Y_1,\dots,Y_k) \in \{-1,1\}^k$ is drawn such that
\[
\Pr\left[ Y_i = 1 \mid X \right]
= \frac{1}{2} + \frac{\mu_i^n(X)/\mu_0^n(X)-1}{4\alpha}
\]
and each bit of $Y$ is distributed independently conditioned on $X$. Note that due to the assumption $|\mu_i^n(X)/\mu_0^n(x)-1|\le \alpha$ it holds that $0 \le \Pr[Y_i=1 \mid X] \le 1$ as required. Next, we define $Z = (Z_1,\dots,Z_k) \in \{-1,1\}^n$ as follows: conditioned on $Y$, each bit $Z_i$ equals $Y_i$ with probability $\frac{1}{2}(1+2\alpha)$ and otherwise $Z_i = -Y_i$; additionally, the bits of $Z$ are independent conditioned on $Y$. A simple calculation shows that for any $X$, $\Pr[Z_i = 1,\ X] = \mu_i^n(X)/2$. Summing over $X$, one obtains that $\Pr[Z_i = 1] = 1/2$. Using Bayes' rule, $\Pr[X \mid Z_i = 1] = \mu_i^n(X)$. To sum up, one obtains the following properties:
\begin{itemize}
\item The random variable $X$ is distributed $\mu_0^n$. Conditioned on $Z_i = 1$, $X$ is distributed $\mu_i^n$.
\item The random variable $Z$ is uniform.
\item The random variable $Z_i$ is a noisy version of $Y_i$.
\end{itemize}

Due to the fact that $\left| \mu^n_i(x)/ \mu_0^n(x) - 1\right| \le \alpha$ only for $X \in \mathcal{T}$, one cannot define the channel $Y \mid X$ as defined above, or otherwise it will not hold that $0 \le \Pr[Y_i=1 \mid X] \le 1$. Hence, we change the definition of $P_{Y\mid X}$. Define the function $\psi  \colon \mathbb{R} \to \mathbb{R}$ by
\begin{equation} \label{eq:psi}
\psi(s) = \begin{cases}
-1 		& s \le -1 \\
s		& -1 \le s \le 1 \\
1		& s \ge 1
\end{cases}.
\end{equation}
The function $\psi$ should be viewed as ``the identity except for some exceptional cases'', where the exceptional cases correspond to $X \notin \mathcal{T}$, as will be clear next. Define the channel $P_{Y \mid X}$ as follows: given $X$, each coordinate of $Y$ is set independently to $-1$ or $1$, where for any coordinate $i$ \footnote{We assume $\mu_0$ has full support, otherwise we can remove from $\Omega$ all elements $x$ with $\mu_0(x) = 0$: by definition of a $\BCD(\rho)$ family, for all $1\le i \le k$ it also holds that $\mu_i(x)=0$.},
\begin{equation} \label{eq:9}
P_{Y|X}(Y_i \mid X) = \frac{1}{2} + \frac{1}{4} Y_i \psi\left(\frac{\mu^n_i(X) / \mu^n_0(X) - 1}{\alpha}\right).
\end{equation}
Note that for $X \in \mathcal{T}$, the function $\psi$ behaves as the identity and we obtain the previous definition of $P_{Y\mid X}$.
The following lemma characterizes the joint distribution $P_{XYZ}$, which satisfies an approximate version of the desired properties listed above.

\begin{lemma} \label{lem:bias}
The following holds for the distribution $P_{XYZ}$:
\begin{enumerate}
\item $ P_{XY}(X,Y) = 2^{-k} \mu^n_0(X) \prod_{i=1}^k \left( 1 + \frac{1}{2} Y_i \psi\left(\frac{\mu^n_i(X) / \mu^n_0(X) - 1}{\alpha}\right) \right)$.
\label{itm:bias1}
\item $ P_{XZ}(X,Z) = 2^{-k} \mu^n_0(X) \prod_{i=1}^k \left( 1 + \alpha Z_i \psi\left(\frac{\mu^n_i(X) / \mu^n_0(X) - 1}{\alpha}\right) \right)$.
\label{itm:bias2}
\item For all $x \in \mathcal{T}$, $P_{XZ_i}(X,1) = \mu^n_i(x)/2$. \label{itm:bias3}
\item For all $1 \le i \le k$:
\[ 
\left\lvert P_{Z_i}(1) - \frac{1}{2} \right\rvert 
\le \max_{0 \le i \le k} \mu_i^n(\mathcal{T}^c)
\le \frac{\alpha^2}{k}.
\]
\label{itm:bias4}
\item For all $1 \le i \le k$:
\[  \left\lvert P_{Y_i}(1) - \frac{1}{2} \right\rvert 
= \left\lvert P_{Z_i}(1) - \frac{1}{2} \right\rvert / (2 \alpha)
\le \frac{\alpha}{2k}.\]
\label{itm:bias5}
\end{enumerate}
\end{lemma}

Before proving this lemma we will prove an auxiliary lemma.
\begin{lemma} \label{lem:aux9}
Let $A \to B$ be a Markov chain, where $A,B \in \{-1,1\}$ are binary random variables. Assume that $P_A(1) = (1 + a)/2$ and assume that $P_{B | A}$ is a channel that flips its input with probability $(1-b)/2$ for some $a,b \in [-1,1]$. Then $P_B(B) = (1+B a b)/2$.
\end{lemma}

\begin{proof}
The proof is by calculation:
\[
P_B(1) 
= P_{AB}(1,1) + P_{AB}(-1,1)
= (1+a)(1+b)/4 + (1-a)(1-b)/4
= (1+ab)/2.
\]
Additionally, $P_B(-1) = 1 - P_B(1) = (1-ab)/2$.
\end{proof}

\begin{proof}[Proof of Lemma~\ref{lem:bias}]

We will prove the lemma items one by one. The first item follows from definition of $X \to Y$.
For proving the second item, fix some $x \in \Omega^n$, and note that conditioned on $X=x$, each $Y_i$ is binary as defined in \eqref{eq:9}. It holds that $P_{Z_i | Y_i}$ is a channel that flips its input $Y_i$ with probability $(1-2\alpha)/2$, therefore applying Lemma~\ref{lem:aux9} with $P_A = P_{Y_i|X=x}$, $P_{B|A}=P_{Z_i|Y_i}$, $a = \frac{1}{2} \psi\left(\frac{\mu^n_i(X) / \mu^n_0(X) - 1}{\alpha}\right)$ and $b=2\alpha$, we get that
\begin{equation} \label{eq:912}
P_{Z_i \mid X}(Z_i \mid x)
= \frac{1}{2}\left( 1 + \alpha Z_i \psi\left(\frac{\mu_i^n(x) / \mu_0^n(x) - 1}{\alpha}\right) \right).
\end{equation}
Note that the bits of $Z$ are independent conditioned on $X$: bits of $Y_i$ are independent conditioned on $X$ and each $Z_i$ depends only on $Y_i$. 
Hence,
\begin{equation}\label{eq:503}
P_{XZ}(X,Z)
= P_X(X) P_{Z\mid X}(Z \mid X)
= P_X(X) \prod_{i=1}^k P_{Z_i\mid X}(Z_i \mid X),
\end{equation}
and the second item follows from \eqref{eq:912}, \eqref{eq:503} and the fact that $P_X(X) = \mu_0^n(X)$ by definition.

The third item is proved as follows:
\begin{align}
P_{XZ_i}(x,1)   
&= P_X(x) P_{Z_i|X} (1|x) \notag\\
&= \mu_0^n(x) \frac{1}{2} \left( 1 + \alpha \psi\left(\frac{\mu_i^n(x) / \mu_0^n(x) - 1}{\alpha}\right) \right) \label{eq:3716} \\
&= \mu_0^n(x) \frac{1}{2}\left( 1 + \alpha \left(\frac{\mu_i^n(x) / \mu_0^n(x) - 1}{\alpha}\right) \right)\label{eq:377}\\
&= \frac{1}{2} \mu_i^n(x), \notag
\end{align}
where \eqref{eq:3716} follows from the fact that $X \sim \mu_0^n$ by definition of $X$ and from \eqref{eq:912}; and \eqref{eq:377} follows from the fact that whenever $X \in \mathcal{T}$, $ \left\lvert \mu_i^n(X)/\mu_0^n(X) - 1 \right\rvert \le \alpha$ by definition of $\mathcal{T}$, hence by definition of $\psi$,
\[\psi\left( \frac{\mu^n_i(X) / \mu^n_0(X) - 1}{\alpha} \right) = \frac{\mu^n_i(X) / \mu^n_0(X) - 1}{\alpha}.\]

To prove the fourth item,
\begin{equation}\label{eq:504}
P_{Z_i}(1)
\ge \sum_{x \in \mathcal{T}} P_{Z_i X} (1, x)
= \frac{1}{2} \mu_i^n(\mathcal{T})
= \frac{1}{2} - \frac{1}{2} \mu_i^n(\mathcal{T}^c),
\end{equation}
where the first equation follows from the third item. Additionally,
\begin{align}
P_{Z_i}(1) 
&= \sum_{x \in \mathcal{T}} P_{Z_i X} (1, x) + \sum_{x \notin \mathcal{T}} P_{Z_i X} (1, x) \notag\\
&\le \sum_{x \in \mathcal{T}}\frac{1}{2} \mu_i^n(x) 
+ \sum_{x \notin \mathcal{T}} P_{X}(x) \label{eq:378}\\
&= \frac{1}{2} \mu_i(\mathcal{T}) + \sum_{x \notin \mathcal{T}} \mu_0^n(x) \label{eq:379} \\
&\le \frac{1}{2} + \mu_0^n(\mathcal{T}^c), \label{eq:505}
\end{align}
where \eqref{eq:378} follows from the third item and \eqref{eq:379} follows from the fact that $X \sim \mu_0^n$ by definition. 
\eqref{eq:504} and \eqref{eq:505} imply that $\left\lvert P_{Z_i}(1)-1/2\right\rvert \le \max_{0\le i \le k} \mu_i^n\left( \mathcal{T}^c\right)$, and the fourth item follow from \lemref{lem:notinT}.

The fifth item follows from Lemma~\ref{lem:aux9} and the fact that $P_{Z_i \mid Y_i}$ flips its input with probability $(1-2\alpha)/2$: substitute $A=Y_i$, $B=Z_i$, $P_{Y_i}(1) = \frac{1}{2}(1+a)$ and $b = 2\alpha$. The lemma implies that $P_{Z_i}(1) = \frac{1}{2}(1+ab)$. Hence, 
\[
P_{Y_i}(1) - \frac{1}{2}
= \frac{a}{2}
= \frac{ab}{2} \frac{1}{b}
= \left( P_{Z_i} - \frac{1}{2} \right) \frac{1}{2\alpha}.
\]
\end{proof}

Next, we claim that the coordinates of $Y$ are almost independent. An intuitive explanation was given in Section~\ref{sec:ideas}, using the central limit theorem. However, due to the slow convergence guarantees of the central limit theorem, we did not find how to apply it without requiring $\rho$ to be exponentially small in $k$. Hence, we have an ad-hoc proof. It defines two auxiliary random variables, $X'$ and $Y'$. The variable $X'$ is uniform on $\{-1,1\}^k$, and in particular, its coordinates are independent. The random variable $Y'$ is constructed from $X'$ the same way that $Y$ is constructed from $X$. Due to the fact that $Y'_i$ depends only on $X'_i$, the coordinates of $Y'$ are also independent. We compare the distribution of $Y$ with the distribution of $Y'$ and show that if the high-order correlations between the coordinates of $X$ are low, then the distribution of $Y$ is similar to the distribution of $Y'$. Assumption \eqref{eq:corr} of Theorem~\ref{thm:main} assures that these higher order correlations are low. These claims are stated formally in the next lemma, where we prove that the entropy of $Y$ is almost the entropy of a random variable uniform over $\{-1,1\}^k$.

\begin{lemma} \label{lem:ent}
There exists some absolute constant $C$ such that
\[
H(Y) \ge k - C.
\]
\end{lemma}

\begin{proof}
We will show that for all $y \in \{-1,1\}^k$, $P_{Y}(y) \le \frac{C'}{2^k}$ for some numerical constant $C'>0$. This will imply that
\begin{equation}\label{eq:522}
H(Y)
= \sum_{y \in \{-1,1\}^k} P_Y(y) \log \frac{1}{P_Y(y)}
\ge \sum_{y \in \{-1,1\}^k} P_Y(y) \log \frac{2^k}{C'}
= \log \frac{2^k}{C'}
= k - \log C'
\end{equation}
and complete the proof.

First, we give an equivalent definition to the channel $P_{Y|X}$ (note the original definition is in \eqref{eq:9}):

\begin{equation} \label{eq:YXnew}
P_{Y_i \mid X}(y_i | x) = \frac{1}{2} \left( 1 + \frac{y_i}{2} \psi \left( \frac{1}{\alpha} \left( \prod_{j=1}^{n} (1 + x_{j,i} \rho) - 1 \right) \right) \right),
\end{equation}

where all bits of $Y$ are drawn independently given $X$. This definition is obtained from the original definition by substituting $\mu_i^n(x)/\mu_0^n(x)$ with $\prod_{j=1}^{n} (1 + x_{j,i} \rho)$. Indeed,
\[
\mu_i^n(x)/\mu_0^n(x)
= \prod_{j=1}^n \mu_i(x_j)/\mu_0(x_j)
= \prod_{j=1}^n (1+ x_{j,i}\rho),
\]
where the last inequality follows from the definition of a $\BCD(\rho)$ family, which requires that if $w=(w_1,\dots,w_k) \in \Omega$  then $\mu_i(w)/\mu_0(w) = 1 + \rho w_i$. We will use this definition of $P_{Y|X}$ in this lemma since it depends only on $X$ and does not depend on $\mu_0,\dots, \mu_k$.

Fix some $y=(y_1, \dots, y_k) \in \{-1,1\}^k$. Let $\{ \mu_1', \dots, \mu_k'\}$ be the $\BCD(\rho)$ family of distributions such that $\mu_0'$, its corresponding $\mu_0$ distribution, is uniform over $\{-1,1\}^k$ and $\mu_1', \dots, \mu_k'$ are derived from $\mu_0'$ as in the definition of a $\BCD(\rho)$ family: for all $1\le i \le k$ and for all $(w_1,\dots, w_k)\in \{-1,1\}^k$, 
\[ \mu_i((w_1, \dots, w_k)) = \mu_0'((w_1,\dots,w_k)) (1+w_i \rho)
= 2^{-k} (1+w_i\rho). \]
Define $X'$ and $Y'$ to be analogous to $X$ and $Y$ with respect to this family: $X'\sim (\mu_0')^n$ and $P_{Y'} = P_{Y|X} \circ P_{X'}$, using the new definition of $P_{Y|X}$ from \eqref{eq:YXnew}. Since $Y'_i$ is a function of the $i$'th column of $X'$ and the columns of $X'$ are independent, $Y'_1, \dots, Y'_k$ are independent. Item 5 of Lemma~\ref{lem:bias} and \lemref{lem:alpha} show that $\lvert P_{Y_i'}(1) - 1/2 \rvert \le \frac{\alpha}{2k}\le \frac{1}{2k}$ for all $1 \le i \le k$ \footnote{Note that this item proves a corresponding statement on $Y_i$, however, we can also substitute it with $Y_i'$: if we substitute $\mu_1,\dots,\mu_k$ with $\mu'_1,\dots,\mu'_k$, all the requirements of Theorem~\ref{thm:main} are satisfied: the only assumption on the family of distributions  is \eqref{eq:corr}, which the new family $\mu_1',\dots,\mu_k'$ satisfies, but since we haven't used this assumption yet, Lemma~\ref{lem:bias} applies to $Y'_i$ even without requiring the new family to satisfy \eqref{eq:corr}.}, therefore
\begin{equation} \label{eq:521}
P_{Y'}(y) = \prod_{i=1}^k P_{Y'_i}(y_i)
\le \frac{1}{2^k}\left( 1 + \frac{1}{k} \right)^k
\le \frac{e}{2^k}.
\end{equation}
We will bound $P_Y(y) / P_{Y'}(y)$ to complete the proof.

Recall that each row of $X$ is a vector distributed according to $\mu_0$ and each row of  $X'$ is a vector distributed according to $\mu'_0$. Define intermediate random variables $X^{(0)}, X^{(1)}, \dots, X^{(n)}$ such that for all $0 \le j \le n$, rows $1$ to $j$ of $X^{(j)}$ are distributed according to $\mu_0$ and rows $j+1$ to $n$ are distributed according to $\mu_0'$, where all rows are independent. Define the random variables $Y^{(0)}, \dots, Y^{(n)} \in \{-1,1\}^k$ accordingly, namely $Y^{(j)} \sim P_{Y \mid X} \circ P_{X^{(j)}}$. It holds that $Y^{(0)}$ has the same distribution as $Y'$ and $Y^{(n)}$ is distributed the same as $Y$.

Fix some $1 \le \ell \le n$ and we will bound $P_{Y^{(\ell)}}(y) / P_{Y^{(\ell-1)}}(y)$. Let $X^{(\ell)}_\ell$ be column $\ell$ of $X^{(\ell)}$, let $X_{-\ell}^{(\ell)}$ be $X^{(\ell)}$ without column $\ell$ and define $X^{(\ell-1)}_\ell$ and $X^{(\ell-1)}_{-\ell}$ similarly. 
Fix some $x_{-\ell} \in \{-1,1\}^{(n-1) \times k}$. For all $i \in [k]$, let 
\begin{align}
p_i 
&= \Pr\left[Y^{(\ell)}_i = y_i \middle| X^{(\ell)}_{-\ell} = x_{-\ell}\right] \notag \\
&= \sum_{b\in\{-1,1\}} \Pr\left[X^{(\ell)}_{\ell,i} = b \middle| X^{(\ell)}_{-\ell} = x_{-\ell} \right] 
\Pr\left[Y^{(\ell)}_i = y_i \middle| X^{(\ell)}_{-\ell} = x_{-\ell},\ X^{(\ell)}_{\ell,i} = b \right] \notag \\
&= \sum_{b\in\{-1,1\}} \Pr\left[X^{(\ell)}_{\ell,i} = b \right] 
\Pr\left[Y^{(\ell)}_i = y_i \middle| X^{(\ell)}_{-\ell} = x_{-\ell},\ X^{(\ell)}_{\ell,i} = b \right] \label{eq:506} \\
&= \frac{1}{2} \Pr\left[Y^{(\ell)}_i = y_i \middle| X^{(\ell)}_{-\ell} = x_{-\ell},\ X^{(\ell)}_{\ell,i} = -1 \right]
+ \frac{1}{2} \Pr\left[Y^{(\ell)}_i = y_i \middle| X^{(\ell)}_{-\ell} = x_{-\ell},\ X^{(\ell)}_{\ell,i} = 1 \right], \label{eq:507}
\end{align}
where \eqref{eq:506} follows from the fact that the rows of $X^{(\ell)}$ are independent, and \eqref{eq:507} follows from the fact that $X^{(\ell)}_\ell$ is distributed $\mu_0$, and by definition of a $\BCD(\rho)$ family, each bit is uniform under $\mu_0$. Recall that $Y^{(\ell)} = P_{Y\mid X} \circ X^{(\ell)}$. It holds that $p_i = \Pr\left[Y^{(\ell)}_i = y_i \middle| X^{(\ell)}_{-\ell} = x_{-\ell}\right]>0$ by definition of $P_{Y\mid X}$. Hence, one can define
\[ \delta_i = \Pr\left[Y^{(\ell)}_i = y_i \middle| X^{(\ell)}_{-\ell} = x_{-\ell},\ X^{(\ell)}_{\ell,i} = 1 \right] / p_i - 1, \]
which implies that
\begin{equation} \label{eq:509}
\Pr\left[Y^{(\ell)}_i = y_i \middle| X^{(\ell)}_{-\ell} = x_{-\ell},\ X^{(\ell)}_{\ell,i} = 1 \right]
= p_i\left(1 + \delta_i \right).
\end{equation}
By \eqref{eq:509} and by \eqref{eq:507}, for any value of $X^{(\ell)}_{\ell,i} \in \{-1,1\}$,
\begin{equation} \label{eq:510}
\Pr\left[Y^{(\ell)}_i = y_i \middle| X^{(\ell)}_{-\ell} = x_{-\ell},\ X^{(\ell)}_{\ell,i} \right]
= p_i\left(1 + X_{\ell,i}^{(\ell)} \delta_i \right).
\end{equation}
Furthermore, since $Y^{(\ell)}_i$ depends only on column $i$ of $X^{(\ell)}$, for any value of $X^{(\ell)}_\ell \in \{-1,1\}^k$,
\begin{equation*} 
\Pr\left[Y^{(\ell)}_i = y_i \middle| X^{(\ell)}_{-\ell} = x_{-\ell},\ X^{(\ell)}_\ell \right]
= p_i\left(1 + X_{\ell,i}^{(\ell)} \delta_i \right).
\end{equation*}
Since the bits of $Y^{(\ell)}$ are independent conditioned on $X^{(\ell)}$,
\begin{equation} \label{eq:511} 
\Pr\left[Y^{(\ell)} = y \middle| X^{(\ell)}_{-\ell} = x_{-\ell},\ X^{(\ell)}_\ell \right]
= \prod_{i=1}^k p_i\left(1 + X_{\ell,i}^{(\ell)} \delta_i \right).
\end{equation}
Hence,
\begin{align}
\Pr\left[Y^{(\ell)} = y \middle| X^{(\ell)}_{-\ell} = x_{-\ell}\right]
&= \mathbb{E}_{X^{(\ell)}_\ell} \left[ \Pr\left[Y^{(\ell)} = y \middle| X^{(\ell)}_{-\ell} = x_{-\ell},\ X^{(\ell)}_\ell \right] \right] \notag\\
&~=~ \mathbb{E}_{X^{(\ell)}_\ell} \prod_{i=1}^k (p_i (1 + X^{(\ell)}_{\ell,i} \delta_i)) \label{eq:516}\\
&= \left(\prod_{i=1}^k p_i \right) \mathbb{E}_{A =(A_1, \dots, A_k) \sim \mu_0} \prod_{i=1}^k (1 + A_i \delta_i)\label{eq:512} \\
&~=~ \left(\prod_{i=1}^k p_i \right) \sum_{S \subseteq [k]}  \mathbb{E}_{A=(A_1, \dots, A_k) \sim \mu_0} \left[\prod_{i\in S} A_i \delta_i \right] \nonumber\\
&= \left(\prod_{i=1}^k p_i \right) \sum_{S \subseteq [k]} \mathbb{E}_{A=(A_1, \dots, A_k) \sim \mu_0} \left[\prod_{i\in S} \left(\mu_i(A)/\mu_0(A) - 1\right) \delta_i/\rho \right].  \label{eq:34}
\end{align}
where \eqref{eq:516} follows from \eqref{eq:511}, \eqref{eq:512} follows from the fact that $X_\ell^{(\ell)} \sim \mu_0$ by definition of $X^{(\ell)}$, and \eqref{eq:34} follows from the fact that by definition of a $\BCD(\rho)$ family,  $\mu_i(A) = \mu_0(A)(1 + A_i \rho)$.
It holds that
\begin{align} 
& \left| \delta_i p_i \right|
= \frac{1}{2} \left\lvert \Pr\left[Y^{(\ell)}_i = y_i \middle| X^{(\ell)}_{-\ell} = x_{-\ell},\ X^{(\ell)}_{\ell,i} = 1 \right]
- \Pr\left[Y^{(\ell)}_i = y_i \middle| X^{(\ell)}_{-\ell} = x_{-\ell},\ X^{(\ell)}_{\ell,i}=-1\right]\right\rvert \label{eq:513} \\
&= \frac{1}{8} \left\lvert \psi \left( \frac{1}{\alpha} \left( (1 + \rho)\prod_{1 \le j \le n, j \ne \ell} (1 + (x_{-\ell})_{j,i} \rho) -1 \right) \right)
- \psi \left( \frac{1}{\alpha} \left( (1 - \rho)\prod_{1 \le j \le n, j \ne \ell} (1 + (x_{-\ell})_{j,i} \rho) -1 \right) \right) \right\rvert. \label{eq:145}
\end{align}
where \eqref{eq:513} follows from \eqref{eq:510} and \eqref{eq:145} follows from the fact that $Y^{(\ell)} = P_{Y\mid X} \circ X^{(\ell)}$ and from the definition of $P_{Y \mid X}$.
If 
\[
\frac{1}{\alpha} \left( (1 - \rho)\prod_{1 \le j \le n, j \ne \ell} (1 + (x_{-\ell})_{j,i} \rho) -1 \right) \ge 1
\]
then by \eqref{eq:145} and by definition of $\psi$ in \eqref{eq:psi},
\begin{equation}\label{eq:514}
\lvert \delta_i p_i \rvert = \frac{1}{8}\left| 1 - 1 \right| = 0.
\end{equation}
Otherwise, 
\[
\prod_{1 \le j \le n, j \ne \ell} (1 + (x_{-\ell})_{j,i} \rho)
\le \frac{\alpha + 1}{1-\rho}
\le 3,
\]
since $\rho \le \alpha \le 1/2$ by \lemref{lem:alpha}. Since $\psi$ is $1$-Lipschitz, \eqref{eq:145} is at most
\begin{equation} \label{eq:515}
\frac{1}{8 \alpha}((1+ \rho) - (1 - \rho)) \prod_{1 \le j \le n, j \ne \ell} (1 + (x_{-\ell})_{j,i} \rho)
\le \frac{6 \rho}{8 \alpha}.
\end{equation}
By \eqref{eq:514} and \eqref{eq:515}, we conclude that $|\delta_i p_i| \le 3\rho/(4\alpha)$. It holds that $p_i = \Pr\left[Y^{(\ell)}_i = y_i \middle| X^{(\ell)}_{-\ell} = x_{-\ell}\right] \ge 1/4$ by definitions of $p_i$ and $P_{Y \mid X}$, and $\alpha \ge 3 \sqrt{n}\rho$ by \lemref{lem:alpha}, hence
\[|\delta_i| \le \frac{3\rho}{4 \alpha p_i} \le \frac{3 \rho}{\alpha} \le \frac{1}{\sqrt{n}}. \]
Hence, \eqref{eq:34} is at most
\begin{align*}
\left(\prod_{i=1}^k p_i \right) \sum_{S \subseteq [k]} n^{-\lvert S \rvert/2} \rho^{-\lvert S \rvert} \left\lvert \mathbb{E}_{A \sim \mu_0} \prod_{i\in S} (\mu_i(A)/\mu_0(A) - 1) \right\rvert
\le \left(\prod_{i=1}^k p_i \right) \left( 1 + \frac{1}{n} \right),
\end{align*}
where the last step follows from \eqref{eq:corr} and the fact that all terms corresponding to $\lvert S \rvert = 1$ equal zero.
This concludes that 
\begin{equation}\label{eq:523}
\Pr\left[Y^{(\ell)} = y \middle| X^{(\ell)}_{-\ell} = x_{-\ell}\right] \le \left(\prod_{i=1}^k p_i \right) \left(1 + \frac{1}{n}\right).
\end{equation}
Since $Y^{(\ell-1)}$ is obtained from $X^{(\ell-1)}$ the same $Y^{(\ell)}$ is obtained from $X^{(\ell)}$ (using the conditioned probabilities of $P_{Y\mid X}$), it holds that
\begin{equation} \label{eq:518}
\Pr\left[Y^{(\ell-1)} = y \middle| X^{(\ell-1)}_{-\ell} = x_{-\ell},\ X^{(\ell-1)}_\ell \right]
=\Pr\left[Y^{(\ell)} = y \middle| X^{(\ell)}_{-\ell} = x_{-\ell},\ X^{(\ell)}_\ell \right]
= \prod_{i=1}^k p_i\left(1 + X_{\ell,i}^{(\ell)} \delta_i \right)
\end{equation}
where the last equation follows from \eqref{eq:511}. 
Since the entries of $X^{(\ell-1)}_\ell$ are distributed $\mu_0'$, they are independent, hence
\begin{align}
\Pr\left[Y^{(\ell-1)} = y \middle| X^{(\ell-1)}_{-\ell} = x_{-\ell} \right]
&= \mathbb{E}_{X^{(\ell-1)}_\ell}\Pr\left[Y^{(\ell-1)} = y \middle| X^{(\ell-1)}_{-\ell} = x_{-\ell}, X^{(\ell-1)}_\ell \right] \notag \\
&= \mathbb{E}_{X^{(\ell-1)}_\ell} \left[ \prod_{i=1}^k p_i\left(1 + X_{\ell,i}^{(\ell)} \delta_i \right) \right]\label{eq:517} \\
&=  \prod_{i=1}^k \mathbb{E}_{X^{(\ell-1)}_\ell} \left[ p_i\left(1 + X_{\ell,i}^{(\ell)} \delta_i \right) \right]\label{eq:519} \\
&= \prod_{i=1}^k p_i \label{eq:520}
\end{align}
where \eqref{eq:517} follows from \eqref{eq:518}, \eqref{eq:519} follows from the fact that entries of $X^{(\ell-1)}_\ell$ are independent, and \eqref{eq:520} follows from the fact that $X_\ell^{(\ell-1)}$ is distributed $\mu_0'$ and by definition of $\mu_0'$, each bit of $X_\ell^{(\ell-1)}$ is distributed uniformly. \eqref{eq:520} and \eqref{eq:523} imply that
\begin{equation} \label{eq:443}
\Pr\left[Y^{(\ell-1)} = y \middle| X^{(\ell-1)}_{-\ell} = x_{-\ell} \right]
= \prod_{i=1}^k p_i
\ge \left( 1 + \frac{1}{n}\right)^{-1} \Pr\left[Y^{(\ell)} = y \middle| X^{(\ell)}_{-\ell} = x_{-\ell} \right].
\end{equation}
Since $X_{-\ell}^{(\ell)}$ and $X^{(\ell-1)}_{-\ell}$ have the same distribution, we can take an expectation over $X_{-\ell}^{(\ell-1)}$ in the LHS of \eqref{eq:443} and over $X^{(\ell)}_{-\ell}$ in the RHS, and obtain that
\[
P_{Y^{(\ell-1)}}(y) \ge \left( 1 + \frac{1}{n}\right)^{-1} P_{Y^{(\ell)}}(y).
\]
Therefore,

\[
P_Y(y)
= P_{Y^{(n)}}(y)
\le \left( 1 + \frac{1}{n} \right)^n P_{Y^{(0)}}(y)
\le e P_{Y^{(0)}}(y)
= e P_{Y'}(y)
\le \frac{e^2}{2^k},
\]
where the last inequality follows from \eqref{eq:521}. \eqref{eq:522} implies that this concludes the proof.
\end{proof}

Next, we show a chain of inequalities to conclude the proof. Fix some channel $P_{\Pi \mid X}$. Using this channel and $P_X=\mu_0^n$ we can define the joint distribution $P_{\Pi X} = P_{X, (\Pi \mid X)}$ and obtain the inverse channel $P_{X \mid \Pi}$ from this joint distribution. We can extend our Markov chain to $\Pi \to X \to Y \to Z$, where the conditional probability of $X$ conditioned on $\Pi$ is obtained from the channel $P_{X\mid\Pi}$. The data processing inequality (Proposition~\ref{prop:DPI}) implies that 
\begin{equation} \label{eq:YX}
I(\Pi ; Y) \le I(\Pi ; X).
\end{equation}
Lemma~\ref{lem:ent} enables us to bound $\sum_{i=1}^k I(\Pi;Y_i)$ in terms of $I(\Pi; Y)$. Formally, we obtain the following:
\begin{align}
\sum_{i=1}^k I(\Pi ; Y_i)
&~=~ \sum_{i=1}^k \left(H(Y_i) - H(Y_i \mid \Pi) \right) \label{eq:524} \\
&~\le~ k - \sum_{i=1}^k \left(H(Y_i \mid \Pi) \right)\label{eq:525}\\
&~\le~ k - H(Y \mid \Pi) \label{eq:526} \\
&~=~ k - H(Y) + I(\Pi ; Y) \label{eq:527}\\
&~\le~ I(\Pi ; Y) + C \label{eq:528}
\end{align}
where \eqref{eq:524} is by definition of the mutual entropy\footnote{The mutual entropy between two random variables $A$ and $B$ equals $I(A;B) = H(A) - H(A\mid B) = H(B) - H(B \mid A)$.}, \eqref{eq:525} follows from the fact that each $Y_i$ is binary hence its entropy is at most $1$, \eqref{eq:526} follows from the inequality $H(AB\mid C) \le H(A\mid C)+H(B\mid C)$ for all random variables $A,B,C$, \eqref{eq:527} follows from the definition of the mutual entropy, \eqref{eq:528} follows from \lemref{lem:ent} where $C$ is the numeric constant from the lemma. 

Next, we utilize the structure of the channel $P_{Z_i\mid Y_i}$ to strongly bound $I(\Pi;Z_i)$ in terms of $I(\Pi;Y_i)$. Recall that $Z_i$ is a noisy version of $Y_i$. There exists a strong data processing for this channel \citep{ahlswede1976spreading}:
\begin{proposition} \label{prop:dpi}
Let $A \to B \to C$ be a Markov chain such that $B$ and $C$ are binary random variables getting values in $\{-1,1\}$. Let $0 \le q \le 1$ be a number and assume that the transition $B \to C$ is defined by $C = B$ with probability $(1+q)/2$.
Then $I(A ; C) \le q^2 I(A ; B)$.
\end{proposition}
Since $\Pi \to Y_i \to Z_i$ is a Markov chain, applying Proposition~\ref{prop:dpi} with $q = 2 \alpha$ we get that 
\begin{equation} \label{eq:ZiYi}
I(\Pi ; Z_i) \le 4 \alpha^2 I(\Pi ; Y_i).
\end{equation}
The following lemma by \citet[Lemma~6.2]{bar2004information}, relates the Hellinger distance with the mutual information.
\begin{lemma} \label{lem:hel-I}
Let $A$ and $B$ be random variables such that $A$ is uniform over $\{-1,1\}$. Then
\[
\helli^2(P_{B\mid A=-1}, P_{B \mid A=1}) \le I(A; B).
\]
\end{lemma}

We would like to use this lemma in order to bound $\helli^2(P_{\Pi \mid Z_i = -1}, P_{\Pi \mid Z_i = 1})$ in terms of $I(\Pi; Z_i)$, however, $Z_i$ is not necessarily uniform. On the other hand, $Z_i$ is not very biased, hence one can reduce the case that $Z_i$ is not very biased to the case that $Z_i$ is uniform, as done in the following lemma.

\begin{lemma} \label{lem:zzi}
Let $A$ and $B$ be random variables such that $A \in \{-1,1\}$. Then
\[
\helli^2(P_{B\mid A=-1}, P_{B \mid A=1}) \le \frac{I(A; B)}{2 \min(\Pr[A=-1], \Pr[A=1])}.
\]
\end{lemma}

\begin{proof}
	Assume without loss of generality that $\Pr[A=1] \ge \Pr[A=-1]$.
	Let $D \in \{0, 1\}$ be a random variable and we will define a Markov chain $D \to A \to B$, extending $A \to B$ \footnote{For any two random variables $X$ and $Y$ one can define a Markov chain $X \to Y$ by first drawing $X$ and then drawing $Y$ conditioned on $X$.}. Define the distribution of $D$ and the conditional distribution of $A$ (conditioned on $D$) as follows:
	\begin{align*}
	&\Pr[D=0] = 2 \Pr[A=-1] = 2 \min(\Pr[A=-1], \Pr[A=1])  \\
	&\Pr[A=1 \mid D = 0] = 1/2 \\
	&\Pr[A=1 \mid D = 1] = 1.
	\end{align*}
	Note that $P_{A | D} \circ P_D = P_A$ as required. Since $A$ is deterministic when $D=1$, 
	\[
	I(A;B \mid D=1)= H(A \mid D=1) - H(A \mid B, D=1)=0-0=0.
	\]
	Since $D\to A \to B$ is a Markov chain, $B$ is independent of $D$ conditioned on $A$, hence
	\begin{align}
	\Pr[D=0] I(A; B \mid D=0)
	&= I(A ; B \mid D) 
	~=~ H(B \mid D) - H(B \mid AD) \notag\\
	&\le H(B) - H(B \mid AD) 
	~=~ H(B) - H(B \mid A) \notag\\
	&= I(A ; B), \label{eq:582}
	\end{align}
	using the inequality $H(X\mid Y) \le H(X)$ for all random variables $X,Y$.
	Hence,
	\begin{align}
	\helli^2(P_{B \mid A=-1}, P_{B \mid A=1})
	&= \helli^2(P_{B \mid A=-1, D=0}, P_{B \mid A=1, D=0}) \label{eq:732}\\
	&\le I(A ; B \mid D=0) \label{eq:733}\\
	&\le I(A; B) / \Pr[D=0]. \label{eq:734}
	\end{align}
	where \eqref{eq:732} follows from the fact that $B\to Z_i \to \Pi$ is a Markov chain hence $\Pi$ is independent of $B$ conditioned on $Z_i$, \eqref{eq:733} follows from Lemma~\ref{lem:hel-I} and \eqref{eq:734} follows from \eqref{eq:582}.
\end{proof}

We apply Lemma~\ref{lem:zzi} by setting $A = Z_i$ and $B = \Pi$. Lemma~\ref{lem:bias} and Lemma~\ref{lem:alpha} imply that 
\[ 
1/2 - \min(\Pr[Z_i = 1], \Pr[Z_i=-1]) 
= \left| 1/2 - \Pr[Z_i = 1]  \right| 
\le \alpha^2/k \le 1/(4k) \le 1/4, \]
hence, Lemma~\ref{lem:zzi} implies that
\begin{equation} \label{eq:helinf}
\helli^2(P_{\Pi \mid Z_i = 1}, P_{\Pi \mid Z_i = -1})
\le 2 I(\Pi; Z_i).
\end{equation}


Next, we claim that 
\begin{equation} \label{eq:zonezminus}
\helli^2(P_{\Pi}, P_{\Pi \mid Z_i=1}) \le \helli^2(P_{\Pi \mid Z_i=-1}, P_{\Pi \mid Z_i=1}).
\end{equation}
Indeed, $P_{\Pi}$ is a convex combination of $P_{\Pi \mid Z_i=1}$ and $P_{\Pi \mid Z_i=-1}$, and one can verify that the following holds:

\begin{proposition}
Let $\mu$ and $\nu$ be two probability distributions. Then, for any $0 \le \lambda \le 1$,
\[
\helli((1-\lambda)\mu + \lambda \nu, \nu) \le \helli(\mu, \nu).
\]
\end{proposition}

In the following lemma we use the fact that conditioned on $Z_i=1$, $X$ is distributed similarly to $\mu_i^n$, to bound $\helli^2\left( P_{\Pi}, P_{\Pi \mid X \sim \mu_i^n} \right)$ in terms of $\helli^2\left( P_{\Pi}, P_{\Pi \mid Z_i=1} \right)$.
\begin{lemma} \label{lem:last}
It holds that
\[ \helli^2\left( P_{\Pi}, P_{\Pi \mid X \sim \mu_i^n} \right) 
\le 2 \helli^2\left( P_{\Pi}, P_{\Pi \mid Z_i=1} \right) + 5 \frac{\alpha^2}{k}. \]
\end{lemma}

\begin{proof}
We inform the reader that any Markov chain can be reversed to get a new Markov chain: if $X_1 \to \cdots \to X_\ell$ is a Markov chain then the joint distribution of $X_1 \cdots X_\ell$  can be viewed as a Markov chain $X_\ell \to \cdots \to X_1$. Since $Z_i \to Y_i \to X \to \Pi$ is a Markov chain, $P_{\Pi | Z_i = 1} = P_{\Pi| X} \circ P_{X | Z_i=1}$ and $P_{\Pi \mid X \sim \mu_i^n} = P_{\Pi \mid X} \circ \mu_i^n$.
It holds that
\begin{align}
\helli^2&\left( P_{\Pi \mid X \sim \mu_i^n}, P_{\Pi \mid Z_i = 1} \right)
= \helli^2\left( P_{\Pi|X} \circ \mu_i^n, P_{\Pi|X} \circ P_{X \mid Z_i = 1} \right) \nonumber \\
&\le \helli^2\left( \mu_i^n, P_{X \mid Z_i = 1} \right) \label{eq:3777}\\
&\le d_{TV}\left( \mu_i^n, P_{X \mid Z_i = 1} \right) \label{eq:2} \\
&= \frac{1}{2} \sum_{x} \left\lvert \mu_i^n(x) - P_{X \mid Z_i = 1}(x) \right\rvert \nonumber \\
&\le \frac{1}{2} \sum_{x} \left\lvert \mu_i^n(x) - 2P_{Z_i}(1) P_{X \mid Z_i = 1}(x) \right\rvert 
+ \frac{1}{2} \sum_{x} \left\lvert 2 P_{Z_i}(1) P_{X \mid Z_i = 1}(x) - P_{X \mid Z_i = 1}(x) \right\rvert, \label{eq:3}
\end{align}
where \eqref{eq:3777} follows from the data processing inequality (Proposition~\ref{prop:DPI}) for the channel $X \to \Pi$ and \eqref{eq:2} follows from Proposition~\ref{prop:heltv}.
We will bound the two terms in \eqref{eq:3} separately. First,
\begin{align}
\frac{1}{2}\sum_x \left\lvert \mu_i^n(x) - 2P_{Z_i}(1) P_{X \mid Z_i = 1}(x) \right\vert
&= \frac{1}{2}\sum_x \left\lvert \mu_i^n(x) - 2P_{X Z_i}(x, 1) \right\vert\nonumber\\
&= \frac{1}{2}\sum_{x \notin \mathcal{T}} \left\lvert \mu_i^n(x) - 2P_{X Z_i}(x, 1) \right\vert \label{eq:6911}\\
&\le \frac{1}{2}\sum_{x \notin \mathcal{T}} \mu_i^n(x) + \frac{1}{2}\sum_{x \notin \mathcal{T}} 2P_{X Z_i}(x, 1) \nonumber\\
&\le \frac{1}{2} \sum_{x \notin \mathcal{T}} \mu_i^n(x) + \sum_{x \notin \mathcal{T}} P_X(x) \nonumber\\
&= \frac{1}{2} \sum_{x \notin \mathcal{T}} \mu_i^n(x) + \sum_{x \notin \mathcal{T}} \mu_0^n(x) \label{eq:375}\\
&= \frac{1}{2} \mu_i^n(X \notin \mathcal{T}) + \mu_0^n(X \notin \mathcal{T}) \nonumber\\
&\le \frac{3}{2} \alpha^2 / k. \label{eq:376} 
\end{align}
where \eqref{eq:6911} follows from item 3 of Lemma~\ref{lem:bias}, \eqref{eq:375} follows from the definition of $P_X$ and \eqref{eq:376} follows from Lemma~\ref{lem:notinT}. Bounding the second term of \eqref{eq:3}, we get:
\begin{equation} \label{eq:530}
\frac{1}{2} \sum_{x} \left\lvert 2 P_{Z_i}(1) P_{X \mid Z_i = 1}(x) - P_{X \mid Z_i = 1}(x) \right\rvert
=\left\lvert P_{Z_i}(1)  - 1/2 \right\rvert \sum_{x} P_{X \mid Z_i = 1}(x)
=\left\lvert P_{Z_i}(1)  - 1/2 \right\rvert
\le \frac{\alpha^2}{k},
\end{equation}
where the last inequality follows from the fourth item of Lemma~\ref{lem:bias}. \eqref{eq:530}, \eqref{eq:376}  and \eqref{eq:3} imply that
\[
\helli^2\left( P_{\Pi \mid X \sim \mu_i^n}, P_{\Pi \mid Z_i = 1} \right) \le \frac{5\alpha^2}{2k}.
\]
Since the (non-squared) Hellinger distance satisfies the triangle inequality, and $(a+b)^2 = a^2 +2ab+ b^2 \le 2a^2 + 2b^2$ for all $a,b \in \mathbb{R}$,
\begin{align*}
\helli^2(P_{\Pi}, P_{\Pi \mid X \sim \mu_i^n})
&\le \left( \helli(P_{\Pi}, P_{\Pi \mid Z_i=1}) + 
\helli(P_{\Pi \mid Z_i = 1}, P_{\Pi \mid X \sim \mu_i^n}) \right)^2 \\
&\le 2 \helli^2(P_{\Pi}, P_{\Pi \mid Z_i=1}) 
+ 2 \helli^2(P_{\Pi \mid Z_i = 1}, P_{\Pi \mid X \sim \mu_i^n})\\
&\le 2 \helli^2(P_{\Pi}, P_{\Pi \mid Z_i=1}) 
+ 5 \frac{\alpha^2}{k}.
\end{align*}
\end{proof}

To conclude the proof:
\begin{align}
\sum_{i=1}^k \helli^2\left( P_{\Pi}, P_{\Pi \mid X \sim \mu_i^n} \right) 
&\le 2 \sum_{i=1}^k \helli^2\left( P_{\Pi}, P_{\Pi \mid Z_i=1} \right) + 5 \alpha^2 \label{eq:111} \\
&\le 2 \sum_{i=1}^k \helli^2(P_{\Pi \mid Z_i=-1}, P_{\Pi \mid Z_i=1}) + 5 \alpha^2 \label{eq:788} \\
&\le 4\sum_{i=1}^k I(\Pi; Z_i) + 5 \alpha^2 \label{eq:112} \\
&\le 16 \alpha^2 \sum_{i=1}^k I(\Pi; Y_i) + 5 \alpha^2 \label{eq:789} \\
&\le 16 \alpha^2 I(\Pi; Y) + (5 + 16 C) \alpha^2 \label{eq:790} \\
&\le 16 \alpha^2 I(\Pi; X) + (5 + 16 C) \alpha^2 \label{eq:113} \\
&\le \rho \sqrt{n} \left( 2 \sqrt{2\ln(2k^2/(9n \rho^2))} + 3 \right) (16C + 5) (I(\Pi; X) + 1) \label{eq:114},
\end{align}
where $C$ is the constant from Lemma~\ref{lem:ent}, \eqref{eq:111} follows from Lemma~\ref{lem:last}, \eqref{eq:788} follows from \eqref{eq:zonezminus},
\eqref{eq:112} follows from \eqref{eq:helinf}, \eqref{eq:789} follows from \eqref{eq:ZiYi}, \eqref{eq:790} follows from \eqref{eq:528}, \eqref{eq:113} follows from \eqref{eq:YX}, and \eqref{eq:114} follows from Lemma~\ref{lem:alpha}.
Note that by definition $X \sim \mu_0^n$, hence $P_{\Pi} = P_{\Pi \mid X \sim \mu_0^n}$, which concludes the proof.

\subsubsection{Proof of Lemma~\ref{lem:multi} \label{sec:multi}}

The core of the proof follows results of \cite{braverman2016communication} and \cite{jayram2009hellinger}. Let $\mathbf{X}=\left(X^{(1)}, \dots, X^{(m)}\right)$ be a random vector distributed $\left(\mu_0^n\right)^m$ where for all $j \in [m]$, $X^{(j)} \in \Omega^n$ is the input of player $j$. Let $\Pi$ be the transcript of a $1/3$-error $(m,n)$ protocol identifying $\mu \in \{\mu_1,\dots,\mu_k\}$, distributed $P_{\Pi \mid \mathbf{X}}$ conditioned on the input of the players being $\mathbf{X}$. Given a vector $\mathbf{a} = (a_1, \dots, a_m) \in \{0,1, \dots, k\}^m$, let $\Pi_{\mathbf{a}}$ be the random variable denoting the transcript $\Pi$ when every player $j \in [m]$ receives an independent input distributed $\mu_{a_j}^n$. Formally, $\Pi_{\mathbf{a}} \sim P_{\Pi \mid \mathbf{X} \sim \left( \mu_{a_1}^n, \dots, \mu_{a_m}^n \right)}$. For any $j\in [m]$ and $i \in [k]$, let $\mathbf{e}_{j,i}$ be the $m$-entry vector that equals $i$ on coordinate $j$ and all other coordinates are zero, and let $\mathbf{i}$ be the all-$i$ vector.

Since $\mathbf{X}\sim \mu_0^{mn}$, for any $j \in [m]$, $P_{\Pi \mid X^{(j)}}$ is the distribution of $\Pi$ conditioned on player $j$ getting the input $X^{(j)} \in \Omega^n$ while all other players get an independent input distributed $\mu_0^n$. Note that for all $j \in [m]$, $\Pi_{\mathbf{0}} \sim P_{\Pi \mid X^{(j)} \sim \mu_0^n}$, and for all $i \in [n]$, $\Pi_{\mathbf{e}_{j,i}} \sim P_{\Pi \mid X^{(j)} \sim \mu_i^n}$. Hence, the conditions of this lemma imply that
\begin{align}
\sum_{j=1}^m \sum_{i=1}^k \helli^2(\Pi_0, \Pi_{\mathbf{e}_{j,i}})
&= \sum_{j=1}^m \sum_{i=1}^k \helli^2(P_{\Pi \mid X^{(j)} \sim \mu_0^n}, P_{\Pi \mid X^{(j)} \sim \mu_i^n})\notag\\
&\le \sum_{j=1}^m \beta (I(\Pi ; X^{(j)}) + 1). \label{eq:bndeji}
\end{align}
In order to bound the last term we present a known inequality in information theory.
\begin{proposition}
If $X^{(1)}, \dots,X^{(m)}$ are independent random variables and $\Pi$ is a random variable then
\[
\sum_{j=1}^m I\left(\Pi; X^{(j)}\right)
\le I\left(\Pi; X^{(1)} \cdots X^{(m)}\right).
\]
\end{proposition}

\begin{proof}
\begin{align*}
I\left(\Pi; X^{(1)} \cdots X^{(m)}\right)
&~=~ \sum_{j=1}^m I\left(\Pi; X^{(j)} \mid X^{(1)} \cdots X^{(j-1)}\right) \\
&~=~ \sum_{j=1}^m H\left(X^{(j)} \mid X^{(1)} \cdots X^{(j-1)}\right) - H\left(X^{(j)}  \mid \Pi X^{(1)} \cdots X^{(j-1)}\right) \\
&~\ge~ \sum_{j=1}^m H\left(X^{(j)}\right) - H\left(X^{(j)}  \middle| \Pi\right) \\
&~=~ \sum_{j=1}^m I\left(\Pi; X^{(j)}\right).
\end{align*}
where the first equation follows from the chain rule for mutual entropy and the first inequality follows from the independence of $X^{(1)}\cdots X^{(m)}$ and the fact that $H(A\mid BC) \le H(A \mid B)$ for any random variables $A,B,C$.
\end{proof}
Hence, \eqref{eq:bndeji} implies that
\begin{equation} \label{eq:eji2}
\sum_{j=1}^m \sum_{i=1}^k \helli^2(\Pi_0, \Pi_{\mathbf{e}_{j,i}})
~\le~ \beta (I(\Pi ;  \mathbf{X}) + m) 
~\le~ \beta (H(\Pi) + m)
~\le~ \beta (\lvert \Pi \rvert + m),
\end{equation}
where $\lvert \Pi \rvert$ is the communication complexity of the protocol. The following Lemma, \citet[Lemma~2]{braverman2016communication} lower bounds $\sum_{j=1}^m \helli^2(\Pi_\mathbf{0}, \Pi_{\mathbf{e}_{j,i}})$.

\begin{lemma}
For any $1 \le i \le m$,
\[
\helli^2(\Pi_{\mathbf{0}}, \Pi_{\mathbf{i}})
\le C \sum_{j=1}^m \helli^2(\Pi_{\mathbf{0}}, \Pi_{\mathbf{e}_{j,i}})
\]
for some numerical constant $C > 0$.
\end{lemma}

This and \eqref{eq:eji2} implies that
\begin{equation} \label{eq:8}
\sum_{i=1}^k \helli^2(\Pi_{\mathbf{0}}, \Pi_{\mathbf{i}})
\le C \beta( \lvert \Pi \rvert + m).
\end{equation}
The next lemma states that for any protocol error $\varepsilon<1/2$, the LHS of \eqref{eq:8} is $\Omega(k)$.

\begin{lemma} \label{lem:err-small}
Assume $\Pi$ is a transcript of a protocol with a worst-case error (over $\mu_i$) of at most $\varepsilon < 1/2$. Then there exists a subset $S \subseteq [k]$ of size $\lvert S \rvert = k-1$ such that for all $i \in S$,
\[
\helli^2(\Pi_{\mathbf{0}}, \Pi_{\mathbf{i}}) \ge \frac{(1-2\varepsilon)^2}{8}.
\]
In particular,
\[
\sum_{i=1}^k \helli^2(\Pi_{\mathbf{0}}, \Pi_{\mathbf{i}}) \ge \frac{(k-1)(1-2\varepsilon)^2}{8}.
\]
\end{lemma}

\begin{proof}
First, note that for any $i\ne i' \in [k]$, $d_{TV}(\Pi_{\mathbf{i}}, \Pi_{\mathbf{i'}}) \ge 1 - 2\varepsilon$. Indeed, fix some $i \ne i'$ and let $\mathcal{A}$ be the set of all values of $\Pi$ such that the protocol outputs $i$ given these values. Since the protocol has $\varepsilon$-error, $\Pr\left[ \Pi_{\mathbf{i}}\in \mathcal{A}\right] \ge 1 - \varepsilon$ and $\Pr\left[\Pi_{\mathbf{i'}}\in \mathcal{A}\right] \le \varepsilon$. Hence, by definition of the total variation distance,
\begin{equation} \label{eq:hellihigh}
d_{TV}(\Pi_{\mathbf{i}}, \Pi_{\mathbf{i'}})
\ge \Pr\left[ \Pi_{\mathbf{i}}\in \mathcal{A}\right] - \Pr\left[\Pi_{\mathbf{i'}}\in \mathcal{A}\right]
\ge 1-2\varepsilon.
\end{equation}

Assume for contradiction that there are $i \ne i' \in [k]$ such that 
\[
\helli^2(\Pi_{\mathbf{0}}, \Pi_{\mathbf{i}}), \helli^2(\Pi_{\mathbf{0}}, \Pi_{\mathbf{i'}}) 
< \frac{(1-2\varepsilon)^2}{8}.
\]
Then, since the Hellinger distance $\helli()$ obeys the triangle inequality and by Proposition~\ref{prop:heltv},
\[
d_{TV}(\Pi_{\mathbf{i}}, \Pi_{\mathbf{i'}})
\le \sqrt{2} \helli(\Pi_{\mathbf{i}}, \Pi_{\mathbf{i'}})
\le \sqrt{2}\helli(\Pi_{\mathbf{0}}, \Pi_{\mathbf{i}}) + \sqrt{2}\helli(\Pi_{\mathbf{0}}, \Pi_{\mathbf{i'}}) 
< 1-2\varepsilon,
\]
in contradiction to \eqref{eq:hellihigh}.
\end{proof}

\lemref{lem:err-small} and \eqref{eq:8} conclude that any $1/3$-error protocol has a communication complexity of at least $Ck/\beta - m$, for some numerical constant $C>0$. We conclude by showing that the communication complexity is at least $Ck/(2\beta)$. Assume for contradiction that the communication complexity is less than $Ck/(2\beta)$. Denote the parties in the protocol by $1,\dots, m$ and assuming without loss of generality that party $1$ is always the first to talk, party $2$ is the first party to talk among parties $2,\dots,m$,  party $3$ talks first among parties $3, \dots, m$ etc.\footnote{The symmetries between the parties imply that if at some point in the protocol a new party is speaking, one can assume that this party has the lowest index among all parties that have not spoken yet.}, then only a subset of parties $1, \dots, \lfloor Ck/(2\beta) \rfloor$ participates in the protocol, hence we can assume that $m \le Ck/(2\beta)$. The communication complexity is at least $Ck/\beta-m \ge Ck/(2\beta)$, which concludes the proof.

\subsubsection{Proof of Lemma~\ref{lem:trans} \label{sec:lem-trans}}
 
We start by defining an opposite channel $P_{\mu|\eta} \colon \Omega_\eta \to \{-1,1\}^k$: given some $y \in \Omega^\eta$, the channel sends it to $x=(x_1,\dots,x_k) \in \{-1,1\}^k$, where each bit of $x$ is set independently, such that: 
\begin{equation} \label{eq:36}
(1+\rho) P_{\mu | \eta}(x_i = 1 \mid y)
+ (1 - \rho) P_{\mu|\eta}(x_i = -1 | y)
= \eta_i(y) / \eta_0(y).\footnote{To avoid issues of devision by $0$, assume that $\eta_0$ has full support. Indeed, one can remove from $\Omega_\eta$ all elements $y$ for which $\eta_0(y)=0$ to obtain $\Omega'_\eta$ and use $\Omega'_\eta$ as the joint sample space of $\eta_0,\dots,\eta_k$. By definition of a $\CD(\rho)$ family, $\eta_i(y) = 0$ for any $y \in \Omega_\eta \setminus \Omega_\eta'$ and for all $i \in [k]$, hence $\eta_1,\dots, \eta_k$ can be viewed as probability distributions over $\Omega'_\eta$.}
\end{equation}
Such a definition is possible since, by the definition of a $\CD(\rho)$ family, 
$1-\rho \le \eta_i(y)/\eta_0(y) \le 1+\rho$.
Define $\mu_0 = P_{\mu | \eta} \circ \eta_0$ and define $\Omega_\mu$ as the support of $\mu_0$.
Note that taking an expectation over $y \sim \eta_0$ in \eqref{eq:36}, one obtains that
\[
(1+\rho) \mu_0(x_i = 1) + (1-\rho) \mu_0(x_i=-1) = 1.
\]
Hence, $\mu_0(x_i=1) = \mu_0(x_i=-1)=1/2$ for all $i\in [k]$, as required by the definition of a $\BCD(\rho)$ family. For $i=1,\dots,k$, define the distribution $\mu_i$ as in the definition of a $\BCD(\rho)$ family: $\mu_i(x) = \mu_0(x)(1 + \rho x_i)$. It holds that $\{\mu_1,\dots,\mu_k\}$ is a $\BCD(\rho)$ family, as required.

Define the channel $P_{\eta|\mu} \colon \Omega_\mu \to \Omega_\eta$ as the channel sending $\mu_0$ to $\eta_0$, namely,
\[
P_{\eta | \mu}(y \mid x) = 
\frac{P_{\mu|\eta}(x \mid y) \eta_0(y)}{\mu_0(x)}
\]
We will show that $\eta_i = P_{\eta \mid \mu} \circ \mu_i$ for any $1 \le i \le k$. Indeed,
\begin{align}
(P_{\eta\mid \mu} \circ \mu_i)(y)
&= \sum_{x \in \Omega_\mu} \mu_i(x) P_{\eta|\mu}(y \mid x) \notag \\
&= \sum_{x \in \Omega_\mu} \mu_i(x) \frac{P_{\mu|\eta}(x \mid y) \eta_0(y)}{\mu_0(x)} \nonumber\\
&= \sum_{x \in \Omega_\mu} (1 +\rho x_i) P_{\mu|\eta}(x \mid y) \eta_0(y) \label{eq:divmuimu0}\\
&= \eta_0(y) \sum_{b \in \{-1,1\}} (1 + b \rho) P_{\mu|\eta}(x_i = b \mid y) \nonumber\\
&= \eta_0(y)\frac{\eta_i(y)}{\eta_0(y)}, \label{eq:37}
\end{align}
where \eqref{eq:divmuimu0} follows from the definition of $\mu_i$ and \eqref{eq:37} follows from \eqref{eq:36}. 
In order to conclude the proof of this lemma, it remains to prove \eqref{eq:corem}. Here is an auxiliary lemma:

\begin{lemma} \label{lem:subset}
Let $U= (U_1, \dots, U_k), V=(V_1, \dots, V_k) \in \mathbb{R}^k$ be random vectors. If for any possible value $u$ of $U$, $\mathbb{E}[V \mid U=u] = u$ and $V_1, \dots, V_k$ are independent conditioned $U=u$, then for any subset $S \subseteq [k]$,
\[
\mathbb{E}\left[ \prod_{i \in S} U_i \right]
= \mathbb{E}\left[ \prod_{i \in S} V_i \right].
\]
\end{lemma}

\begin{proof}
Fix some set $S\subseteq [k]$. It holds that:
\[
\mathbb{E}\left[ \prod_{i \in S} V_i \right]
= \mathbb{E}\left[ \mathbb{E}\left[ \prod_{i \in S} V_i ~\middle|~ U \right]\right]
= \mathbb{E}\left[ \prod_{i \in S} \mathbb{E}\left[ V_i ~\middle|~ U \right]\right]
= \mathbb{E}\left[ \prod_{i \in S} U_i \right].
\]
\end{proof}

Let $X \sim \mu_0$ and $Y\sim \eta_0$. We conclude the proof of \eqref{eq:corem} by applying \lemref{lem:subset} with $U_i = \eta_i(Y)/\eta_0(Y)-1$ and $V_i = \mu_i(X)/\mu_0(X)-1$. \eqref{eq:36} implies that the condition $\mathbb{E}[V\mid U=u] = u$ holds. By definition of $P_{\mu \mid \eta}$, the bits of $X$ are independent conditioned on $U$. By definition of $\mu_i$, $V_i = \rho X_i$, hence $V_1 \cdots V_k$ are independent conditioned on $U$, which implies that all conditions of \lemref{lem:subset} hold.

\subsection{Proofs from \subsecref{subsec:binary}}

Fix some $0<\rho<1$ and define $\Omega = \{-1,1\}^d$ for some $d \ge 2$. Let $\mathcal{I}$ be the set of all nonempty subsets of $\{1,\dots,d\}$. For any $I \in \mathcal{I}$ and $0 <\rho < 1$, let $\mu_{I,\rho}$ be the distribution over $\Omega$ defined by 
\[
\mu_{I,\rho}((x_1,\dots,x_d)) = 2^{-d} (1 + \rho \prod_{i\in I} x_i).
\]
We will write $\mu_I$ whenever $\rho$ is implied from the context. Note that $\mu_I$ is almost uniform, with a small bias towards inputs that contain an even number of $1$-values on $I$. 
For any subset $\mathcal{U} \subseteq \mathcal{I}$ and $0 < \rho < 1$, let $\mathcal{P}_{\mathcal{U},\rho} = \{ \mu_{I,\rho} \colon I \in \mathcal{U}\}$.
Note that $\mathcal{P}_{\mathcal{U},\rho}$ is a $\CD(\rho)$ family and the corresponding $\mu_0$ distribution is the uniform distribution over $\Omega$. 

\subsubsection{Proof of Theorem~\ref{thm:subset-parity}} \label{sec:pr-subset}

Let $A=(A_1,\dots,A_d) \sim \mu_0$ and for any $I \in \mathcal{I}$, define the random variable $B_I$ as a function of $A$: 
\begin{equation} \label{eq:BCorr}
B_I = \prod_{i\in I} A_i.
\end{equation}
Note that for all $0<\rho<1$, 
\begin{equation} \label{eq:corr-BI}
B_I = (\mu_{I,\rho}(A)/\mu_0(A) - 1)/\rho,
\end{equation}
a term which appears in \eqref{eq:corr} of Theorem~\ref{thm:main}.
The next lemma states what are the correlations between these random variables $B_I$.
\begin{lemma} \label{lem:subsets}
Let $\mathcal{J} \subseteq \mathcal{I}$. Then
\[
\mathbb{E} \left[ \prod_{I\in \mathcal{J}} B_I \right] = \begin{cases}
1 & \bigtriangleup \mathcal{J} = \emptyset \\
0 & \text{otherwise}
\end{cases},
\]
where $\bigtriangleup \mathcal{J}$ is the symmetric difference between all sets in $\mathcal{J}$ which contain all elements $i \in \{1,\dots,d\}$ which appear in an odd number of sets from $\mathcal{J}$.
In particular, if $I_1, I_2 \in \mathcal{I}$ are distinct sets then $\mathbb{E} B_{I_1}B_{I_2} = 0$.
\end{lemma}
\begin{proof}
Note that 
\begin{align}
\mathbb{E} \prod_{I \in \mathcal{J}} B_I
= \mathbb{E}\prod_{I \in \mathcal{J}} \prod_{i \in I} A_i 
= \mathbb{E} \prod_{i \in \bigtriangleup \mathcal{J}} A_i\notag
= \prod_{i \in \bigtriangleup \mathcal{J}} \mathbb{E} A_i\notag
= \begin{cases} 
1 & \bigtriangleup \mathcal{J} = \emptyset \\
0 & \text{otherwise}
\end{cases}
\end{align}
where the first equation follows from the definition of $B_I$, the third equation follows from the fact that the coordinates of $A$ are independent and and the empty product is regarded as $1$.
\end{proof}

Lemma~\ref{lem:subsets} states that the $B_I$ are pairwise independent, hence, for any subset $\mathcal{U} \subseteq \mathcal{I}$ and suitable values of $n$ and $\rho$  one can apply Theorem~\ref{thm:main} on the family of distributions $\mathcal{P}_{\mathcal{U},\rho}$:  Lemma~\ref{lem:subsets} imply that all the terms in \eqref{eq:corr} corresponding to $\lvert S \rvert = 2$ are zero, hence \thmref{thm:main} can be applied for any $n \ge k^6$ (and a suitable $\rho$). This proves Theorem~\ref{thm:subset-parity}.

\subsubsection{Proof of Lemma~\ref{lem:biased-subset}} \label{sec:pr-biased}

By definition of $\mu_I$,
\begin{align}
\mathbb{E}_{x\sim \mu_{I}} \prod_{i\in I'} x_i
&= \sum_{x \in \{-1,1\}^d} \mu_I(x) \prod_{i\in I'} x_i \nonumber\\
&= \sum_{x \in \{-1,1\}^d} 2^{-d} (1+ \rho\prod_{i\in I} x_i) \prod_{i\in I'} x_i \nonumber\\
&= \mathbb{E}_{X \sim \mu_0} (1+ \rho\prod_{i\in I} X_i) \prod_{i\in I'} X_i \label{eq:bin1}\\
&= \prod_{i\in I'} \mathbb{E}_{X \sim \mathrm{Uniform}(\{-1,1\})} \left[X_i \right]
+ \rho \prod_{i\in I \bigtriangleup I'} \mathbb{E}_{X \sim \mathrm{Uniform}(\{-1,1\})} \left[X_i \right] \label{eq:bin3}\\
&= \rho \prod_{i\in I \bigtriangleup I'} \mathbb{E}_{X \sim \mathrm{Uniform}(\{-1,1\})} \left[X_i \right] \label{eq:bin2}\\
&=\begin{cases}
\rho & I=I' \\
0 & I \ne I'
\end{cases} \nonumber
\end{align}
where \eqref{eq:bin1} and \eqref{eq:bin3} follow from the fact that $\mu_0$ is the uniform measure over $\{-1,1\}^d$ and \eqref{eq:bin2} follows from the fact that $I' \ne \emptyset$.

\subsection{Proof of Theorem~\ref{thm:normal}} \label{sec:pr-normal}

First, we give an outline to the proof. Recall that $\eta_{I,\sigma}$ is defined as the Gaussian distribution over $\mathbb{R}^d$ with mean zero and its covariance matrix, $\Sigma_{I,\sigma}$, is almost the identity, except for two coordinates, $i$ and $j$, with a covariance of $\sigma$. These coordinates satisfy $I = \{i,j\}$. Denote by $\eta_0$ the Gaussian distribution over $\mathbb{R}^d$ with zero mean and its covariance, $\Sigma_0$, is the identity matrix. For any $x \in \mathbb{R}^n$, let $\eta_{I,\sigma}(x)$ denote the density of $\eta_{I,\sigma}$ on $x$.

We start with some preliminaries in \subsecref{sec:pr-gau-prel}. Then, we show that for any $I \ne I'$, $\eta_{I,\sigma}$ and $\eta_{I',\sigma}$ are pairwise uncorrelated with respect to $\eta_0$ in the following way:
\[
\mathbb{E}_{X \sim \eta_0} \left[\left(\frac{\eta_{I,\sigma}(X)}{\eta_0(X)} - 1\right) \left(\frac{\eta_{I',\sigma}(X)}{\eta_0(X)} -1 \right) \right] = 0,
\]
which is equivalent to
\[
\mathbb{E}_{X \sim \eta_0} \left[\frac{\eta_{I,\sigma}(X)}{\eta_0(X)} \frac{\eta_{I',\sigma}(X)}{\eta_0(X)} \right] = 1.
\]
This is proved by taking an integral and calculating a determinant. We denote $\Sigma_{I,I'}^{-1} = \Sigma_{I,\sigma}^{-1} + \Sigma_{I',\sigma}^{-1} - \Sigma_0^{-1}$. The following holds:
\begin{align}
\mathbb{E}_{X \sim \eta_0} \left[\frac{\eta_{I,\sigma}(X)}{\eta_0(X)} \frac{\eta_{I',\sigma}(X)}{\eta_0(X)} \right]
&=\int_{x\in \mathbb{R}^d} \frac{\eta_{I}(x) \eta_{I'}(x)}{\eta_0(x)}dx \notag\\
&= \int_{x \in \mathbb{R}^d} \frac{1}{(2\pi)^{d/2}\sqrt{\det(\Sigma_{I,\sigma}) \det(\Sigma_{I',\sigma})}} \exp\left(-\frac{1}{2}x^t \Sigma_{I,I'}^{-1} x\right) dx\notag\\
&= \frac{\sqrt{\det(\Sigma_{I,I'})}}{\sqrt{\det(\Sigma_{I,\sigma})\det(\Sigma_{I',\sigma})}} \int_{x \in \mathbb{R}^d} \frac{1}{(2\pi)^{d/2}\sqrt{\det(\Sigma_{I,I'})}} \exp\left(-\frac{1}{2}x^t \Sigma_{I,I'}^{-1} x\right) dx \label{eq:exp-gau2}\\
&= \frac{\sqrt{\det(\Sigma_{I,I'})}}{\sqrt{\det(\Sigma_{I,\sigma})\det(\Sigma_{I',\sigma})}} \label{eq:exp-gau1},
\end{align}
where \eqref{eq:exp-gau1} follows from the fact that the integrand in \eqref{eq:exp-gau2} is a density function of a normal distribution with mean zero and covariance $\Sigma_{I,I'}$. The term in \eqref{eq:exp-gau1} equals $1$ for any $I \ne I'$, as required. In the proof we also calculate higher order correlations, namely, 
\begin{equation} \label{eq:exp-gau-mult} 
\mathbb{E}_{X \sim \eta_0} \left[\left(\prod_{i=1}^r \frac{\eta_{I_i,\sigma}(X)}{\eta_0(X)} -1 \right) \right],
\end{equation}
for distinct $I_1, \dots, I_r$. In order to calculate this expectation, we define the matrix $\Sigma_{I_1,\dots,I_r}^{-1}$ in \eqref{eq:def-sig1r} similarly to $\Sigma_{I,I'}^{-1}$. In \lemref{lem:g-det} we prove some properties of $\Sigma_{I_1,\dots,I_r}^{-1}$ and in \lemref{lem:gzero} we show that \eqref{eq:exp-gau-mult} equals zero for some collections $I_1,\dots,I_r$. These two lemmas and other auxiliaries appear in \subsecref{sec:gau-aux}.

Note that we cannot apply \thmref{thm:main} directly on the family of Gaussian distributions: the theorem requires that for any $x \in \mathbb{R}^d$, $|\eta_{I,\sigma}(x)/\eta_0(x)-1| \le \rho$ for some $\rho > 0$, which is incorrect for the Gaussian distributions. Hence, we apply it on a family of truncated normal distributions, in \subsecref{sec:gau-trunc}.
The truncated Gaussian, $\eta_{I,\sigma,R}$, is defined as a truncation of $\eta_{I,\sigma}$ to $[-R,R]^d$, where $R$ is logarithmic in the problem parameters. Indeed, it holds that $|\eta_{I,\sigma}(x)/\eta_0(x)-1| \le \rho$ for some $\rho = \tilde{O}(\sigma)$. Additionally, we show that due to the fact that the truncated Gaussians are almost identical to the Gaussian distributions, their higher order correlations (\eqref{eq:exp-gau-mult}) are almost identical to those of Gaussian distributions. Hence, one can apply \thmref{thm:main} as required.

Lastly, in \subsecref{sec:gau-trun-std} we show that a communication lower bound on learning a truncated Gaussian implies a lower bound on learning a non-truncated Gaussian: due to the fact that high deviations in normal distributions are rare, with high probability all samples fall within $[-R,R]^d$. In that case, one cannot learn with little communication.

\subsubsection{Preliminaries} \label{sec:pr-gau-prel}

The next proposition states some basic properties of the determinant (denoted $\det$).
\begin{proposition} \label{prop:det}
The following hold for any matrix $M \in \mathbb{R}^{n \times n}$:
\begin{enumerate}
\item $\det (cM) = c^n \det M$ for any $c \in \mathbb{R}$. \label{itm:det-const}
\item If the matrix $M$ can be written as
$
M = \begin{pmatrix}
A& C \\
0& B 
\end{pmatrix}$,
where $A \in \mathbb{R}^{n_1 \times n_1}$, $B \in \mathbb{R}^{n_2 \times n_2}$, $C \in \mathbb{R}^{n_1 \times n_2}$ and the $0$-block is of size $n_2 \times n_1$ for some integers $n_1$ and $n_2$ satisfying $n_1 + n_2 = n$, then $\det M = \det A \det B$. \label{itm:det-block}
\item Assume that $M$, $M_1$ and $M_2$ are $n \times n$ matrices which are identical except for column $i$ (for some $1 \le i \le n$), such that column $i$ of $M_1 + M_2$ equals column $i$ of $M$. Then $\det M = \det M_1 + \det M_2$. \label{itm:det-sum}
\item If $A$ and $B$ are squared matrices with the same dimension, then $\det (AB) = \det A \det B$. In particular, if $A$ is invertible then $\det A \det A^{-1} = \det (A A^{-1}) = \det \mathrm{I} = 1$. \label{itm:det-inv}
\item \[
\det M = \begin{cases}
\sum_{i = 1}^n (-1)^{i-1} M_{1 i} \det M_{-1,-i} & n > 1 \\
M_{11} & n = 1
\end{cases}
\]
where $M_{-1,-i}$ is the $(n-1)\times(n-1)$ matrix obtained from $M$ by removing its first row and column $i$. \label{itm:det-rec}
\end{enumerate}
\end{proposition}

Next, we define a \emph{positive definite} matrix:
\begin{definition} \label{def:pd}
Fix an integer $\ell \ge 1$. A squared matrix $M \in \mathbb{R}^{\ell \times \ell}$ is \emph{positive definite} if one of the equivalent conditions hold:
\begin{enumerate}
\item For any nonzero vector $v \in \mathbb{R}^{\ell}$, $v^t M v > 0$.
\item All the eigenvalues of $M$ are positive.
\end{enumerate}
\end{definition}
Note that any positive definite matrix does not have the eigenvalue $0$, hence it is invertible.
Note that applying the same permutation on the rows and the columns of a matrix keeps many of its properties:
\begin{proposition}\label{prop:perm}
Fix an integer $\ell \ge 1$, a matrix $M \in \mathbb{R}^{\ell\times \ell}$ and a permutation $\pi \colon [\ell] \to [\ell]$. Let $\pi(M)$ be the matrix obtained after applying $\pi$ on both the rows and columns of $M$: $(\pi(M))_{\pi(i),\pi(j)} = M_{i,j}$. The following hold:
\begin{enumerate}
\item $\det \pi(M) = \det M$.
\item $\pi(M)$ is positive definite if and only if $M$ is positive definite.
\item If $M$ is invertible then $\pi(M)$ is invertible and $\pi(M^{-1}) = (\pi(M))^{-1}$.
\end{enumerate}
\end{proposition}

Next, we define a multivariate normal distribution:
\begin{definition}
For any integer $\ell \ge 1$ and a symmetric positive definite matrix $\Sigma \in \mathbb{R}^{\ell \times \ell}$, the $\ell$-variate normal distribution with mean $0$ and covariance $\Sigma$ is defined by the density function
\[
\frac{1}{\sqrt{\det(2 \pi\Sigma)}} \exp \left( -x^t \Sigma^{-1} x / 2\right)
\]
as a function of $x \in \mathbb{R}^\ell$.
\end{definition}

For any $\ell \ge 1$, let $\mathrm{I}_\ell$ be the identity matrix of dimension $\ell \times \ell$.

\subsubsection{Auxiliary Technical Results} \label{sec:gau-aux}

Let $\eta_0$ be the normal distribution in $\mathbb{R}^d$ with zero mean and its covariance matrix $\Sigma_0$ is the identity matrix. Recall the definition of $\Sigma_{I,\sigma}$ and $\eta_{I,\sigma}$ from \subsecref{subsec:gaussian}. They will be written as $\Sigma_I$ and $\eta_I$  when $\sigma$ is clear from context.

\begin{lemma} \label{lem:sigma-I}
For any $I \in \mathcal{I}_2$ and $0 < \sigma < 1$, $\Sigma_{I,\sigma}$ is symmetric, positive definite, $\det(\Sigma_{I,\sigma}) = 1 - \sigma^2$ and 
\begin{equation} \label{eq:def-sig-inv}
\Sigma_{I,\sigma}^{-1}
= \frac{1}{1-\sigma^2} \begin{cases}
1 & i=j, i \in I \\
1 - \sigma^2 & i = j, i \notin I \\
-\sigma & i \ne j, I = \{i,j\} \\
0 & i \ne j, I \ne \{i,j\}
\end{cases}.
\end{equation}
In particular, there exists a random vector with mean $0$ and covariance $\Sigma_{I,\sigma}$.
\end{lemma}
\begin{proof}
By proposition~\ref{prop:perm} it is sufficient to assume that $I = \{1,2\}$. Then, $\Sigma_I$ is a block matrix
\[
\Sigma_I = \begin{pmatrix}
A & 0\\
0 & \mathrm{I}_{d-2}
\end{pmatrix},
\]
where
\[
A=\begin{pmatrix}
1 & \sigma \\
\sigma & 1
\end{pmatrix}.
\]
It holds that
\[
\Sigma_I^{-1} = \begin{pmatrix}
A^{-1} & 0 \\
0 & \mathrm{I}_{d-2}^{-1}
\end{pmatrix}
\]
where
\[
A^{-1} = \frac{1}{1-\sigma^2} \begin{pmatrix}
1 & -\sigma \\
-\sigma & 1
\end{pmatrix},
\]
which concludes the proof for the formula of $A^{-1}$.
To calculate the determinant, note that $\det \Sigma_I = \det A \det \mathrm{I}_{d-2} = 1 - \sigma^2$.
Lastly, $\Sigma_{I,\sigma}$ is positive definite because it is strictly diagonally dominant with positive diagonal entries and symmetric.
\end{proof}

Fix $b = 5$, and assume that the constant $C$ in Theorem~\ref{thm:normal} is sufficiently small to ensure that
\begin{equation} \label{eq:sigmab}
\frac{4b^2 \sigma}{1-\sigma^2} \le 1/2.
\end{equation}
For any integer $2 \le r \le b$ and any distinct sets $I_1,\dots, I_r \in \mathcal{I}_2$, define
\begin{equation} \label{eq:def-sig1r}
\Sigma_{I_1,\dots,I_r} := \left( \mathrm{I}_d + \sum_{i=1}^r \left(\Sigma_{I_i}^{-1} - \mathrm{I}_d \right)\right)^{-1}.
\end{equation}
The following lemma shows that $\Sigma_{I_1,\dots,I_r}$ exists and estimates some of its properties.
Given a matrix $M$ let $\max \lvert M \rvert$ denote the maximal absolute value of an element of $M$.

\begin{lemma} \label{lem:g-det}
Fix distinct pairs $I_1,\dots,I_r \in \mathcal{I}_2$ for some $2 \le r \le b$. The matrix $\Sigma_{I_1,\dots,I_r}$ defined in \eqref{eq:def-sig1r} exists and satisfies:
\begin{enumerate}
\item $\Sigma_{I_1,\dots, I_r}$ is symmetric and positive definite. \label{itm:pdsym}
\item \label{itm:det}
$
\det \left(\Sigma_{I_1,\dots,I_r} \right) 
\le 2
$.
\item \label{itm:inf}
$
\max \left\lvert \Sigma_{I_1,\dots,I_r} \right\rvert
\le 2
$.
\end{enumerate}
In particular, there exists a multivariate normal distribution with mean $0$ and covariance $\Sigma_{I_1,\dots,I_r}$.
\end{lemma}

\begin{proof}
Since each $I_i$ is a set of two elements, $\left\lvert \cup_{i=1}^r I_i \right\rvert \le 2r$. By Proposition~\ref{prop:perm}, one can assume that $\bigcup_{i=1}^r I_i \subseteq [2r]$. 
Define
\[
\Sigma^{-1}_{I_1,\dots I_r} = 
\mathrm{I}_d + \sum_{i=1}^r \left( \Sigma_{I_i}^{-1} - \mathrm{I}_d \right).
\]
By Lemma~\ref{lem:sigma-I}, for any $I \in \mathcal{I}_2$,
\begin{equation} \label{eq:sigma-minus-I}
(\Sigma_{I}^{-1} - \mathrm{I}_d)(i,j)
= \frac{1}{1-\sigma^2} \begin{cases}
\sigma^2 & i = j \in I \\
-\sigma & \{i,j\} = I \\
0 & \text{otherwise}
\end{cases}.
\end{equation}
Hence,
\begin{equation} \label{eq:Adef}
\Sigma_{I_1,\dots,I_r}^{-1}(i,j) = \begin{cases}
1 + \frac{\sigma^2}{1-\sigma^2}\sum_{i=1}^r \lvert \{ i \} \cap I_i \rvert & i=j\\
-\frac{\sigma}{1-\sigma^2} &  \{i,j \} = I_i \text{ for some } 1 \le i \le r \\
0 & \text{otherwise}
\end{cases}.
\end{equation}
By the assumption $\bigcup_{i=1}^r I_i \subseteq [2r]$,
\begin{equation}\label{eq:sig-decomp}
\Sigma_{I_1,\dots,I_r}^{-1} = \begin{pmatrix}
A & 0 \\
0 & \mathrm{I}_{d-2r}
\end{pmatrix}
\end{equation}
where $A \in \mathbb{R}^{(2r) \times (2r)}$, $\mathrm{I}_{d-2r}$ is the identity matrix of size $(d-2r)\times (d-2r)$ and the two zero blocks are of sizes $(2r)\times (d-2r)$ and $(d-2r)\times(2r)$.
We will start by showing that $\Sigma^{-1}_{I_1,\dots,I_r}$ is positive definite. Fix some nonzero $v \in \mathbb{R}^{d}$. Let $v_A$ be the vector containing the first $2d$ coordinates of $v$ and let $v_{\mathrm{I}}$ be the vector containing its remaining coordinates. It holds that
\begin{align}
v^t \Sigma_{I_1,\dots,I_r}^{-1} v
&= v_A^t A v_A + v_{\mathrm{I}}^t \mathrm{I}_{d-2r} v_{\mathrm{I}} \notag\\
&=\sum_{i=1}^{2r} \sum_{j=1}^{2r} A_{i,j} v_i v_j + \lVert v_{\mathrm{I}} \rVert_2^2 \notag\\
&=\sum_{i=1}^{2r} A_{i,i} v_i^2 + \sum_{\substack {i,j \in \{1,\dots,2r\} \\ i \ne j}} A_{i,j} v_i v_j + \lVert v_{\mathrm{I}} \rVert_2^2 \notag \\
&\ge \sum_{i=1}^{2r} v_i^2 - \frac{\sigma}{1-\sigma^2}\sum_{i=1}^{2r} \sum_{j=1}^{2r} \lvert v_i v_j \rvert + \lVert v_{\mathrm{I}} \rVert_2^2 \notag \\
&= \lVert v_A \rVert_2^2 - \frac{\sigma}{1-\sigma^2} \lVert v_A \rVert_1^2 + \lVert v_{\mathrm{I}} \rVert_2^2 \notag \\
&\ge \left(1 - \frac{2r \sigma}{1-\sigma^2} \right) \lVert v_A \rVert_2^2+ \lVert v_{\mathrm{I}} \rVert_2^2 \label{eq:g635} 
\end{align}
where \eqref{eq:g635} follows from the fact that for any vector $v$ in $\mathbb{R}^\ell$, $\lVert v \rVert_1 \le \sqrt{\ell} \lVert v \rVert_2$. By the assumption of this lemma, $1 - 2 r \sigma / (1-\sigma^2) > 0$. Since $v$ is nonzero, either $\lVert v_A \rVert_2 > 0$ or $\lVert v_{\mathrm{I}} \rVert_2 >0$, which implies that the term in \eqref{eq:g635} is positive. By definition of positive definiteness (Definition~\ref{def:pd}), this implies that $\Sigma_{I_1,\dots,I_r}^{-1}$ is positive definite. In particular, this implies that $\Sigma_{I_1,\dots,I_r}^{-1}$ is invertible. 

The positive definiteness of $\Sigma_{I_1,\dots,I_r}$ follows from the positive definiteness of $\Sigma_{I_1,\dots,I_r}^{-1}$: a matrix $M$ is positive definite if and only if $M^{-1}$ is positive definite.

Note that the calculation in \eqref{eq:g635} implies that lowest eigenvalue of $A$ is at least $1 - 2r\sigma/\left(1-\sigma^2\right)$. Since the determinant is the multiplication of all eigenvalues, and using \eqref{eq:sig-decomp}, it holds that
\[
\det \Sigma_{I_1,\dots, I_r}^{-1}
= \det A
\ge \left(1 - 2r\sigma/\left(1-\sigma^2\right)\right)^{2r}
\ge 1 - 4r^2 /\left(1-\sigma^2\right)
\ge 1/2,
\]
where the last inequality follows from \eqref{eq:sigmab} and the assumption of this lemma that $r \le b$. This concludes the bound on $\det \Sigma_{I_1,\dots,I_r} = \left(\det \Sigma_{I_1,\dots,I_r}^{-1}\right)^{-1}$.

The matrix $\Sigma_{I_1,\dots,I_r}$ is symmetric due to the fact that $\Sigma_{I_1,\dots,I_r}^{-1}$ is symmetric (see \eqref{eq:sig-decomp} and \eqref{eq:Adef}) and the fact that for any symmetric invertible matrix $M$, $M^{-1}$ is symmetric.

Equation~\eqref{eq:Adef} implies that
\begin{equation}\label{eq:distAI}
\lvert \left( A - \mathrm{I}_{2r}\right)(i,j) \rvert 
\le \begin{cases}
\frac{\sigma^2}{1-\sigma^2}r \le \frac{\sigma}{1-\sigma^2} & i=j \\
\frac{\sigma}{1-\sigma^2} & i\ne j
\end{cases}
\end{equation}
using the assumption using \eqref{eq:sigmab} and the assumption $r\le b$ which imply $\sigma r \le \sigma b \le 1$. By induction on $\ell = 1,2,\dots$, one obtains that
\begin{equation}\label{eq:sigmainf}
\max\lvert (A-\mathrm{I}_{2r})^\ell\rvert \le
\frac{\sigma}{1-\sigma^2} \left( \frac{2r \sigma}{1-\sigma^2} \right)^{\ell-1}.
\end{equation}
For $\ell = 1$ it follows from \eqref{eq:distAI} and for $\ell > 1$:
\begin{align*}
&\max \left\lvert (A-\mathrm{I}_{2r})^\ell\right\rvert
=\max \left\lvert (A-\mathrm{I}_{2r})^{\ell-1} (A-\mathrm{I}_{2r}) \right\rvert\le \\
&2r \max \left\lvert (A-\mathrm{I}_{2r})^{\ell-1}\right\rvert \max \left\lvert A-\mathrm{I}_{2r} \right\rvert
\le\frac{\sigma}{1-\sigma^2} \left( \frac{2r \sigma}{1-\sigma^2} \right)^{\ell-1}
\end{align*}
where the first inequality follows from the formula for matrix multiplication.
Given a squared matrix $M$ its Neumann series is defined as 
\[
\sum_{\ell=0}^\infty M^\ell.
\]
If the Neumann series of $M$ converges then $(\mathrm{I}-M)^{-1}$ exists and equals the Neumann series of $M$ ($\mathrm{I}$ is the identity matrix). Substituting $M =\mathrm{I}_{2r} -A$, inequality~\eqref{eq:sigmainf} implies that
\begin{equation}\label{eq:76}
\sum_{\ell=0}^\infty \max \left\lvert (I_{2r}-A)^{\ell} \right\rvert_\infty
\le \sum_{\ell=0}^\infty \left( \frac{2r \sigma}{1-\sigma^2} \right)^\ell
\le \frac{1}{1-2r \sigma/(1-\sigma^2)}
\le 2,
\end{equation}
hence, the series converges in absolute value, therefore it converges and equals $(\mathrm{I}_{2r} - (\mathrm{I}_{2r} - A))^{-1} = A^{-1}$. In particular, $\max \lvert A^{-1} \rvert \le 2$. Using \eqref{eq:sig-decomp} one can verify that
\[
\Sigma_{I_1,\dots,I_r} = \begin{pmatrix}
A^{-1} & 0 \\
0 & \mathrm{I}_{d-2r}
\end{pmatrix}
\]
which concludes the proof.
\end{proof}

Fix $I_1, \dots, I_r \in \mathcal{I}_2$ for some $r \le b$. It holds that
\begin{align}
\eta_0(x) \prod_{i=1}^r \frac{\eta_{I_i}(x)}{\eta_0(x)}
&= \frac{1}{\sqrt{\det(2 \pi \Sigma_0)}} \exp\left(-\frac{1}{2}x^t \Sigma_0^{-1} x\right) \prod_{i=1}^r \frac{1}{\sqrt{1-\sigma^2}} 
\exp \left( -\frac{1}{2}x^t \left( \Sigma_I^{-1} - \Sigma_0^{-1} \right) x \right)\label{eq:g1}\\
&= (2\pi)^{-d/2} (1-\sigma^2)^{-r/2} \exp \left(-\frac{1}{2} x^t \left( \sum_{i=1}^r\Sigma_{I_i}^{-1} - (r-1)\Sigma_0^{-1} \right)x\right) \notag\\
&= \sqrt{\frac{\det \Sigma_{I_1, \dots, I_r}}{(1-\sigma^2)^r}} \frac{1}{\sqrt{\det(2 \pi \Sigma_{I_1,\dots,I_r})}} \exp \left(- \frac{1}{2}x^t \Sigma_{I_1,\dots,I_r}^{-1}x\right) \label{eq:prod-gaus}
\end{align}
using $\det \Sigma_{I_i} = 1-\sigma^2$ from Lemma~\ref{lem:sigma-I}. In particular, the RHS of \eqref{eq:g1} is the density of a $d$-variate normal distribution with mean  $0$ and covariance $\Sigma_{I_1,\dots,I_r}$, multiplied by a constant. Hence,
\begin{align}
\mathbb{E}_{X \sim \eta_0} \prod_{i=1}^r \frac{\eta_{I_i}(X)}{\eta_0(X)}
&= \int_{x \in \mathbb{R}^d} \eta_0(x) \prod_{i=1}^r \frac{\eta_{I_i}(x)}{\eta_0(x)} \notag\\
&= \sqrt{\frac{\det \Sigma_{I_1, \dots, I_r}}{(1-\sigma^2)^r}} \int_{x \in \mathbb{R}^d} \frac{1}{\sqrt{\det(2 \pi \Sigma_{I_1,\dots,I_r})}} \exp \left(- x^t \Sigma_{I_1,\dots,I_r}^{-1}x / 2 \right) \label{eq:g2}\\
&=\sqrt{\frac{\det \Sigma_{I_1, \dots, I_r}}{(1-\sigma^2)^r}},\label{eq:g5}
\end{align}
where the last equation holds since the integral in \eqref{eq:g2} is over the density function of a probability distribution.

\begin{lemma} \label{lem:gzero}
Assume that $I_1,\dots, I_r \in \mathcal{I}_2$ for some $r \le b$ such that there exists $j \in \{1,\dots,d\}$ which satisfies $j \notin \bigcup_{i=1}^{r-1} I_i$ and $j \in I_r$. Then
\begin{equation} \label{eq:g9}
\mathbb{E}_{X\sim \eta_0} \prod_{i=1}^{r} \left( \frac{\eta_{I_i}(X)}{\eta_0(X)}-1 \right)
= 0.
\end{equation}
\end{lemma}

\begin{proof}
We will start by showing that $\det \Sigma_{I_1,\dots,I_r}^{-1} = \frac{1}{1-\sigma^2} \det \Sigma^{-1}_{I_1,\dots,I_{r-1}}$.
By Proposition~\ref{prop:perm} one can assume that the element $j$ unique to $I_r$ is $1$ and that $I_r = \{1,2\}$. Then, \eqref{eq:Adef} implies that
\begin{align}
\Sigma_{I_1,\dots,I_{r-1}}^{-1} = \frac{1}{1-\sigma^2}\begin{pmatrix}
1 - \sigma^2 & 0 \\
0 & A
\end{pmatrix}\notag
\end{align}
where $A \in \mathbb{R}^{d\times d}$ and the two $0$-blocks contain $d-1$ zeros. From items \ref{itm:det-const} and \ref{itm:det-block} of Proposition~\ref{prop:det},
\begin{equation} \label{eq:g7}
\det \Sigma_{I_1,\dots,I_{r-1}}^{-1}
= (1-\sigma^2)^{-d} (1-\sigma^2) \det A.
\end{equation}
Additionally, from \eqref{eq:def-sig1r} and \eqref{eq:sigma-minus-I},

\begin{equation} \label{eq:gm1r}
\Sigma_{I_1,\dots,I_r}^{-1} 
= \Sigma_{I_1,\dots, I_{r-1}}^{-1} + \Sigma_{I_r}^{-1} - \Sigma_0^{-1}
= \frac{1}{1-\sigma^2} \begin{pmatrix}
1 & -\sigma & 0 \\
-\sigma & A_{00} + \sigma^2 & A_{01} \\
0 & A_{10} & A_{11}
\end{pmatrix},
\end{equation}
where 
\[
A = \begin{pmatrix}
A_{00} & A_{01} \\
A_{10} & A_{11}
\end{pmatrix}
\]
such that $A_{00} \in \mathbb{R}$, $A_{11} \in \mathbb{R}^{(d-2)\times (d-2)}$, $A_{01} \in \mathbb{R}^{1\times (d-2)}$ and $A_{10} \in \mathbb{R}^{(d-2)\times 1}$. Additionally, the two zero blocks in \eqref{eq:gm1r} contain $d-2$ zeros and the $1$ and $-\sigma$ blocks contain one entry. Proposition~\ref{prop:det} imply that
\begin{align}
(1-\sigma^2)^r \det \Sigma^{-1}_{I_1,\dots,I_r}
&= \det \begin{pmatrix}
1 & -\sigma & 0 \\
-\sigma & A_{00} + \sigma^2 & A_{01} \\
0 & A_{10} & A_{11}
\end{pmatrix} \label{eq:g654}\\
&= \det \begin{pmatrix}
A_{00} + \sigma^2 & A_{01}\\
A_{10} & A_{11}
\end{pmatrix}
- (-\sigma) \det \begin{pmatrix}
-\sigma & A_{01} \\
0 & A_{11}
\end{pmatrix} \label{eq:g655}\\
&= \det \begin{pmatrix}
A_{00} & A_{01} \\
A_{10} & A_{11}
\end{pmatrix}
+ \det \begin{pmatrix}
\sigma^2 & A_{01} \\
0 & A_{11}
\end{pmatrix}
- \sigma^2 \det A_{11} \label{eq:g657} \\
&= \det A. \label{eq:g8}
\end{align}
where \eqref{eq:g654} follows from item~\ref{itm:det-const} of Proposition~\ref{prop:det}, \eqref{eq:g655} follows from item~\ref{itm:det-rec}, \eqref{eq:g657} follows from items~\ref{itm:det-sum}~and~\ref{itm:det-block} and \eqref{eq:g8} follows from item~\ref{itm:det-block}.
Equations \eqref{eq:g7} and \eqref{eq:g8} imply that 
\[\det \Sigma_{I_1,\dots, I_r}^{-1} = (1-\sigma^2)^{-1} \det \Sigma_{I_1,\dots,I_{r-1}}^{-1}, \]
hence
\[\det \Sigma_{I_1,\dots, I_r} = (1-\sigma^2) \det \Sigma_{I_1,\dots,I_{r-1}}, \]
therefore \eqref{eq:g5} implies that 
\begin{equation}\label{eq:g10}
\mathbb{E}_{X\sim \eta_0} \prod_{i=1}^{r} \frac{\eta_{I_i}(X)}{\eta_0(X)}
= \mathbb{E}_{X\sim \eta_0} \prod_{i=1}^{r-1} \frac{\eta_{I_i}(X)}{\eta_0(X)}.
\end{equation}
Note that \eqref{eq:g10} can be applied when substituting $\{1,\dots, r-1\}$ with any subset, namely, for any $S \subseteq \{1,\dots,r-1 \}$,
\begin{equation}\label{eq:g11}
\mathbb{E}_{X\sim \eta_0} \prod_{i\in S \cup \{r\}} \frac{\eta_{I_i}(X)}{\eta_0(X)}
= \mathbb{E}_{X\sim \eta_0} \prod_{i\in S} \frac{\eta_{I_i}(X)}{\eta_0(X)}.
\end{equation}
Indeed, for any such $S$, $\{ I_i \}_{i \in S \cup \{r\}}$ satisfy the conditions of Lemma~\ref{lem:gzero}.
To conclude the proof, note that
\begin{align*}
\mathbb{E}_{X\sim \eta_0} \prod_{i=1}^{r} \left( \frac{\eta_{I_i}(X)}{\eta_0(X)}-1 \right)
&= \sum_{S \subseteq \{1,\dots, r\}}(-1)^{r-\lvert S \rvert} \prod_{i \in S} \frac{\eta_{I_i}(X)}{\eta_0(X)}  \\
&= \sum_{S \subseteq \{1,\dots, r-1\}}(-1)^{r-\lvert S \rvert} \left( 
\prod_{i \in S} \frac{\eta_{I_i}(X)}{\eta_0(X)}  - \prod_{i \in S \cup\{r\}} \frac{\eta_{I_i}(X)}{\eta_0(X)} \right) \\
&=0,
\end{align*}
where the last equation follows from \eqref{eq:g11}.
\end{proof}

\subsubsection{Applying Theorem~\ref{thm:main} for Truncated Gaussians} \label{sec:gau-trunc}

We apply Theorem~\ref{thm:main} on a family of truncated normal distributions, which is a $\CD(\rho)$ family for some small value of $\rho$. Recall that $b$ is a constant integer defined above ($b=5$).
Define 
\begin{equation} \label{eq:defp}
p =\sigma^b\left( 64 b n \binom{n}{\le b} 2^b \right)^{-1}
\end{equation}
and
\begin{equation} \label{eq:R-def}
R=\max \left(\sqrt{2 \ln (2 d m n)}, \sqrt{4 \ln \frac{2d}{p}}, \sqrt{2 \ln (d/\sigma)}, 1 \right).
\end{equation}
For all $0 < \sigma < 1$ and $I \in \mathcal{I}_2 \cup \{0\}$, let $\eta_{I,\sigma,R}$ be the truncation of $\eta_{I,\sigma}$ to $[-R,R]^d$, namely,
\[
\eta_{I,\sigma,R}(x)
= \frac{1}{\eta_{I,\sigma}\left( [-R,R]^d \right)} \begin{cases}
\eta_{I,\sigma}(x) & x \in [-R,R]^d \\
0 & x \notin [-R,R]^d.
\end{cases}
\]
We write shortly $\eta_{I,R}$ when $\sigma$ is implied from context. Define $\mathcal{G}_{\sigma,R} = \{ \eta_{I,\sigma,R} \colon{I \in \mathcal{I}_2}\}$.

Here is a well known tail bound for the normal distribution.
\begin{proposition} \label{prop:norm-tail}
Let $W$ be a random variable distributed normally with mean $0$ and variance $\sigma^2$. Then,
for any $w > 0$,
\begin{equation} \label{eq:norm-tail}
\Pr[\lvert W \rvert \ge w]
\le \frac{2e^{-w^2/(2\sigma^2)}}{(w/\sigma) \sqrt{2 \pi}}.
\end{equation}
\end{proposition}

\begin{proof}
Start by assuming that $\sigma = 1$. Then,
\begin{equation*}
\Pr[\lvert W \rvert \ge w]
= 2 \int_{t = w}^\infty \frac{1}{\sqrt{2 \pi}} e^{-t^2/2} 
\le 2\int_{t = w}^\infty \frac{t}{w} \frac{1}{\sqrt{2 \pi}} e^{-t^2/2} 
= \frac{2 e^{-w^2/2}}{w \sqrt{2 \pi}}.
\end{equation*}
Assuming that $\sigma \ne 1$, one can apply \eqref{eq:norm-tail} on $W/\sigma$ which is distributed normally with mean $0$ and variance $1$ and obtain
\[
\Pr[\lvert W \rvert \ge w]
= \Pr\left[\left\lvert \frac{W}{\sigma}\right\rvert \ge \frac{w}{\sigma}\right]
\le \frac{2e^{-w^2/(2\sigma^2)}}{(w/\sigma) \sqrt{2 \pi}}.
\]
\end{proof}

By substituting $w = R$, and recalling that $R \ge 1$ by definition, we get that for any normally distributed $W$ with zero mean,
\begin{equation} \label{eq:tail-R}
\Pr[\lvert W \rvert \ge R] \le \mathrm{Var}(W)^{1/2} e^{-\frac{1}{2} R^2/\mathrm{Var}(W)}.
\end{equation}
In particular, this implies that if $X=(X_1,\dots, X_d) \sim \eta_I$ for some $I \in \mathcal{I}_2 \cup \{0\}$,
\begin{equation} \label{eq:tail-X}
\Pr\left[X \notin [-R,R]^d\right] 
\le \sum_{i=1}^d \Pr\left[ \left\lvert X_i \right\rvert \ge R \right]
\le d e^{-R^2/2}
\le \min \left(
\frac{1}{2mn},
p,
\sigma
\right).
\end{equation}
\begin{lemma} \label{lem:app-zero}
Let $1 \le r \le b$ be an integer and let $I_1,\dots,I_r \in \mathcal{I}_2$ be distinct sets. Then
\[
\left\lvert \mathbb{E}_{X\sim \eta_0} \prod_{i=1}^r \left( \frac{\eta_{I_i}(X)}{\eta_0(X)}-1 \right) 
-  \mathbb{E}_{X\sim \eta_{0,R}} \prod_{i=1}^r \left( \frac{\eta_{I_i,R}(X)}{\eta_{0,R}(X)}-1 \right) \right\rvert
\le \frac{\sigma^b}{4n {\binom{n}{\le b}}}.
\]
\end{lemma}

\begin{proof}
We start by showing that 
\begin{equation}\label{eq:app-aux}
\left\lvert \mathbb{E}_{X\sim \eta_0} \prod_{i=1}^r \frac{\eta_{I_i}(X)}{\eta_0(X)} 
-  \mathbb{E}_{X\sim \eta_{0,R}} \prod_{i=1}^r  \frac{\eta_{I_i,R}(X)}{\eta_{0,R}(X)} \right\rvert
\le \frac{\sigma^b}{4n {\binom{n}{\le b}} 2^b}.
\end{equation}
Note that
\begin{align}
&\left\lvert \mathbb{E}_{X\sim \eta_0} \prod_{i=1}^r \frac{\eta_{I_i}(X)}{\eta_0(X)} - \mathbb{E}_{X\sim \eta_{0,R}} \prod_{i=1}^r  \frac{\eta_{I_i,R}(X)}{\eta_{0,R}(X)}\right\rvert\label{eq:g633}\\
&\le \left\lvert \mathbb{E}_{X\sim \eta_0}\left[ \prod_{i=1}^r \frac{\eta_{I_i}(X)}{\eta_0(X)} \middle| X \in [-R,R]^d \right] \Pr_{X \sim \eta_0}\left[X \in [-R,R]^d\right] - \mathbb{E}_{X\sim \eta_{0,R}} \prod_{i=1}^r  \frac{\eta_{I_i,R}(X)}{\eta_{0,R}(X)} \right\rvert \label{eq:g786} \\
&+ \left\lvert \mathbb{E}_{X\sim \eta_0}\left[ \prod_{i=1}^r \frac{\eta_{I_i}(X)}{\eta_0(X)} \middle| X \notin [-R,R]^d \right] \Pr_{X \sim \eta_0}\left[X \notin [-R,R]^d\right] \right\rvert\label{eq:g787}.
\end{align}
We will bound the terms \eqref{eq:g786} and \eqref{eq:g787} separately. Let $W=(W_1,\dots,W_d)$ be a random variable distributed normally with mean $0$ and covariance $\Sigma_{I_1,\dots,I_r}$ and let $P_W$ denote its density function.
\begin{align}
\eqref{eq:g786}
&= \left\lvert \frac{\prod_{i=1}^r \eta_{I_i}\left([-R,R]^d\right)}{\eta_0\left([-R,R]^d\right)^{r-1}}\mathbb{E}_{X\sim \eta_0}\left[ \prod_{i=1}^r \frac{\eta_{I_i}(X) / \eta_{I_i}\left([-R,R]^d\right)}{\eta_0(X)/\eta_0\left([-R,R]^d\right)} \middle| X \in [-R,R]^d \right] 
- \mathbb{E}_{X\sim \eta_{0,R}} \prod_{i=1}^r  \frac{\eta_{I_i,R}(X)}{\eta_{0,R}(X)}\right\rvert\\
&=\left\lvert \mathbb{E}_{X\sim \eta_0}\left[ \prod_{i=1}^r \frac{\eta_{I_i}(X) / \eta_{I_i}\left([-R,R]^d\right)}{\eta_0(X)/\eta_0\left([-R,R]^d\right)} \middle| X \in [-R,R]^d \right] - \mathbb{E}_{X\sim \eta_{0,R}} \prod_{i=1}^r  \frac{\eta_{I_i,R}(X)}{\eta_{0,R}(X)} \right. \notag\\
&+ \left. \left( \frac{\prod_{i=1}^r \eta_{I_i}\left([-R,R]^d\right)}{\eta_0\left([-R,R]^d\right)^{r-1}}-1 \right) 
\frac{\eta_0\left([-R,R]^d\right)^r}{\prod_{i=1}^r \eta_{I_i}\left([-R,R]^d\right)}
\mathbb{E}_{X\sim \eta_0}\left[ \prod_{i=1}^r \frac{\eta_{I_i}(X)}{\eta_0(X)} \middle| X \in [-R,R]^d \right] \right\rvert \notag\\
&= \left\lvert \left( \frac{\prod_{i=1}^r \eta_{I_i}\left([-R,R]^d\right)}{\eta_0\left([-R,R]^d\right)^{r-1}}-1 \right) 
\frac{\eta_0\left([-R,R]^d\right)^r}{\prod_{i=1}^r \eta_{I_i}\left([-R,R]^d\right)}
\mathbb{E}_{X\sim \eta_0}\left[ \prod_{i=1}^r \frac{\eta_{I_i}(X)}{\eta_0(X)} \middle| X \in [-R,R]^d \right] \right\rvert \notag\\
&= \left\lvert \left( \frac{\prod_{i=1}^r \eta_{I_i}\left([-R,R]^d\right)}{\eta_0\left([-R,R]^d\right)^{r-1}}-1 \right) 
\frac{\eta_0\left([-R,R]^d\right)^r}{\prod_{i=1}^r \eta_{I_i}\left([-R,R]^d\right)}
\int_{x \in [-R,R]^d} \frac{\eta_0(x)}{\eta_0\left([-R,R]^d\right)} \prod_{i=1}^r \frac{\eta_{I_i}(x)}{\eta_0(x)} \right\rvert \notag\\
&= \left\lvert \left( \frac{\prod_{i=1}^r \eta_{I_i}\left([-R,R]^d\right)}{\eta_0\left([-R,R]^d\right)^{r-1}}-1 \right) 
\frac{\eta_0\left([-R,R]^d\right)^{r-1}}{\prod_{i=1}^r \eta_{I_i}\left([-R,R]^d\right)}
\sqrt{\frac{\det \Sigma_{I_1, \dots, I_r}}{(1-\sigma^2)^r}} 
\int_{x \in [-R,R]^d} P_W(x) \right\rvert \label{eq:g658}\\
&\le\left\lvert\left(1 - \frac{\eta_0\left([-R,R]^d\right)^{r-1}}{\prod_{i=1}^r \eta_{I_i}\left([-R,R]^d\right)} \right)
\sqrt{\frac{\det \Sigma_{I_1, \dots, I_r}}{(1-\sigma^2)^r}} \right\rvert \notag\\
&\le\left\lvert\left(1 - \frac{\eta_0\left([-R,R]^d\right)^{r-1}}{\prod_{i=1}^r \eta_{I_i}\left([-R,R]^d\right)} \right)
\sqrt{\frac{2}{1-\sigma^2 r}}
\right\rvert \label{eq:g623} \\
&\le\frac{2}{\prod_{i=1}^r \eta_{I_i}\left([-R,R]^d\right)} \left\lvert \prod_{i=1}^r \eta_{I_i}\left([-R,R]^d\right) - \eta_0\left([-R,R]^d\right)^{r-1} \right\rvert
\label{eq:g624} \\
&\le \frac{2}{(1-p)^r} \left\lvert \prod_{i=1}^r \eta_{I_i}\left([-R,R]^d\right) - \eta_0\left([-R,R]^d\right)^{r-1} \right\rvert \label{eq:g659}\\
&\le 4 \left\lvert \prod_{i=1}^r \eta_{I_i}\left([-R,R]^d\right) - \eta_0\left([-R,R]^d\right)^{r-1} \right\rvert
\label{eq:g625}.
\end{align}
where \eqref{eq:g658} follows from \eqref{eq:prod-gaus}, \eqref{eq:g623} follows from Lemma~\ref{lem:g-det}, \eqref{eq:g624} follows \eqref{eq:sigmab} which implies that $\sigma^2 r \le \sigma r \le \sigma b \le 1/2$ ,\eqref{eq:g659} follows from \eqref{eq:tail-X} and \eqref{eq:g625} follows from $(1-p)^r \ge 1 - rp \ge 1 - bp \ge 1/2$. We will estimate the term in \eqref{eq:g625}. From \eqref{eq:tail-X},
\begin{equation*}
\prod_{i=1}^r \eta_{I_i}\left([-R,R]^d\right) - \eta_0\left([-R,R]^d\right)^{r-1}
\le 1 - (1-p)^{r-1}
\le 1 - (1-p)^b
\le 1 - (1-bp)
= bp \notag.
\end{equation*}
Additionally,
\begin{equation*}
\prod_{i=1}^r \eta_{I_i}\left([-R,R]^d\right) - \eta_0\left([-R,R]^d\right)^{r-1}
\ge (1-p)^r - 1
\ge (1-p)^b - 1
\ge (1-bp) - 1
= -bp \notag.
\end{equation*}
Hence, by definition of $p$ in~\eqref{eq:defp},
\begin{equation} \label{eq:g626}
\eqref{eq:g786}
\le \eqref{eq:g625}
\le 4 bp
\le \frac{\sigma^b}{8n {\binom{n}{\le b}} 2^b}.
\end{equation}
Next, we will bound \eqref{eq:g787}.

It follows from Lemma~\ref{lem:g-det} that $\max \lvert \Sigma_{I_1,\dots,I_r} \rvert \le 2$ hence the variance of each coordinate of $W$ (the $d$-variate normally distributed random variable with mean $0$ and covariance $\Sigma_{I_1,\dots,I_r}$) is at most $2$. Hence,
\begin{equation} \label{eq:g651}
\Pr\left[W \notin [-R,R]^d\right]
\le \sum_{i=1}^d \Pr\left[\lvert W_i \rvert > d \right]
\le p,
\end{equation}
using \eqref{eq:tail-R} and $R \ge \sqrt{4 \ln \frac{2}{dp}}$. Therefore,

\begin{align}
\eqref{eq:g787}
&= \left(1 - \eta_0\left([-R,R]^d\right)\right) \int_{x \in \mathbb{R}^d \setminus [-R,R]^d} \frac{\eta_0(x)}{1 - \eta_0\left([-R,R]^d\right)} \prod_{i=1}^r \frac{\eta_{I_i}(X)}{\eta_0(X)} \notag \\
&= \sqrt{\frac{\det \Sigma_{I_1, \dots, I_r}}{\left(1-\sigma^2\right)^r}} \int_{x \in \mathbb{R}^d \setminus [-R,R]^d} P_W(x) \label{eq:g627} \\
&\le 4 \Pr\left[ W \notin [-R,R]^d \right] \label{eq:g628} \\
&\le 4 p \label{eq:g653}\\
&\le \frac{\sigma^b}{8n {\binom{n}{\le b}} 2^b} \label{eq:g632}
\end{align}
where \eqref{eq:g627} follow from \eqref{eq:prod-gaus}, \eqref{eq:g628} follows from the same calculation as in \eqref{eq:g623} and \eqref{eq:g624}, \eqref{eq:g653} follows from \eqref{eq:g651} and \eqref{eq:g632} follows from the definition of $p$ in \eqref{eq:defp}.

Note that \eqref{eq:g633}, \eqref{eq:g626} and \eqref{eq:g632} imply \eqref{eq:app-aux}. To conclude the proof,
\begin{align}
&\left\lvert \mathbb{E}_{X\sim \eta_0} \prod_{i=1}^r \left( \frac{\eta_{I_i}(X)}{\eta_0(X)}-1 \right) 
-  \mathbb{E}_{X\sim \eta_{0,R}} \prod_{i=1}^r \left( \frac{\eta_{I_i,R}(X)}{\eta_{0,R}(X)}-1 \right) \right\rvert \notag\\
&\le \sum_{S \subseteq \{1,\dots, r\}} 
\left\lvert \mathbb{E}_{X\sim \eta_0} \prod_{i\in S} \frac{\eta_{I_i}(X)}{\eta_0(X)}
-  \mathbb{E}_{X\sim \eta_{0,R}} \prod_{i=1}^r \frac{\eta_{I_i,R}(X)}{\eta_{0,R}(X)}\right\rvert
\le \frac{\sigma^b}{4n {\binom{n}{\le b}}},\notag
\end{align}
where the last inequality follows from \eqref{eq:app-aux}.
\end{proof}

Define
\begin{equation} \label{eq:grhodef}
\rho = \frac{2\sigma^2}{1-\sigma^2} + \frac{4\sigma R^2}{\left(1-\sigma^2\right)^2} + 2\sigma.
\end{equation}
We will show that $\mathcal{G}_{\sigma,R}$, the family of truncated distributions, is $\CD(\rho)$ for the value of $\rho$ presented in the following lemma.

\begin{lemma} \label{lem:g-rho}
For any $x \in [-R,R]^d$ and any $I \in \mathcal{I}_2$,
\[
\left \lvert \frac{\eta_{I,R}(x)}{\eta_{0,R}(x)} - 1 \right\rvert
\le \rho.
\]
\end{lemma}

\begin{proof}
Fix some $I = \{ i,j \}$ and $x \in [-R,R]^d$.
\begin{align}
\frac{\eta_{I,\sigma}(x)}{\eta_{0,\sigma}(x)}
&= \sqrt{\frac{\det \Sigma_0}{\det \Sigma_I}} \exp\left( -\frac{1}{2} x^t \left(\Sigma_I^{-1} - \mathrm{I}_d\right) x\right) \notag \\
&= \frac{1}{1-\sigma^2} \exp\left( -\frac{x_i^2 \sigma^2 - 2 \sigma x_i x_j + x_j^2 \sigma^2}{2\left(1-\sigma^2\right)} \right) \label{eq:g638} \\
&\le \frac{1}{1-\sigma^2} \exp\left(\frac{\sigma R^2}{1-\sigma^2} \right) \notag \\
&\le \frac{1}{1-\sigma^2} \left(1 + \frac{2\sigma R^2}{1-\sigma^2} \right) \label{eq:g639} \\
&= 1 + \frac{\sigma^2}{1-\sigma^2} + \frac{2\sigma R^2}{\left(1-\sigma^2\right)^2} \label{eq:g641}
\end{align}
where \eqref{eq:g638} follows from \eqref{eq:sigma-minus-I} and from $\det \Sigma_I = 1 - \sigma^2$ which is proved in Lemma~\ref{lem:sigma-I}, \eqref{eq:g639} follows from $\sigma R^2 /\left(1-\sigma^2\right) \le 1$ which holds if the constant $C$ in Theorem~\ref{thm:normal} is sufficiently large and the fact that $e^x \le 1 + 2x$ for all $0 \le x \le 1$. Additionally, 
\begin{align}
\frac{\eta_{I,\sigma}}{\eta_{0,\sigma}}
&= \frac{1}{1-\sigma^2} \exp\left( -\frac{x_i^2 \sigma^2 - 2 \sigma x_i x_j + x_j^2 \sigma^2}{2\left(1-\sigma^2\right)} \right) \label{eq:g640} \\
&\ge \frac{1}{1-\sigma^2} \left(1 -\frac{x_i^2 \sigma^2 - 2 \sigma x_i x_j + x_j^2 \sigma^2}{2\left(1-\sigma^2\right)} \right) \label{eq:g660} \\
&\ge 1 -\frac{R^2 \sigma^2}{\left(1-\sigma^2\right)^2} \label{eq:g642}
\end{align}
where \eqref{eq:g640} is \eqref{eq:g638} and \eqref{eq:g660} follows from $e^{-x} \ge 1-x$ for all $x \in \mathbb{R}$. Together, \eqref{eq:g641} and \eqref{eq:g642} imply that
\begin{equation} \label{eq:g643}
\left\lvert \frac{\eta_{I,\sigma}}{\eta_{0,\sigma}} - 1 \right\rvert 
\le \rho'
:= \frac{\sigma^2}{1-\sigma^2} + \frac{2\sigma R^2}{\left(1-\sigma^2\right)^2}.
\end{equation}
Hence,
\begin{align} \label{eq:g-rho}
\left \lvert \frac{\eta_{I,R}(x)}{\eta_{0,R}(x)} - 1 \right\rvert
&\le \left \lvert \frac{\eta_{I,R}(x)}{\eta_{0,R}(x)} - \frac{\eta_I(x)}{\eta_0(x)} \right\rvert
+ \left \lvert \frac{\eta_I(x)}{\eta_0(x)} - 1 \right\rvert \\
&= \frac{\eta_I(x)}{\eta_0(x)} \left \lvert \frac{\eta_0\left([-R,R]^d\right)}{\eta_I\left([-R,R]^d\right)} - 1 \right\rvert
+ \left \lvert \frac{\eta_I(x)}{\eta_0(x)} - 1 \right\rvert \notag \\
&\le (1+\rho') \left \lvert \frac{\eta_0\left([-R,R]^d\right)}{\eta_I\left([-R,R]^d\right)} - 1 \right\rvert
+ \rho' \label{eq:g644} \\
&= (1+\rho') \frac{\left \lvert \eta_0\left([-R,R]^d\right)-\eta_I\left([-R,R]^d\right) \right\rvert}{\eta_I\left([-R,R]^d\right)} + \rho' \notag \\
&\le (1+\rho') \frac{\sigma}{1-\sigma} + \rho' \label{eq:g661} \\
&\le (1+\rho') 2\sigma + \rho' \label{eq:g662} \\
&\le 2\sigma + 2\rho' \label{eq:g663} \\
&= \rho
\end{align}
where \eqref{eq:g644} follows from \eqref{eq:g643}, \eqref{eq:g661} follows from \eqref{eq:tail-X} and \eqref{eq:g662} and \eqref{eq:g663} hold if the constant $C$ from Theorem~\ref{thm:normal} is sufficiently large such that $\sigma \le 1/2$.
\end{proof}

Next, we show that \eqref{eq:corr} holds. We start with an auxiliary lemma.

\begin{lemma} \label{lem:tuples-aux}
For any $3 \le \ell \le d$, $2 \le r \le d$, the number of collections of sets $\mathcal{J} \subseteq \mathcal{I}_r$, $\lvert \mathcal{J} \rvert = \ell$, for which no element of $\{ 1,\dots, d\}$ appears in exactly one set is at most $d^{\ell r/2} C(\ell, r)$,
where
\[
C(\ell,r)
= \begin{cases}
\frac{1}{(\ell r/2)!} {\binom{\binom{\ell r/2}{r}}{\ell}} & \ell r \text{ is even} \\
0 & \ell r \text{ is odd}
\end{cases}.
\]
\end{lemma}

\begin{proof}
For any $\mathcal{J}$ satisfying the condition of the lemma, every index $i \in \bigcup \mathcal{J}$ is a member of at least $2$ sets $I \in \mathcal{J}$. Therefore,
\[
\left\lvert \bigcup \mathcal{J} \right\rvert 
\le \frac{1}{2} \sum_{I \in \mathcal{J}} \lvert I \rvert
= \frac{1}{2} \ell r,
\] 
since $\mathcal{J}$ contains $\ell$ sets, each of size $r$.
Hence, each such $\mathcal{J}$ satisfies that $\bigcup \mathcal J$ is contained in some set $J$ of size $r \ell /2$. There are $\binom{d}{r \ell/2}$ sets $J$ of size $r\ell/2$. Each such $J$ is a super-set of $\binom{|J|}{r}$ sets of size $r$, hence there are 
\[\binom{\binom{|J|}{r}}{\ell}
= \binom{\binom{\ell r/2}{r}}{\ell}
\]
collections of $\ell$ subsets of $J$ of size $r$. In total, there can be no more than
\[
\binom{d}{\ell r/2} \binom{\binom{\ell r/2}{r}}{\ell}
\le \frac{d^{\ell r/2}}{(\ell r/2)!} \binom{\binom{\ell r/2}{r}}{\ell}
\]
collections $\mathcal{J}\subseteq \mathcal{I}_r$ of size $\ell$ for which $\bigtriangleup\mathcal{J} = \emptyset$. This completes the proof for the case that $\ell r$ is even.
If $r \ell$ is odd, there is no collection $\mathcal{J}\subseteq \mathcal{I}_r$ of $\ell$ sets which satisfies $\bigtriangleup \mathcal{J} = \emptyset$: at least one of the elements in $\bigcup \mathcal{J}$ has to appear in an odd number of sets.
\end{proof}

\begin{lemma} \label{lem:g-corr}
The following holds:
\begin{equation} \label{eq:lem-gcorr}
\sum_{S \subseteq \mathcal{I}_2 \colon \lvert S \rvert \ge 2} n^{-\lvert S \rvert/2} \rho^{-\lvert S \rvert} \left\lvert \mathbb{E}_{A \sim \eta_0} \prod_{i\in S} \left(\eta_{I,\sigma,R}(A)/\eta_0(A) - 1\right) \right\rvert
\le \frac{1}{n}.
\end{equation}
\end{lemma}

\begin{proof}
First, Lemma~\ref{lem:nk6} implies that the sum of terms corresponding to $\lvert S \rvert > 5$ is at most $1/(2n)$, assuming that the constant $C$ of Theorem~\ref{thm:normal} is sufficiently large. Recall that $b = 5$ by definition and we will bound the sum of terms corresponding to $2 \le \lvert S \rvert \le b$.
For any $2 \le \ell \le b$, let $\mathcal{U}_\ell$ be the set of all collections of pairs $S \subseteq \mathcal{I}_2$ of size $\lvert S \rvert = \ell$ for which no element $i \in [d]$ appears in exactly one pair $I \in S$. We will bound the sum of terms corresponding to $S \notin \bigcup_{\ell=2}^b \mathcal{U}_\ell$, $2 \le \lvert S \rvert \le b$: Lemma~\ref{lem:gzero} and Lemma~\ref{lem:app-zero} imply that each such $S$ contributes to the LHS of \eqref{eq:lem-gcorr} at most
\[
n^{-\lvert S \rvert /2} \rho^{-\lvert S \rvert} \frac{\rho^{b}}{4n {\binom{n}{\le b}}}
\le \frac{1}{4n {\binom{n}{\le b}}},
\]
where the inequality follows from $\lvert S \rvert \le b$ and $\rho \le 1$, assuming that the constant $C$ from Theorem~\ref{thm:normal} is sufficiently large.
The number of such sets is at most $\binom{n}{\le b}$, hence the total contribution of these sets is at most $1/(4n)$. Lastly, we bound the contribution of sets $S \in \bigcup_{\ell=2}^b \mathcal{U}_\ell$. It follows from Lemma~\ref{lem:tuples-aux} that there is a numerical constant $C'$ such that $\lvert \mathcal{U}_\ell\rvert \le C' d^\ell$ for all $2 \le \ell \le b$. Furthermore, it trivially holds that $\lvert\mathcal{U}_2\rvert = 0$. Each set $S$ contributes to the sum at most $n^{-\lvert S \rvert / 2}$, from Lemma~\ref{lem:g-rho}. Hence, the total contribution of sets $S \in \bigcup_{\ell=2}^b \mathcal{U}_\ell$ is at most
\[
\sum_{\ell=3}^b \left\lvert \mathcal{U}_\ell\right\rvert n^{-\ell / 2}
\le \frac{1}{n} \sum_{\ell=3}^b C' d^\ell n^{-\ell/2 - 1}
\le \frac{1}{n} \sum_{\ell=3}^b C' d^\ell n^{-\ell/6}
\le \frac{1}{n} \sum_{\ell=3}^b \frac{C'}{C}
\le \frac{1}{4n},
\]
where $C$ is the constant from Theorem~\ref{thm:main} and the last inequality holds if $C$ is sufficiently large.
\end{proof}

We apply \thmref{thm:main} on $\mathcal{G}_{R,\sigma}$. \lemref{lem:g-rho} implies that $\mathcal{G}_{\sigma,R}$
is a $\CD(\rho)$ family. Lemma~\ref{lem:g-corr} implies that \eqref{eq:corr} holds. Additionally, the definition of $\rho$ in \eqref{eq:grhodef} and the definition of $R$ in \eqref{eq:R-def} imply that if the constant $C$ from \thmref{thm:normal} is sufficiently large then the requirement that $\rho$ is sufficiently small with respect to $n$ and $k$ holds.

\subsubsection{From Truncated to Standard Gaussians} \label{sec:gau-trun-std}

To conclude the proof, we reduce the hardness of identifying a truncated normal distribution to the hardness of identifying a normal distribution, using the fact that with high probability, if we draw $mn$ samples from a normal distribution $\eta \in \mathcal{G}_\sigma$, they are all in $[-R,R]^d$.

\begin{lemma} \label{lem:trunc-to-no}
For any $0 < \sigma < 1$ and $0 < \varepsilon < 1$, if a protocol identifies $\eta \in \mathcal{G}_\sigma$ with a worst case error of $\varepsilon$ then it identifies $\eta \in \mathcal{G}_{\sigma,R}$ with a worst case error of at most $2 \varepsilon$.
\end{lemma}

\begin{proof}
Let $\pi$ be a protocol for identifying $\eta \in \mathcal{G}_\sigma$ with a worst case error of $\varepsilon$. Given an input containing the samples $x^{(1)},\dots,x^{(mn)}$ distributed by the $m$ parties, let $\Pi(x^{(1)},\dots,x^{(mn)}) \in \mathcal{I}_2$ be the random variable denoting the output of $\pi$ when the input is $x^{(1)},\dots,x^{(mn)}$. Then, for any $I \in \mathcal{I}_2$,
\begin{align}
\varepsilon
\ge& \Pr_{(X^{(1)},\dots,X^{(mn)}) \sim \eta_{I,\sigma}^{mn}} \left[ \Pi(X^{(1)},\dots,X^{(mn)}) \ne I \right] \notag \\
\ge& \Pr_{(X^{(1)},\dots,X^{(mn)}) \sim \eta_{I,\sigma}^{mn}} \left[ \Pi(X^{(1)},\dots,X^{(mn)}) \ne I , (X^{(1)},\dots,X^{(mn)}) \in \left( [-R,R]^d \right)^{mn} \right] \notag \\
=& \Pr_{(X^{(1)},\dots,X^{(mn)}) \sim \eta_{I,\sigma}^{mn}} \left[\left(X^{(1)},\dots,X^{(mn)}\right) \in \left( [-R,R]^d \right)^{mn}\right] \cdot \notag \\
& \Pr_{(X^{(1)},\dots,X^{(mn)}) \sim \eta_{I,\sigma}^{mn}} \left[ \Pi(X^{(1)},\dots,X^{(mn)}) \ne I \mid (X^{(1)},\dots,X^{(mn)}) \in \left( [-R,R]^d \right)^{mn} \right] \notag \\
\ge& \frac{1}{2} \Pr_{(X^{(1)},\dots,X^{(mn)}) \sim \eta_{I,\sigma}^{mn}} \left[ \Pi(X^{(1)},\dots,X^{(mn)}) \ne I \mid (X^{(1)},\dots,X^{(mn)}) \in \left( [-R,R]^d \right)^{mn} \right] \label{eq:g650}\\
=& \frac{1}{2} \Pr_{(X^{(1)},\dots,X^{(mn)}) \sim \eta_{I,\sigma,R}^{mn}} \left[ \Pi(X^{(1)},\dots,X^{(mn)}) \ne I \right].
\end{align}
where \eqref{eq:g650} follows from \eqref{eq:tail-X}.
\end{proof}

\section{Improved Results for Identifying Order-r Correlations}\label{sec:tuples}

As discussed in \subsecref{subsec:binary}, Theorems \ref{thm:subset-parity} and \ref{thm:mem-subset} assume that the correlation $\rho$ is sufficiently small compared to the other problem parameters (and in the communication-constrained case, that $n$ is sufficiently large). The following two theorems show that the assumptions can be somewhat relaxed, if we consider specifically the case of detecting order-$r$ correlations (that is, the family of coordinate subsets we consider are $\mathcal{U} = \{ I \in \mathcal{I} \colon \lvert I \rvert = r\}$, for some $r\geq 2$). In that case, we only require $n=\Omega(d^{3r})$ for $r$ even and $n = \Omega(d^{2r+\varepsilon})$ for $r$ odd in the communication-constrained case (whereas \thmref{thm:subset-parity} requires $n=\Omega(d^{6r})$), and in the memory-constrained case, only $\rho=O(d^{-3r/2})$ for $r$ even or even $\rho=O(d^{-(1+\epsilon)r})$ for $r$ odd (whereas \thmref{thm:mem-subset} requires $\rho=O(d^{-3r})$). 

\begin{theorem} \label{thm:tuples}
There exist numerical constants $C',C''$ and a positive function $C(r) \colon \mathbb{N} \to \mathbb{R}_+$ such that the following holds. Fix $2 \le r \le d-1$, and let $k = \binom{d}{r}$.
Let $n$ be an integer such that $n \ge d^{3 r} C(r)$.
Fix a number $0 < \rho \le (n \ln k)^{-1/2}/C'$.
Let $m\ge 1$ be an integer.
Then, any $(m,n)$ protocol identifying  $\mu \in \mathcal{P}_{\mathcal{I}_r,\rho}$ has a communication complexity of at least
\begin{equation}\label{eq:tuple-statement}
\frac{k}{C'' \rho^2 n \log (k^2 /(n\rho^2))}.
\end{equation}
Furthermore, if $r$ is odd then for any $0 < \varepsilon < 1$ there exists a number $C(r,\varepsilon)$ which depends only on $r$ and $\varepsilon$ such that \eqref{eq:tuple-statement} holds whenever $n \ge d^{(2+\varepsilon) r} C(r, \varepsilon)$.
\end{theorem}

\begin{theorem} \label{thm:tuples-mem}
There exist a numerical constant $C'$ and a positive function $C(r) \colon \mathbb{N} \to \mathbb{R}_+$ such that the following holds.
Fix $2 \le r \le d-1$ and fix a number $0 < \rho \le d^{-3r/2} \ln^{-1/2}d/C(r)$. For any integers $t,s\ge 1$, any $(t,s)$-algorithm identifying  $\mu \in \mathcal{P}_{\mathcal{I}_r,\rho}$ satisfies
\begin{equation*}
ts \ge \frac{\binom{d}{r}}{C' \rho^2\ln \binom{d}{r} }.
\end{equation*}
Furthermore, if $r$ is odd then for any $0 < \varepsilon < 1$ there exists a number $C(r,\varepsilon)\ge 1$ which depends only on $r$ and $\varepsilon$ such that \eqref{eq:tuple-statement} holds whenever $\rho \le d^{-(1+\varepsilon) r} /C(r, \varepsilon)$.
\end{theorem}

\thmref{thm:tuples} is derived from our general result (\thmref{thm:main}), by more delicately bounding the expression in \eqref{eq:corr}, allowing us to use larger values of $\rho$ and smaller values of $n$. A full proof is presented below. \thmref{thm:tuples-mem} is derived as a direct corollary of \thmref{thm:tuples}, using the same communication-to-memory reduction that we used for proving \thmref{thm:mem} based on \thmref{thm:main}. 

\subsection{Proof of \thmref{thm:tuples}}

Lemma~\ref{lem:tuples-aux} implies that:
\begin{equation} \label{eq:231}
\sum_{\mathcal{J} \subseteq \mathcal{\mathcal{I}_r} \colon \lvert \mathcal{J} \rvert = \ell} \left\lvert \mathbb{E} \prod_{I \in \mathcal{J}} B_I \right\rvert
\le k^{\ell r/2} C(\ell, r),
\end{equation}
for the value $C(\ell,r)$ appearing in this lemma. Indeed, from Lemma~\ref{lem:subsets}, the LHS of \eqref{eq:231} equals the number of collections of sets $\mathcal{J} \subseteq \mathcal{I}_r$ of size $\lvert \mathcal{J} \rvert = \ell$ for which $\bigtriangleup \mathcal{J} = \emptyset$. For any such $\mathcal{J}$, every index $i \in \bigcup \mathcal{J}$ is a member of an even number of sets $I \in \mathcal{J}$, hence there is no element in $\{ 1,\dots, d\}$ appearing in exactly one set from $\mathcal{J}$. This implies that the LHS of \eqref{eq:231} is at most the term bounded in Lemma~\ref{lem:tuples-aux}, hence \eqref{eq:231} holds.

To prove Theorem~\ref{thm:tuples}, it is sufficient to show that \eqref{eq:corr} holds. Under the conditions of Theorem~\ref{thm:tuples}, the requirement that $n \ge d^{3r} \ge \binom{d}{r}^3 = k^3 = k^{2(5 + 1)/(5-1)}$ and Lemma~\ref{lem:nk6} imply that the sum of all terms in \eqref{eq:corr} corresponding to $\lvert S \rvert > 5$ is at most $1/(2n)$. From \eqref{eq:corr-BI} and Lemma~\ref{lem:subsets} it holds that the sum of terms in \eqref{eq:corr} corresponding to $\lvert S \rvert = 2$ is zero, and by \eqref{eq:corr-BI} and \eqref{eq:231} the sum of terms corresponding to $\lvert S \rvert = \ell$ is at most $n^{-\ell/2} C(\ell,r) d^{\ell r/2}$, where $C(\ell, r)$ is the number from Lemma~\ref{lem:tuples-aux} . Hence, the sum of terms corresponding to $3 \le \lvert S \rvert \le 5$ is at most
\begin{align*}
\sum_{\ell=3}^5 n^{-\ell/2} d^{\ell r/2} C(\ell,r)
&= \frac{1}{n} \sum_{\ell=3}^5 n^{-\ell/2+1} d^{\ell r/2} C(\ell,r) \\
&\le \frac{1}{n} \sum_{\ell=3}^5 d^{-3\ell r/2+3r} C(r)^{-\ell/2+1} d^{\ell r/2} C(\ell,r) \\
&= \frac{1}{n} \sum_{\ell=3}^5 d^{-\ell r+3r} C(r)^{-\ell/2+1} C(\ell,r) \\
&\le \frac{1}{n} \sum_{\ell=3}^5 C(r)^{-\ell/2+1} C(\ell,r) \\
&\le \frac{1}{2n},
\end{align*}
where the last inequality holds whenever $C(r)$ is sufficiently large as a function of $C(\ell,r)$, for $3 \le \ell \le 5$. Whenever this holds, \eqref{eq:corr} holds.

Next, assume that $r$ is odd, fix $0 < \varepsilon < 1$, and let $\ell(\varepsilon)$ be the smallest integer which satisfies $2(\ell(\varepsilon)+1)/(\ell(\varepsilon)-1) \le 2+\varepsilon$. Since
$n \ge d^{(2+\varepsilon)r} \ge \binom{d}{r}^{2+\varepsilon} \ge \binom{d}{r}^{2 (\ell(\varepsilon)+1)/(\ell(\varepsilon)-1)}$, Lemma~\ref{lem:nk6} implies that the sum of all terms in \eqref{eq:corr} corresponding to $\lvert S \rvert > \ell(\varepsilon)$ is at most $1/(2n)$. As in the case of a general $r$, all terms corresponding to $\lvert S \rvert = 2$ are zero. Inequalities \eqref{eq:231} and \eqref{eq:corr-BI} imply that the sum of terms corresponding to $\lvert S \rvert = 3$ is zero. Similarly to the calculation in the case of a general $r$, the sum of terms corresponding to $4 \le \lvert S \rvert \le \ell(\varepsilon)$ is at most
\begin{align*}
\frac{1}{n} \sum_{\ell=4}^{\ell(\varepsilon)} n^{-\ell/2+1} d^{\ell r/2} C(\ell,r)
&\le \frac{1}{n} \sum_{\ell=4}^{\ell(\varepsilon)} d^{-\ell r + 2 r} C(r, \varepsilon)^{-\ell /2 +1} d^{\ell r/2} C(\ell,r) \\
&= \frac{1}{n} \sum_{\ell=4}^{\ell(\varepsilon)} d^{-\ell r/2 + 2 r} C(r, \varepsilon)^{-\ell /2 +1} C(\ell,r) \\
&\le \frac{1}{n} \sum_{\ell=4}^{\ell(\varepsilon)} C(r, \varepsilon)^{-\ell /2 +1} C(\ell,r) \\
&\le \frac{1}{2n},
\end{align*}
where the last inequality holds whenever $C(r, \varepsilon)$ is sufficiently large, which concludes the proof.

\section{Comparison to \cite{raz2016fast}} \label{sec:raz-compare}

\cite{raz2016fast} studied the problem of learning a linear function over $\mathbb{Z}_2^d$ (namely, $d$ dimensional vectors of integers modulo $2$): given samples $\left(x, y\right)$, where $x$ is picked uniformly at random from $\mathbb{Z}_2^d$ and $y = \langle w,x\rangle (\mod 2)$ for some unknown $w \in \mathbb{Z}_2^d$, the goal is learn $w$. He showed that with less than $d^2/10$ bits of memory, exponentially many samples are required. Intuitively, this memory requirement follows from the fact that one has to store $\Omega(d)$ samples in memory in order to learn $w$.

One can view this problem as a problem of learning a distribution over 
$\left(x_1, \dots, x_{d+1}\right) \in \mathbb{Z}_2^{d+1}$, where $x = 
(x_1,\dots, x_d)$ and $y = x_{d+1}$. There are $2^d$ possible distributions, 
each distribution corresponding to some $w \in \mathbb{Z}_2^d$. Furthermore, 
each distribution is $\mu_{I,\rho}$ for some $I \subseteq \{1,2,\dots,d+1\}$ 
and \footnote{Given $w = (w_1, \dots, w_d)$, let $I = \{ i \colon w_i = 1\} 
\cup \{d+1\}$. The distribution corresponding to $w$ is uniform over all 
$(x_1,\dots,x_{d+1}) \in \mathbb{Z}_2^{d+1}$ for which $\sum_{i \in I} x_i = 
0$, which is equivalent to $\mu_{I,\rho}$ with $\rho = 1$ (the only difference 
is that in our setting, the elements are from $\{-1,1\}$ instead of $\{0,1\}$ 
and the operation is multiplication instead of addition).} $\rho = 1$. 
Moreover, the memory requirement is $\Theta(d^2)$ and $d + O(1)$ samples are 
required. In contrast, we use different techniques to study a very different 
regime: There are $k$ distributions for some $k \in \{2, 3, \dots, 2^d\}$, 
$\rho$ is polynomially small in $k$, the memory requirement is 
$\tilde\Theta(k)$ and $\tilde\Theta(1/\varepsilon^2)$ samples are required. 
Additionally, our threshold is soft: one can learn with less memory and more 
samples, as opposed to requiring exponentially many samples already for 
$d^2/10$.

\end{document}